\documentclass[11pt]{article}

\RequirePackage{amsthm,amsmath,amsfonts,amssymb,bbm}
\RequirePackage{commath}
\RequirePackage{mathtools}
\RequirePackage{mathrsfs}
\RequirePackage{stmaryrd}
\RequirePackage[numbers,sort&compress]{natbib}
\RequirePackage[colorlinks,citecolor=blue!40!black,urlcolor=bluecolorlinks=true,
    linkcolor=blue!60!black]{hyperref}
\RequirePackage{graphicx}
\RequirePackage{tikz,pgfplots}
\pgfplotsset{compat=1.18}
\usepackage{authblk}
\usepackage[left=1in, right=1in, top=1in, bottom=1in, a4paper]{geometry}
\usepackage[utf8]{inputenc}
\usepackage{euscript}

\newcommand{\mbf}[1]{\mathbf{#1}}

\newcommand{\msf}[1]{\mathsf{#1}}
\newcommand{\bmsf}[1]{\boldsymbol{\mathsf{#1}}}

\newcommand{\mbs}[1]{\boldsymbol{#1}}

\newcommand{\ba}{\mbf{a}}
\newcommand{\bA}{\mbf{A}}

\newcommand{\bB}{\mbf{B}}

\newcommand{\bC}{\mbf{C}}

\newcommand{\bD}{\mbf{D}}
\newcommand{\be}{\mbf{e}}

\newcommand{\bG}{\mbf{G}}

\newcommand{\bI}{\mbf{I}}

\newcommand{\bK}{\mbf{K}}

\newcommand{\bM}{\mbf{M}}

\newcommand{\bN}{\mbf{N}}

\newcommand{\bu}{\mbf{u}}
\newcommand{\bU}{\mbf{U}}
\newcommand{\bv}{\mbf{v}}
\newcommand{\bV}{\mbf{V}}

\newcommand{\bW}{\mbf{W}}
\newcommand{\bx}{\mbf{x}}
\newcommand{\bX}{\mbf{X}}
\newcommand{\by}{\mbf{y}}
\newcommand{\bY}{\mbf{Y}}

\newcommand{\bZ}{\mbf{Z}}

\newcommand{\bsG}{\mbs{G}}

\newcommand{\bsM}{\mbs{M}}

\newcommand{\bsU}{\mbs{U}}

\newcommand{\msfc}{\msf{c}}

\newcommand{\msfd}{\msf{d}}

\newcommand{\msfg}{\msf{g}}
\newcommand{\msfG}{\msf{G}}

\newcommand{\msfs}{\msf{s}}
\newcommand{\msfS}{\msf{S}}

\newcommand{\msfv}{\msf{v}}
\newcommand{\msfV}{\msf{V}}

\newcommand{\bmsfA}{\bmsf{A}}

\newcommand{\bmsfX}{\bmsf{X}}

\newcommand{\bmsfY}{\bmsf{Y}}

\newcommand{\bmsfZ}{\bmsf{Z}}

\newcommand{\olbY}{\overline{\bY}}

\newcommand{\bmu}{\boldsymbol{\mu}}

\newcommand{\bsigma}{\boldsymbol{\sigma}}

\newcommand{\cA}{\mathcal{A}}
\newcommand{\cB}{\mathcal{B}}

\newcommand{\cD}{\mathcal{D}}
\newcommand{\cE}{\mathcal{E}}
\newcommand{\cF}{\mathcal{F}}

\newcommand{\cK}{\mathcal{K}}
\newcommand{\cL}{\mathcal{L}}
\newcommand{\cM}{\mathcal{M}}
\newcommand{\cN}{\mathcal{N}}
\newcommand{\cO}{\mathcal{O}}
\newcommand{\cP}{\mathcal{P}}

\newcommand{\cS}{\mathcal{S}}

\newcommand{\cU}{\mathcal{U}}
\newcommand{\cV}{\mathcal{V}}
\newcommand{\cW}{\mathcal{W}}
\newcommand{\cX}{\mathcal{X}}

\newcommand{\bCh}{\widehat{\mathbf{C}}}

\newcommand{\bCt}{\widetilde{\mathbf{C}}}

\newcommand{\bmsfXh}{\widehat{\bmsfX}}

\newcommand{\R}{\,\mathbb{R}}

\newcommand{\Q}{\,\mathbb{Q}}

\newcommand{\N}{\,\mathbb{N}}
\newcommand{\Z}{\,\mathbb{Z}}

\newcommand{\bbS}{\,\mathbb{S}}
\newcommand{\E}{\,\mathbb{E}}
\newcommand{\Pp}{\,\mathbb{P}}
\newcommand{\Dkl}{D_{\mathrm{KL}}}

\newcommand{\cov}{\mathop{\mathrm{Cov}}}
\newcommand{\var}{\mathop{\mathrm{Var}}}

\def\independenT#1#2{\mathrel{\rlap{$#1#2$}\mkern2mu{#1#2}}}
\newcommand{\subG}{\mathsf{subG}}

\newcommand{\op}{\msf{op}}
\newcommand{\Fr}{\msf{F}}
\newcommand{\vect}{\mathrm{vec}}

\newcommand{\rank}{\mathrm{rank}}

\newcommand{\ind}[1]{\mathbbm{1}_{\left\{#1\right\}}}
\newcommand{\Tr}[1]{\mathsf{Tr}\left\{#1\right\}}

\renewcommand{\det}[1]{\mathsf{det}\left\{#1\right\}}

\newcommand{\iid}{i.\@i.\@d.\ }

\newcommand{\bzero}{\boldsymbol{0}}
\newcommand\indep{\protect\mathpalette{\protect\independenT}{\perp}}

\newcommand{\Eb}{\EuScript{E}}
\newcommand{\Ebt}{\widetilde{\Eb}}
\newcommand{\DCME}{\msf{DCME}}
\newcommand{\DCCME}{\msf{DCCME}}
\newcommand{\cross}{\msf{cross}}
\newcommand{\self}{\msf{self}}
\newcommand{\msfcc}{\msf{cc}}
\newcommand{\msfsc}{\msf{sc}}
\newcommand{\dist}{\msf{dist}}

\newcommand{\ME}{\EuScript{M}_{\dist}}
\newcommand{\MEop}{\EuScript{M}_{\op}}
\newcommand{\MEF}{\EuScript{M}_{\Fr}}
\newcommand{\MEc}{\EuScript{M}_{\dist}^{(\cross)}}
\newcommand{\MEopc}{\EuScript{M}_{\op}^{(\cross)}}
\newcommand{\MEFc}{\EuScript{M}_{\Fr}^{(\cross)}}
\newcommand{\alphaopcc}{\alpha^{(\msfcc)}_{\op}}
\newcommand{\alphaopsc}{\alpha^{(\msfsc)}_{\op}}
\newcommand{\alphaFcc}{\alpha^{(\msfcc)}_{\Fr}}
\newcommand{\alphaFsccross}{\alpha^{(\msfsc,\cross)}_{\Fr}}
\newcommand{\alphaFsc}{\alpha^{(\msfsc)}_{\Fr}}
\newcommand{\alphaFccself}{\alpha^{(\msfcc,\self)}_{\Fr}}
\newcommand{\Frob}{Frobenius~}
\newcommand{\Vol}{\msf{Vol}}
\newcommand{\mkv}{-\!\!\!\!\minuso\!\!\!\!-}

\theoremstyle{plain}
\newtheorem{thm}{Theorem}[section]
\newtheorem{prop}[thm]{Proposition}
\newtheorem{lem}[thm]{Lemma}
\newtheorem{cor}[thm]{Corollary}
\theoremstyle{definition}
\newtheorem{dfn}[thm]{Definition}

\theoremstyle{remark}
\newtheorem{remark}[thm]{Remark}

\DeclareFontFamily{U}{matha}{\hyphenchar\font45}
\DeclareFontShape{U}{matha}{m}{n}{
      <5> <6> <7> <8> <9> <10> gen * matha
      <10.95> matha10 <12> <14.4> <17.28> <20.74> <24.88> matha12
      }{}
\DeclareSymbolFont{matha}{U}{matha}{m}{n}
\DeclareFontSubstitution{U}{matha}{m}{n}

\DeclareFontFamily{U}{mathx}{\hyphenchar\font45}
\DeclareFontShape{U}{mathx}{m}{n}{
      <5> <6> <7> <8> <9> <10>
      <10.95> <12> <14.4> <17.28> <20.74> <24.88>
      mathx10
      }{}
\DeclareSymbolFont{mathx}{U}{mathx}{m}{n}
\DeclareFontSubstitution{U}{mathx}{m}{n}

\DeclareMathDelimiter{\vvvert}{0}{matha}{"7E}{mathx}{"17}
\def\normiii#1{\left\vvvert #1\right\vvvert}

\def\tep{\tilde{\varepsilon}}
\theoremstyle{plain}
\newtheorem{axiom}{Axiom}
\newtheorem{claim}[axiom]{Claim}
\newtheorem{theorem}{Theorem}[section]
\newtheorem{lemma}[theorem]{Lemma}
\newtheorem{proposition}[theorem]{Proposition}
\theoremstyle{remark}
\newtheorem{definition}[theorem]{Definition}
\newtheorem*{example}{Example}
\newtheorem*{fact}{Fact}

\newcommand{\bXt}{\widetilde{\bX}}
\newcommand{\bZt}{\widetilde{\bZ}}
\newcommand{\bXb}{\overline{\bX}}
\newcommand{\Xb}{\overline{X}}
\newcommand{\bZb}{\overline{\bZ}}

\newcommand{\msfss}{\msfs^{\star}}
\newcommand{\FG}{\cF_{\msfG}}

\title{Fundamental Limits of Distributed Covariance Matrix Estimation via a Conditional Strong Data Processing Inequality}

\author{Mohammad Reza Rahmani}
\author{Mohammad Hossein Yassaee}
\author{Mohammad Reza Rahmani~~~ Mohammad Hossein Yassaee \and Mohammad Reza Aref\thanks{The authors are with the Department of Electrical Engineering, Sharif University of Technology, Tehran, Iran. Emails: mohammadreza\_rahmani@ee.sharif.edu,\{yassaee,aref\}@sharif.edu }}
\date{}

\begin{document}

\maketitle

\begin{abstract}
Estimating high--dimensional covariance matrices is a key task across many fields. This paper explores the theoretical limits of distributed covariance estimation in a feature--split setting, where communication between agents is constrained. Specifically, we study a scenario in which multiple agents each observe different components of i.i.d. samples drawn from a sub--Gaussian random vector. A central server seeks to estimate the complete covariance matrix using a limited number of bits communicated by each agent. We obtain a nearly tight minimax lower bound for covariance matrix estimation under operator norm and \Frob norm. Our main technical tool is a novel generalization of the strong data processing inequality (SDPI), termed the ``{\em Conditional Strong Data Processing Inequality (C-SDPI) coefficient}", introduced in this work. The C-SDPI coefficient shares key properties—such as tensorization—with the conventional SDPI. Crucially, it quantifies the average contraction in a state-dependent channel and can be significantly lower than the worst-case SDPI coefficient over the state input.  
 Utilizing the doubling trick of Geng--Nair and an operator Jensen inequality, we compute this coefficient for Gaussian mixture channels. We then employ it to establish minimax lower bounds on estimation error, capturing the trade-offs among sample size, communication cost, and data dimensionality. Building on this, we present a nearly optimal estimation protocol whose sample and communication requirements match the lower bounds up to logarithmic factors. Unlike much of the existing literature, our framework does not assume infinite samples or Gaussian distributions, making it broadly applicable. Finally, we extend our analysis to interactive protocols, showing interaction can significantly reduce communication requirements compared to non--interactive schemes.
\end{abstract}
\tableofcontents

\section{Introduction}
The estimation of the covariance matrix of a random vector from independent and identically distributed (i.i.d.) samples is a cornerstone problem with pervasive applications across diverse quantitative disciplines, including financial mathematics, classical and high--dimensional statistics, and modern machine learning methodologies \citep{hotelling1933analysis, dahmen2000structured, ledoit2003improved}. A particularly salient extension of this fundamental problem involves its analysis within distributed computational environments. In such settings, typified by emerging paradigms like federated learning \citep{mcmahan2017communication}, data are inherently partitioned and dispersed among multiple autonomous agents, with each agent having access to only a local subset of the overall data. These distributed data architectures can be broadly classified into two principal categories: (i) the \emph{sample--split (or horizontal split)} scenario, where each agent possesses a distinct subset of the data samples; and (ii) the \emph{feature--split (or vertical split)} scenario, where each agent has access to a specific subset of the dimensions (or features) for all available samples.

This paper investigates the challenging problem of covariance matrix estimation specifically within a feature--split distributed setting, under the critical constraint of limited communication. Our system model comprises a central server and multiple agents. Each Agent $k$ observes $d_k$ dimensions of $m$ i.i.d. samples of a $d$--dimensional random vector $\bZ\in\R^d$. The objective is for the central server to accurately estimate the true covariance matrix $\bC=\E[\bZ\bZ^\top]$. A defining characteristic of this problem is the restricted communication budget: each Agent $k$ is permitted to transmit messages of at most $B_k$ bits to the central server. Consequently, the central server must synthesize its estimate of the covariance matrix solely from these bandwidth--constrained messages. This setting prompts two pivotal questions:
\begin{enumerate}
\item What are the fundamental limits on estimation accuracy, considering the interplay between the agents' constrained communication budgets and the finite number of available samples?
\item What estimation schemes can effectively achieve these ultimate accuracy limits?
\end{enumerate}

This work provides comprehensive answers to both questions. We establish the information--theoretic lower bounds on the accuracy of distributed covariance estimation, thereby delineating the inherent performance bottlenecks imposed by communication and sample limitations. Building upon these fundamental insights, we derive explicit lower bounds for both the sample complexity and communication complexity that any estimation scheme must satisfy to attain a desired accuracy. Furthermore, we develop and analyze a novel estimation scheme designed to operate effectively under the stipulated communication constraints. We rigorously demonstrate that the proposed scheme's sample and communication complexities are optimal, aligning with the derived lower bounds within a logarithmic factor.

A cornerstone of our theoretical analysis, particularly for deriving lower bounds in estimation problems subjected to communication constraints, is the Strong Data Processing Inequality (SDPI) \citep{ahlswede1976spreading}. SDPI refines the classical Data Processing Inequality (DPI), which asserts that mutual information cannot increase through a Markov chain. SDPI quantifies this information degradation using contraction coefficients \citep{ahlswede1976spreading}. In this paper, we extend the classical SDPI to encompass state--dependent channels, introducing a novel concept termed the Conditional Strong Data Processing Inequality (C--SDPI) coefficient. This new quantity precisely measures the conditional mutual information loss between the input and output of a channel whose characteristics depend on an external state. We formally introduce C--SDPI, prove several of its key properties, and meticulously compute this coefficient for Gaussian mixture channels. This theoretical tool is then directly applied to derive minimax lower bounds on the estimation error for the distributed covariance matrix estimation problem.

\subsection{Prior works}
This section reviews prior research relevant to our study, broadly categorized into distributed covariance matrix estimation and works pertaining to the Strong Data Processing Inequality.

\subsubsection{Prior Works on Distributed Covariance Matrix Estimation}

For $m$ i.i.d. samples $\{\bZ^{(i)}\}_{i=1}^{m} = \{\bZ^{(1)},\bZ^{(2)},\dots,\bZ^{(m)}\}$ of a random vector $\bZ$, the canonical sample covariance estimator is given by:
\begin{equation}\label{sample_covar}
    \bCh = \frac{1}{m}\sum_{i=1}^{m}\bZ^{(i)}\bZ^{(i)\top}.
\end{equation}
For sub--Gaussian random vectors, bounds on the operator norm of the estimation error for this estimator are well--established \citep[Theorem 4.7.1]{vershynin2018high}. Beyond this standard estimator, a rich body of literature addresses covariance matrix estimation under various structural assumptions, such as sparsity \citep{bickel2008regularized, bickel2008covariance, el2008operator}, low--rankness \citep{furrer2007estimation}, and Toeplitz--structure \citep{huang2006covariance, wu2009banding, chen2012masked}, often deviating from the traditional sample covariance approach. The optimality of certain covariance matrix estimators has also been a focus of investigation \citep{cai2010optimal, cai2013optimal}.

The extension of fundamental machine learning algorithms to distributed settings, particularly with sample--split data, has been widely explored. Examples include distributed principal component analysis for dimension reduction \citep{qu2002principal, bai2005principal, balcan2014improved, kannan2014principal}, distributed gradient descent algorithms \citep{langford2009sparse, zinkevich2010parallelized, niu2011hogwild}, and distributed support vector machines \citep{navia2006distributed, zhu2007parallelizing, lu2008distributed, forero2010consensus}.

Statistical inference in distributed environments, especially under communication constraints, has garnered significant research interest. Zhang et al. \citep{zhang2013information} provided information--theoretic lower bounds for distributed parameter estimation, demonstrating the minimum communication required for specific performance levels. While not directly focused on covariance estimation, their methodologies offer a general framework. In distributed mean estimation, Suresh et al. \citep{suresh2017distributed} proposed a communication--efficient algorithm that achieves a mean squared error of $\Theta(d/n)$ with constant bits per dimension per client. Subsequent work by Cai and Wei \citep{cai2024distributed} characterized minimax convergence rates for Gaussian mean estimation under communication constraints. Beyond estimation, the closely related problem of distributed hypothesis testing has also been explored. Notably, Szabó et al. \cite{szabo2023optimal} investigate minimax testing errors in a sample--split distributed framework, where limited communication budget to a central machine. They consider both high--dimensional and infinite--dimensional signal detection problems within a Gaussian white noise model, deriving minimax lower bounds on error probabilities and proposing distributed testing algorithms that achieve these theoretical limits. Further work in distributed estimation includes Braverman et al. \citep{braverman2016communication}, who derived a minimax lower bound for sparse Gaussian vector mean estimation. Han et al. \citep{han2018geometric} offered a geometric approach, distinct from SDPI, for analyzing communication budgets in estimation problems, complementing the information--theoretic methods used in \citep{zhang2013information}.

In contrast to the sample--split paradigm, the feature--split setting presents unique challenges and arises in applications such as distributed medical databases where different health data dimensions of a patient reside in separate locations \citep{allaart2022vertical}, or in environmental monitoring where sensor stations collect partial weather information for correlation analysis without full data centralization. Prior work extending machine learning tasks to this vertical--split setting includes \citep{yang2019quasi, shen2019secure, hadar2019communication, hadar2019distributed, wu2020privacy}.

Within distributed covariance estimation with communication limits, some studies address the horizontal--split case \citep{zhang2013information, braverman2016communication, han2018geometric}, while others focus on vertical--split scenarios \citep{hadar2019communication, hadar2019distributed}. Specifically, \citep{hadar2019communication} investigates estimating the correlation $\rho=\E[XY]$ between two \emph{scalar} ($d_1=d_2=1$) \emph{Gaussian} or \emph{binary} random variables $X, Y$ in a vertical--split setup, assuming a one--sided communication constraint ($B_2=\infty$) and infinite samples ($m=\infty$). For this specific setting, they characterize the exact order of the optimal communication budget for achieving a specific estimation accuracy. \citep{hadar2019distributed} proposes a solution for estimating the correlation $\E[X_kY]$ between \emph{a vector} $\bX=[X_1,\cdots, X_d]^\top$ and a \emph{scalar} $Y$ ($d_1>d_2=1$), without any claim on its optimality. Their proposed solution outperforms the solutions based on estimating the correlation $\E[X_kY]$, for each $k$, separately. Our research significantly advances this line of inquiry by considering multi--dimensional random vectors, finite sample regimes, and communication constraints on all agents, providing a more general and practical framework.

\subsubsection{Prior Works on Strong Data Processing Inequality}

The Strong Data Processing Inequality (SDPI) is a pivotal tool for our lower bound proofs. It refines the classical Data Processing Inequality (DPI) \citep{cover1999elements}, which states that mutual information cannot increase along a Markov chain. SDPI quantifies this information loss via contraction coefficients \citep{ahlswede1976spreading}. Introduced by Ahlswede and Gács in 1976 \citep{ahlswede1976spreading}, SDPI has revealed connections to other information--theoretic measures, such as maximal correlation \citep{witsenhausen1975sequences}. Further theoretical developments and refinements of SDPI properties have been explored in various works \citep{cohen1993relative, choi1994equivalence, miclo1997remarques, del2003contraction, subramanian2013improved}. It has been established that the SDPI constant for any channel is upper--bounded by the Dobrushin contraction coefficient \citep{cohen1993relative}, another established measure of a channel's noise level \cite{dobrushin1956central1, dobrushin1956central2}, a result rediscovered in the machine learning community \citep{boyen2013tractable}. Relations between SDPI constants for different $f$--divergences have also been investigated \citep{choi1994equivalence}.

SDPI has found diverse applications, including studying the existence and uniqueness of Gibbs measures, establishing log--Sobolev inequalities, and analyzing performance limits of noisy circuits \cite{evans1999signal, polyWu2023}. In recent years, strong data processing inequalities have gained considerable attention in the information theory community \cite{kamath2012non, anantharam2013maximal, courtade2013outer, raginsky2013logarithmic, liu2014key, polyanskiy2017strong, makur2015bounds, calmon2015strong, raginsky2016strong}.
In \cite{anantharam2013maximal}, a new geometric characterization of the Hirschfeld--Gebelein--R{\'e}nyi maximal correlation \cite{witsenhausen1975sequences} is provided and its relation to SDPI is studied. In \cite{courtade2013outer}, a relation between rate distortion function and SDPI is obtained. In  \cite{raginsky2013logarithmic}, it is shown that the problem of finding SDPI constsnt and input distributions that achieve it can be addressed using so--called logarithmic Sobolev inequalities, which relate input relative entropy to certain measures of input--output correlation. In \cite{calmon2015strong}, SDPI for the power--constrained additive Gaussian channel is studied and the amount of decrease of mutual information under convolution with Gaussian noise is bounded.

In many classical techniques for deriving lower bounds in statistical inference—such as the Cramér method, Fano’s method, and Assouad’s lemma—a key step entails bounding the mutual information between two components of the statistical model. In scenarios subject to structural constraints, including communication or privacy limitations, these components are often connected through communication channels that impose a specific Markov structure. Consequently, a natural approach to bounding the mutual information is to investigate how it contracts through these channels, typically by employing strong data processing inequalities (SDPIs).

A substantial body of work has examined the use of SDPIs in distributed and sample--splitting estimation contexts. For instance, \citep{garg2014communication} applied SDPIs to analyze distributed estimation of the mean of a high--dimensional Gaussian distribution. Similarly, \citep{braverman2016communication} introduced a distributed variant of SDPI to establish nearly tight (up to logarithmic factors) trade--offs between estimation error and communication budget in sparse Gaussian mean estimation. More recently, \citep{cai2024distributed} leveraged SDPI to derive sharp lower bounds for mean estimation of multivariate Gaussian distributions under an independent distributed protocol, where each machine communicates with a central server through separate channels without interaction, covering the full range of communication budgets. The work \citep{szabo2023optimal} further explored mean testing by comparing global and local $\chi^2$ divergences, an approach closely related to SDPI. Beyond these, SDPIs have also been instrumental in deriving lower bounds for other distributed estimation problems \citep{xu2015converses} and in differentially private estimation \citep{duchi2013local}.

The work of \citep{hadar2019distributed} was the first to apply SDPI in a vertically partitioned setting, utilizing SDPI and its generalization—symmetric SDPI—to establish lower bounds on the estimation of correlation between two scalar random variables in both one--way and interactive protocols. Subsequently, \citep{sahasranand2021communication} employed SDPI to study testing of correlation between a vector and a scalar in a vertical--split context. Our work can be viewed as a generalization of \citep{hadar2019distributed}; specifically, we use a further generalization of SDPI to derive lower bounds for covariance matrix estimation in vertical--split data settings.

It is important to note that alternative methodologies have also been developed to establish lower bounds under information constraints, in the horizontal--split setting. For example, \citep{CannoneTyagiNEURIPS2023} extended the framework of \citep{acharya2020, acharya2020inference} to develop information contraction bounds, which can be interpreted as variants of SDPI, and demonstrated their effectiveness in problems such as sparse Gaussian mean estimation. In addition, geometric techniques introduced by \citep{han2018geometric}, \citep{barnes2020lower}, and \citep{barnes2020fisher} have been successfully applied to derive lower bounds in distributed parameter estimation and density estimation.

\subsection{Main Contributions}

In this paper, we address the problem of estimating the covariance matrix in a vertical--split setting under communication constraints. Our model considers a distributed system with $K$ agents and a central server. Each Agent $k\in[K]$ possesses $d_k$ dimensions of $m$ i.i.d. samples of a $d$--dimensional sub--Gaussian random vector $\bZ$. The central server's objective is to estimate the covariance matrix, with each agent $k$ limited to sending messages of at most $B_k$ bits. The central server then forms its estimate from these received messages. We seek to answer two fundamental questions: (1) What is the ultimate estimation accuracy given limited communication and samples? (2) How can this optimal accuracy be achieved?

Our principal contributions are as follows:

\begin{itemize}
    \item We introduce and rigorously develop a novel theoretical framework termed the \emph{Conditional Strong Data Processing Inequality (C–-SDPI)}, and establish its fundamental properties. Furthermore, we derive explicit expressions for the C-–SDPI constant in the setting of Gaussian mixture channels. Our analysis leverages a range of advanced technical tools, including techniques for establishing Gaussian optimality in certain classes of optimization problems \citep{GengNair2014,AJC2022}, as well as key properties of specific {\em operator convex functions} \citep{Tropp2015}.
    \item We formally define the Distributed Covariance Matrix Estimation ($\DCME$) problem and derive a near--optimal trade--off that captures the interplay between the number of samples ($m$), communication budgets ($B_1, B_2, \dots, B_K$), the data dimensionality available to each agent ($d_1, d_2, \dots, d_K$), and the resulting estimation error.

A central contribution of our work is the relaxation of several restrictive assumptions commonly found in prior studies:
\begin{itemize}
\item We allow the underlying random vectors $\bZ$ to be general sub--Gaussian, thereby extending beyond the conventional Gaussian setting.
\item Our analysis explicitly accommodates the practically relevant case of a finite number of i.i.d.\ samples ($m$), in contrast to earlier work that often assumes access to an infinite sample regime.
\end{itemize}
    
    \item By leveraging the \emph{C–-SDPI}, a variant of Fano’s method, and a symmetrization argument, we establish lower bounds on the \emph{minimax} estimation error achievable by any algorithm parameterized by $(m, d_1, d_2, B_1, B_2)$, measured with respect to both the operator norm and the Frobenius norm. In particular, to estimate the covariance matrix up to an error of $\varepsilon$ in the operator norm, the sample complexity and communication budgets must satisfy
$m = \Omega\Bigl(\frac{d}{\varepsilon^2}\Bigr)$
and $B_k = \Omega\Bigl(\frac{d d_k}{\varepsilon^2}\Bigr),
\quad k=1,2$, where $d=d_1+d_2$. 
    
    \item We construct an explicit estimation scheme that approximates the covariance matrix to within $\varepsilon$ error, with sample complexity and communication budgets satisfying
\(
m = \widetilde{\cO}\Bigl(\tfrac{d}{\varepsilon^2}\Bigr),
\quad
B_k = \widetilde{\cO}\Bigl(\tfrac{d d_k}{\varepsilon^2}\Bigr),
\quad k=1,2.
\)
This establishes the near--optimality of our approach, as it matches the minimax lower bounds up to logarithmic factors.
\item Using our derived minimax lower bound, we extend it to the general case with $K>2$ agents. We then present a new scheme that meets this bound, proving that it is achievable.
\item We extend our analysis to the interactive setting of the $\DCME$ problem, providing both a minimax lower bound and an achievable scheme. We demonstrate that the total communication budget required to estimate the covariance matrix can be reduced from $B := B_1 + B_2 = \widetilde{\Theta}\bigl(\frac{d^2}{\varepsilon^2}\bigr)$ in the non--interactive setting to $B = \widetilde{\Theta}\bigl(\frac{d_1 d_2}{\varepsilon^2}\bigr)$ in the interactive setting. This reduction can be substantial, particularly when there is significant imbalance between $d_1$ and $d_2$; for example, when $d_1=1$ and $d_2=d-1$. To the best of our knowledge, this improvement is only known in the non--parametric estimation for the vertical--split problems in the literature, cf. \citep{liu2023few}.

Furthermore, our results reveal that the statement in \cite[Theorem 6]{sahasranand2021communication}, which claims that the communication budget must be $\Omega\bigl(\frac{d^2}{\varepsilon^2}\bigr)$ in the case $d_1=1, d_2=d-1$, is incorrect. 
\end{itemize}

\subsection{Paper Structure}
The remainder of this paper is organized as follows: Section \ref{sec:preliminaries} provides necessary preliminaries. Section \ref{sec:CSDPI} introduces the Conditional Strong Data Processing Inequality (C--SDPI) and its properties, serving as the primary theoretical tool. In Section \ref{sec:problem_def}, we formally define the Distributed Covariance Matrix Estimation $(\DCME)$ problem as an application of C--SDPI, presenting theorems on both error lower bounds and achievable schemes. Section \ref{sec:Proof_lower_bound} is dedicated to the complete proof of the lower bound theorems. Section \ref{sec:proof_achievable_scheme} details the proposed estimation algorithm and proves the theorems related to its achievable performance. Section \ref{sec:Interactive_Case} defines the interactive $\DCME$ problem, presents an achievable scheme for this scenario, and proves its optimality. Finally, Section \ref{sec:conclusion} concludes the paper.

\subsection{Notations}
\label{sec:notation}
In this paper, we use specific notation to represent mathematical entities. Matrices are represented by uppercase bold symbols (e.g., $\bA$), while vectors are denoted by lowercase bold symbols (e.g., $\bv$). When referring to sets of vectors or matrices formed by concatenating vectors or matrices, we use a bold sans-serif font (e.g., $\bmsfA$). For any vector $\bv=[v_1,v_2,\cdots,v_d]^\top$, its $\ell_p$--norm is defined as $\norm{\bv}_p = \left(\sum_{i=1}^{d}|v_i|^p\right)^{1/p}$. For a matrix $\bA\in\R^{m\times n}$, the operator norm is denoted by $\norm[0]{\bA}_\op$, and the Frobenius norm is denoted by $\norm[0]{\bA}_\Fr$. We also define a general norm, $\norm{.}_{\dist}$, which encompasses both the operator and Frobenius norms for flexibility. We use $a \vee b$ and $a\wedge b$ to denote $\max\{a,b\}$ and $\min\{a,b\}$, respectively.  Lastly, the notation $[N]$ represents the set of integers $\{1,2,\dots,N\}$.
\section{Preliminaries}
\label{sec:preliminaries}

\subsection{Signed Permutation Matrices}
\begin{definition}[Signed Permutation Matrix]\label{def:signed_permutation}
A signed permutation matrix $\bA$ is a square  matrix with all the entries belong to $\{-1,0,1\}$, possessing exactly one {\em non--zero} entry of $\pm 1$ in each row and each column, with all other entries being 0.
\end{definition}
Let $\{\be_1,\dots,\be_d\}$ be the standard basis for $\R^d$. A signed permutation matrix $\bA$ of size $d$ can be represented by a permutation $\pi$ on $\{1,\cdots,d\}$ and a signed vector $\bsigma=[\sigma_1,\cdots,\sigma_d]^\top$, where $\sigma_i\in\{-1,1\}$, as follows:
\begin{equation}\label{eqn:SPM}
    \bA=\sum_{i=1}^{d} \sigma_i \be_i\be_{\pi(i)}^\top.
\end{equation}
Each signed permutation matrix $\bA$ is an orthogonal matrix, satisfying $\bA\bA^\top=\bI$. 
Let $\cP_d$ be the set of all signed permutation matrices of size $d$. 
The representation \eqref{eqn:SPM} implies that $|\cP_d|=2^d d!$. Furthermore, the signed permutation matrices in $\cP_d$ form a group under {\em matrix multiplication}.

The subsequent lemma is crucial for simplifying certain computations within this work.
\begin{lemma}\label{lem:SPM}
Let $\bA$ be a random matrix drawn uniformly from the set of signed permutation matrices $\cP_d$. Then, for any fixed matrix $\bB$, the following holds:
\begin{equation}
    \E\left[\bA^\top \bB\bA\right]=\frac{\Tr{\bB}}{d}\bI_d.
\end{equation}
\end{lemma}
The proof of Lemma \ref{lem:SPM} is stated in Appendix \ref{app:proof_lem_SPM}.

\subsection{Sub--Gaussian Random Variables}
Sub--Gaussian random variables, formally defined in Definition~\ref{def:sub_Gaussian_random_variable}, constitute a family of random variables whose tail behavior exhibits faster decay than that of a Gaussian distribution.
\begin{dfn}[Sub--Gaussian Random Variable {\citep[~Definition 2.2]{wainwright2019high}}]  \label{def:sub_Gaussian_random_variable}
  A random variable $X$ is defined as $\sigma$--sub--Gaussian if it satisfies:
  \[
      \E \Big[e^{\lambda(X - \E[X])}\Big] \leq \exp\Big({\frac{\lambda^2\sigma^2}{2}}\Big),\quad \text{for all } \lambda\in \R.
  \]
\end{dfn}
The definition of sub--Gaussian random variables can be extended to random vectors as follows:
\begin{dfn}[Sub--Gaussian Random Vector {\citep[~Section 6.3]{wainwright2019high}}] \label{def:sub_gaussian_random_vector}
  A random vector $\bX\in\R^d$ is called sub--Gaussian with parameter $\sigma$ if for all $\bv\in \bbS^{d-1}$, $\bv^\top \bX$ is a $\sigma$--sub--Gaussian random variable,
  where $\bbS^{d-1}=\{\bu\in\R^d\,:\, \norm{\bu}=1\}$ represents the $d$--dimensional unit sphere.
\end{dfn}
Several properties of sub--Gaussian random variables are detailed in Appendix \ref{app:prop_sub_gaussian_gamma}.

\subsection{Packing and Covering Numbers}\label{sec:pack_cover_nums}
Packing and covering numbers are two widely used notions for quantifying the "size" or complexity of a subset $\{\bx_1,\bx_2,\dots,\bx_N\}$ within a metric space $\cK$. We present their formal definitions below:

\begin{dfn}[Covering Number {\citep[~Definition 5.1]{wainwright2019high}}]
    A set $\{\bx_1,\bx_2,\dots,\bx_N\}\subseteq \cK$ is defined as an $\epsilon$--covering set with respect to a metric $\msfd$ if for all $x\in\cK$, there exists some $j\in[N]$ such that $\msfd(\bx,\bx_j)\leq\epsilon$. The covering number $\cN(\cK,\msfd,\epsilon)$ is defined as the cardinality of the smallest $\epsilon$--covering set of set $\cK$, with respect to the metric $\msfd$.
\end{dfn}

\begin{dfn}[Packing Number {\citep[~Definition 5.4]{wainwright2019high}}]
  A set $\{\bx_1,\bx_2,\dots,\bx_M\}\subseteq \cK$ is called a $\epsilon$--packing set with respect to a metric $\msfd$ if $\msfd(\bx_i,\bx_j)>\epsilon$ for all distinct $i,j\in[M]$. The packing number $\cM(\cK,\msfd,\epsilon)$ defined as the cardinality of the largest $\epsilon$--packing set of set $\cK$, with respect to the metric $\msfd$.
\end{dfn}
\section{Conditional Strong Data Processing Inequality}
\label{sec:CSDPI}
The Data Processing Inequality is a well-established result. However, for the sake of completeness, we present its definition here:
\begin{thm}[Data Processing Inequality {\citep[~Theorem 2.15]{polyWu2023}}]\label{thm:DPI}
Let $P_Y = T_{Y|X} \circ P_X$ and $Q_Y = T_{Y|X} \circ Q_X$, Then:
\[\Dkl(Q_Y\|P_Y)\leq \Dkl(Q_X\|P_X).\]
\end{thm}
\emph{Strong Data Processing Inequality} constitutes a refinement of the Data Processing Inequality.
\begin{dfn}[Strong Data Processing Inequality (SDPI) or Contraction Coefficient {\citep[Definition 33.10]{polyWu2023}}]\label{def:contraction_coefficient}
    For any input distribution $P_X$ and Markov kernel $T_{Y|X}$, we define the contraction coefficient as:
    \[s(P_X,T_{Y|X}):=\sup_{Q_X:0<\Dkl(Q_X\|P_X)<\infty} \frac{\Dkl(Q_Y\|P_Y)
	}{\Dkl(Q_X\|P_X)},\]
	where $P_Y = T_{Y|X} \circ P_X$ and $Q_Y = T_{Y|X} \circ Q_X$. If either $X$ or $Y$ is a constant, we define $s(P_X,T_{Y|X})$ to be $0$.
	
	Furthermore SDPI coefficient can be equivalently represented using the following mutual information characterization \cite[Theorem V.2]{raginsky2016strong}:
	\[s(P_X,T_{Y|X}):=\sup_{P_{U|X}:U\mkv X\mkv Y} \frac{I(U;Y)
	}{I(U;X)}.\]
\end{dfn}
We now present a generalization of the Strong Data Processing Inequality for state-dependent channels.
\begin{dfn}[Conditional Strong Data Processing (C--SDPI) Coefficient]\label{def:csdpi_coefficient}
    For any input distribution $P_X$, {\em state} distribution $P_V$ and Markov kernel $T_{Y|X,V}$, we define the Conditional Strong Data Processing Inequality (C--SDPI) coefficient as:
    \[
    \msfs(P_X,T_{Y|X,V}\mid P_V):=\sup_{Q_X:0<\Dkl(Q_X\|P_X)<\infty} \frac{\Dkl(Q_{Y|V}\|P_{Y|V}\mid P_V)
	}{\Dkl(Q_X\|P_X)},
	\]
	where $P_{Y|V} = T_{Y|X,V} \circ P_X$ and $Q_{Y|V} = T_{Y|X,V} \circ Q_X$.
	If either $X$ or $Y$ is a constant, we define $\msfs(P_X,T_{Y|X,V}\mid P_V)$ to be $0$.
	\end{dfn}
	It is readily apparent that:
\begin{equation}\label{eq:simp_bound_csdpi}
    \msfs(P_X,T_{Y|X,V}|P_V)\leq\sup_v\left\{s(P_X,T_{Y|X,V=v})\right\}.
\end{equation}
Also a more careful examination of the definition implies:
\begin{equation}\label{eq:simple_bound_csdpi}
    \msfs(P_X,T_{Y|X,V}|P_V)\leq\E_{\bar{V}\sim P_V}\left[s(P_X,T_{Y|X,V=\bar{V}})\right].
\end{equation}
	\begin{remark}\label{re:cond-or-uncond}
	The C--SDPI coefficient $\msfs(P_X,T_{Y|X,V}\mid P_V)$ can be interpreted as SDPI coefficient of the Markov kernel $\widetilde{T}_{Y,V|X}=P_VT_{Y|X,V}$, that is:
	\[
    \msfs(P_X,T_{Y|X,V}\mid P_V)=s(P_X,\widetilde{T}_{Y,V|X}).
	\]
	In particular, this implies the following mutual information characterization holds for the C--SDPI coefficient:
	\begin{equation*}
	 \msfs(P_X,T_{Y|X,V}\mid P_V)=\sup_{P_{U|X}:U\mkv (X,V)\mkv Y\atop (U,X)\indep V} \frac{I(U;Y|V)
	}{I(U;X)}.
		\end{equation*}
\end{remark}
	The reason we distinguish between the SDPI and C--SDPI is the following {\em conditional} tensorization property:

\begin{theorem}[Tensorization of C--SDPI coefficient] \label{thm:tensorization}
    For any independent input distributions $P_{X_1}, P_{X_2}, P_V$ and Markov kernels $P_{Y_1|X_1,V}$ and $P_{Y_2|X_2,V}$ conditioned on the same random variable $V$, the following equality holds:
    \begin{equation}
    \begin{aligned}
        \msfs\left(P_{X_1} P_{X_2},T_{Y_1|X_1 V} T_{Y_2|X_2,V}\mid P_V\right)=\max_{k\in\{1,2\}}\left\{\msfs(P_{X_k},T_{Y_k|X_k,V}|P_V)\right\}.
    \end{aligned}
    \end{equation}
    In particular:
    \begin{equation}
    	\msfs(P_X^{\otimes n},T^{\otimes n}_{Y|X,V}\mid P_V)=\msfs(P_X,T_{Y|X,V}\mid P_V),
    \end{equation}
    where we have employed the abuse notation $T^{\otimes n}_{Y|X,V}(y^n|x^n,v):=\prod_{i=1}^nT_{Y|XV}(y_i|x_i,v)$.
\end{theorem}
\begin{remark}
    Using the tensorization property of the SDPI \citep[Proposition 33.11]{polyWu2023}, we obtain:
    \[
    \msfs(P_X^{\otimes n},\widetilde{T}^{\otimes n}_{Y|X,V}|P_V^{\otimes n})=s(P_X^{\otimes n},\widetilde{T}^{\otimes n}_{Y,V|X})=s(P_X,\widetilde{T}_{Y,V|X})=\msfs(P_X,T_{Y|X,V}\mid P_V)
    \]
    Here, the product channel is defined as: $\widetilde{T}^{\otimes n}_{Y|X,V}(y^n|x^n,v^n)=\prod_{i=1}^n \widetilde{T}_{Y|X,V}(y_i|x_i,v_i)$. This result contrasts with the conditional tensorization property (Theorem~\ref{thm:tensorization}): in the unconditional setting, the channel state varies across instances, whereas in the conditional setting, it remains fixed.
\end{remark}
The subsequent identity is crucial for the proof of Theorem \ref{thm:tensorization} and for establishing Gaussian optimality in the subsequent section.
\begin{lemma}\label{le:identity-from-CR}
    For any joint distributions $Q_{X_1,X_2}$, product channel $T_{Y_1|X_1,V}T_{Y_2|X_2,V}$, distribution $P_V$ and arbitrary conditional distributions $P_{Y_k|V}$ for $k=1,2$, the following identity holds:
    \begin{equation}
    \begin{aligned}
       	\Dkl(Q_{Y_1,Y_2|V}\|P_{Y_1|V}P_{Y_2|V}|P_V)
    	=\Dkl(Q_{Y_1|V}\|P_{Y_1|V}|P_V)&+\Dkl(Q_{Y_2|X_1,V}\|P_{Y_2|V}|Q_{X_1,V})\\&-I_Q(X_1;Y_2|Y_1,V),
    \end{aligned} 
    \end{equation}
    where $Q_{Y_1,Y_2|V}$ and the mutual information are computed with respect to the following joint distribution:
    \begin{equation*}
        Q_{V,X_1,X_2,Y_1,Y_2}=P_VQ_{X_1,X_2}T_{Y_1|X_1,V}T_{Y_2|X_2,V}.
    \end{equation*}
\end{lemma}

\begin{proof}
    The proof is stated in Appendix \ref{app:proof_lemma_CSDPI}.
\end{proof}

We now proceed with the proof of Theorem \ref{thm:tensorization}.
\begin{proof}[Proof of Theorem \ref{thm:tensorization}]
     For brevity, Let $\msfs_k=\msfs(P_{X_k},T_{Y_k|X_k,V}|P_V)$ and $\msfs=\max\{\msfs_1,\msfs_2\}$. Observing that, given the product input  $P_{X_1X_2}=P_{X_1}P_{X_2}$, we have $P_{Y_1,Y_2|V}=P_{Y_1|V}P_{Y_2|V}$, we invoke Lemma \ref{le:identity-from-CR} to obtain:
    \begin{equation}
    \begin{aligned}
    	\Dkl(Q_{Y_1,Y_2|V}\|P_{Y_1|V}P_{Y_2|V}|P_V)&\stackrel{\text{(a)}}{\leq} \Dkl(Q_{Y_1|V}\|P_{Y_1|V}|P_V)+\Dkl(Q_{Y_2|X_1,V}\|P_{Y_2|V}|Q_{X_1,V})\\
    	&\stackrel{\text{(b)}}{\leq} \msfs_1 \Dkl(Q_{X_1}\|P_{X_1})+\E_{Q_{X_1}}\left[\Dkl(Q_{Y_2|X_1,V}\|P_{Y_2|V}|P_{V})\right]\\
    	&\stackrel{\text{(c)}}{\leq} \msfs_1 \Dkl(Q_{X_1}\|P_{X_1})+\E_{Q_{X_1}}\left[\msfs_2\Dkl(Q_{X_2|X_1}\|P_{X_2})\right]\\
    	&\stackrel{\text{(d)}}{\leq}\msfs \Dkl(Q_{X_1,X_2}\|P_{X_1}P_{X_2}),
    \end{aligned}
    \end{equation}
    where (a) follows from Lemma \ref{le:identity-from-CR}, (b) follows from the definition of C--SDPI coefficient, (c) also follows from the definition of the C--SDPI coefficient by noting that for a fixed $X_1=x_1$, the output distribution associated with $Q_{X_2|X_1=x_1}$ is $Q_{Y_2|V,X_1=x_1}$ and (d) follows from the definition of $\msfs$ and the chain rule for KL--divergence.
    
    Thus, we have established that $\msfs\left(P_{X_1} P_{X_2},T_{Y_1|X_1 V} T_{Y_2|X_2,V}\mid P_V\right)\leq\max_{k\in\{1,2\}}\left\{\msfs(P_{X_k},T_{Y_k|X_k,V}|P_V)\right\}$. The converse inequality, $\msfs\left(P_{X_1} P_{X_2},T_{Y_1|X_1 V} T_{Y_2|X_2,V}\mid P_V\right)\geq \max_{k\in\{1,2\}}\msfs(P_{X_k},T_{Y_k|X_k,V}|P_V)$ is readily follows from the definition.
\end{proof}

\subsection{SDPI Coefficient for Normal Distribution}

In \citep{kim2017discovering}, the SDPI coefficient is derived for multivariate normal distribution.
\begin{lem}[{\citep[~Section 2.6]{kim2017discovering}}]\label{lem:sdpi_gaussian}
    If $\begin{bmatrix}
        \bX\\
        \bY
    \end{bmatrix}\sim\cN\Big(\bmu,\bC=\begin{bmatrix}
            \bC_{\bX\bX} & \bC_{\bX\bY} \\
            \bC_{\bX\bY}^\top & \bC_{\bY\bY}
        \end{bmatrix}\Big)$, then the following holds:
    \[s(P_{\bX};T_{\bY|\bX}) = \norm{\bC_{\bX\bX}^{-1/2}\bC_{\bX\bY}\bC_{\bY\bY}^{-1/2}}_{\op}^2.\]
\end{lem}

\subsection{C--SDPI Coefficient for (Mixture) Normal Distribution}
Consider a joint distribution of $(\bX,\bY,V)$ such that the conditional distribution of $(\bX,\bY)$ given $V=v$ is a normal distribution, and $(\bX,V)$ are mutually independent. Consequently, $\bX$ is normally distributed, and  $(\bX,\bY)$ follows a mixture normal distribution. More specifically, consider the mixture normal distribution defined by the following conditional normal distribution:
\begin{equation}\label{eqn:mixture}
	\begin{bmatrix}
        \bX\\
        \bY
    \end{bmatrix}|\{V=v\}\sim \cN\left(\bzero,
	\begin{bmatrix}
    \bI_{d_{\bX}}&\bA_v^\top\\
    \bA_v&\bI_{d_{\bY}}
    \end{bmatrix}
    \right),
\end{equation}
where $\bA_v$ denotes the {\em cross} covariance matrix between $\bX$ and $\bY$, conditioned on $V=v$. It is also assumed that $\bA_v\bA_v^\top\preceq \bI_{d_Y}$. The conditional relationship between $\bX$ and $\bY$ given $V=v$ can be expressed by the following Gaussian channel:
\begin{equation}\label{eq:Markov-N}
	\bY=\bA_v\bX+\bZ_v,
\end{equation} 
where $\bZ_v\sim\cN(0,\bI_{d_Y}-\bA_v\bA_v^\top)$ is independent of $\bX$. Let $P_{\bX}$ denotes the distribution of $\bX$, $P_V$ denotes the distribution of $V$, and $T_{\bY|\bX,V}$ represents the Markov kernel (channel) corresponding to \eqref{eq:Markov-N}.

We are interested in determining the conditional strong data processing coefficient for this joint distribution. A straightforward upper bound, derived using \eqref{eq:simple_bound_csdpi}, is given by:
\begin{equation}\label{eqn:eta-U}
\msfs(P_{\bX},T_{\bY|\bX,V}|P_V)\leq\E_{\bar{V}\sim P_V}\left[s(P_X,T_{Y|X,V=\bar{V}})\right]
=\E_V\left[\|\bA_v^\top \bA_v\|_{\op}\right],
\end{equation}
where the final equality follows from Lemma \ref{lem:sdpi_gaussian}.

To establish a lower bound, let $Q_{\bX} = \cN(\bmu,\bI_{d_X})$. Then, given $V=v$, and using \eqref{eq:Markov-N}, we have $Q_{\bY|V=v}=\cN(\bA_v\bmu,\bI_{d_Y})$. For all $\msfs\geq0$ we have:
\begin{equation}
\begin{aligned}
    \msfs \Dkl(Q_{\bX}\|P_{\bX})-\Dkl(Q_{\bY|V}\|P_{\bY|V}|P_V)&=\frac{1}{2}\left(
	\msfs\|\bmu\|^2-\E_V\left[\|\bA_V\bmu\|^2\right]
	\right)\\
	&=\frac{1}{2}\bmu^\top\left(
	\msfs \bI_{d_X}-\E_V\left[\bA_V^\top \bA_V\right]
	\right)\bmu.
\end{aligned}
\end{equation}
Now for $\msfs<\|\E[\bA_V^\top \bA_V]\|_{\op}$, the infimum of the above expression over all $\bmu$ is $-\infty$, Consequently, the C--SDPI coefficient admits the following lower bound:
\begin{equation}\label{eqn:eta-L}
\msfs(P_{\bX},T_{\bY|\bX,V}|P_V)\geq
\|\E[\bA_V^\top \bA_V]\|_{\op}.
\end{equation}
As will be demonstrated, this lower bound is indeed the C--SDPI coefficient. Before formally stating this result, we compare the upper bound \eqref{eqn:eta-U} and the lower bound \eqref{eqn:eta-L} through a simple example.

\begin{example} 
    Let $d_{\bX}=d$ and $d_Y=1$. Assume that $V$ is drawn uniformly from the set $\{1,\dots,d\}$. Furthermore, let $\bA_v=\be_v^\top$, where $\{\be_1,\dots,\be_d\}$ constitutes the standard basis for $\R^d$. In this instance, it is straightforward to verify that $\E[\bA_V^\top \bA_V]=\frac{1}{d}\bI_d$. Moreover, $\|\bA_v^\top \bA_v\|_{\op}=1$. Consequently, \eqref{eqn:eta-U} and \eqref{eqn:eta-L} yield the bounds $\frac{1}{d}\leq \msfs(P_{\bX},T_{Y|\bX,V}|P_V)\leq 1$. This demonstrates that the upper and lower bounds can exhibit a significant difference.
\end{example}

\begin{theorem}\label{thm:csdp-mix}
    The conditional strong data processing coefficient for the mixture normal distribution defined in \eqref{eqn:mixture} is given by:
    \begin{equation}\label{eqn:eta-mixture}
    \msfs(P_{\bX},T_{\bY|\bX,V}|P_V)=
    \norm[1]{\E[\bA_V^\top \bA_V]}_{\op}.
    \end{equation}
\end{theorem}

We highlight the principal steps of the proof, deferring the detailed derivations to the Supplementary Material.
\begin{proof}
    Let $\msfss=\|\E[\bA_V^\top \bA_V]\|_{\op}$. Define the following functional:
    \begin{equation}\label{eqn:def-functional}
    \msfg(Q_{\bX}):=\msfss \Dkl(Q_{\bX}\|P_{\bX})-\Dkl(Q_{\bY|V}\|P_{\bY|V}|P_V),   
    \end{equation}
    where $P_{\bX}=\cN(0,\bI_{d_{\bX}})$ and $P_{\bY|V}=P_{\bY}=\cN(0,\bI_{d_{\bY}})$.
    It was established in \eqref{eqn:eta-L} that $\msfss$ constitutes a lower bound for $\msfs(P_{\bX},T_{\bY|\bX,V}|P_V)$. Thus, to proceed,  it suffices to prove that:
    \begin{equation}\label{eqn:inf-gen}
        \inf_{Q_{\bX}:\Dkl(Q_{\bX}\|P_{\bX})<\infty}\msfg(Q_{\bX})=0.
    \end{equation}
    The key step is to show that the optimization problem \eqref{eqn:inf-gen} can be restricted to the family of centered normal distributions. More precisely, we have:
    \begin{lemma}\label{le:normal-opt}
    Let $\FG$ denote the set of centered normal distributions $Q_{\bX}$ with the property that $\Dkl(Q_{\bX}\|P_{\bX})$ is finite. Then:
    \begin{equation}\label{eqn:inf-normal}
        \inf_{Q_{\bX}:\Dkl(Q_{\bX}\|P_{\bX})<\infty}\msfg(Q_{\bX})=\inf_{Q_{\bX}\in\FG}\msfg(Q_{\bX}).
    \end{equation}
    \end{lemma}
    The proof scheme closely follows the approach proposed by Geng and Nair \cite{GengNair2014} and subsequently employed by Anantharam et al. \cite{AJC2022} for similar functionals. 
    Specifically, we demonstrate that the optimizer must satisfy a certain rotational invariance, thereby implying that it is a normal distribution. The proof of Lemma \ref{le:normal-opt} is provided in Supplementary Material \ref{app:gaussian_optimality}.
    
    Based on Lemma \ref{le:normal-opt}, it suffices to prove:
    \begin{equation}
        \inf_{Q_{\bX}\in\FG}\msfg(Q_{\bX})=0
    \end{equation}
    We observe that $\msfg(P_{\bX})=0$. Thus, it remains to show that $\msfg(Q_{\bX})\geq 0$ for any $Q_{\bX}\neq P_{\bX}$. Our proof relies on the concept of operator convexity \cite{Tropp2015} applied to a specific function. In particular, we utilize the following lemma:
    \begin{lemma}
    
    \begin{enumerate}
        \item \cite[Proposition 8.4.8]{Tropp2015} The logarithm function $f(\bB)=\log \bB$ is an operator concave function on the set of positive definite matrices.
        \item ({\bf Operator Jensen inequality}, \cite[Theorem 8.5.2]{Tropp2015})
        Let $f$ be an operator concave function defined on the set of positive definite matrices. Let $\bK_1$ and $\bK_2$ be two (rectangular) matrices that decompose the identity matrix as follows:
        \begin{equation}\label{eqn:matrix-cons}
        \bK_1  \bK_1^\top+\bK_2 \bK_2^\top=\bI
        \end{equation}
        Then, for any two positive definite matrices $\bD_1$ and $\bD_2$, the following Jensen-type inequality holds:
        \begin{equation}\label{eqn:matrix-jensen}
            f\left(\bK_1 \bD_1\bK_1^\top+\bK_2 \bD_2\bK_2^\top\right)\succeq \bK_1 f(\bD_1)\bK_1^\top+\bK_2 f(\bD_2)\bK_2^\top
    \end{equation}
    \end{enumerate}
    \end{lemma}
    
    We now proceed to complete the proof.
    
    Let $Q_{\bX}=\cN(0,\bI_{d_{\bX}}+\bB)$, where $\bB\succ -\bI$. Algebraic manipulation reveals that $Q_{\bY|V}=\cN(0,\bI+\bA_v\bB\bA_v^\top)$. Applying the KL divergence formula for Gaussian distributions \cite[Example 2.2]{polyWu2023}, we obtain:
    \begin{align}
        \Dkl(Q_{\bX}\|P_{\bX})&=\frac{1}{2}\Tr{\bB-\log(\bI+\bB)},\\
        \Dkl(Q_{\bY|V}\|P_{\bY}|P_V)&=\frac{1}{2}\E_V\left[\Tr{\bA_V\bB\bA_V^\top-\log(\bI+\bA_V\bB\bA_V^\top)}\right],\label{eqn:out-kl}
    \end{align}
    where $\log(\bA)$ denotes the matrix logarithm of $\bA$. Applying \eqref{eqn:matrix-jensen} to the logarithm function with $\bD_1=\bI+\bB$, $\bD_2=\bI$, $\bK_1=\bA_v$ and an appropriate $\bK_2$ such that \eqref{eqn:matrix-cons} is satisfied (such $\bK_2$ exists, because $\bA_v\bA_v^\top \preceq \bI_{d_Y}$ for any $v$) yields:
    \[
    \log(\bI+\bA_v\bB\bA_v^\top)\succeq \bA_v\log(\bI+\bB)\bA_v^\top.
    \]
    This implies:
    \begin{equation*}
        \bA_v\bB\bA_v^\top-\log(\bI+\bA_v\bB\bA_v^\top)\preceq \bA_v(\bB-\log(\bI+\bB))\bA_v^\top.
    \end{equation*}
    Taking the trace on both sides implies:
    \begin{equation*}
      \Tr{\bA_v\bB\bA_v^\top-\log(\bI+\bA_v\bB\bA_v^\top)}\leq \Tr{\bA_v^\top \bA_v(\bB-\log(\bI+\bB))}.
    \end{equation*}
    Consequently, substituting this into \eqref{eqn:out-kl} results in:
    \begin{align}
        \Dkl(Q_{\bY|V}\|P_{\bY}|P_V)&\leq \frac{1}{2}\Tr{\E\left[\bA_V^\top \bA_V\right](\bB-\log(\bI+\bB))} \\
        &\leq \frac{\msfss}{2}\Tr{\bB-\log(\bI+\bB)}\label{eqn:trace-norm}\\
        &=\msfss \Dkl(Q_{\bX}\|P_{\bX}),
    \end{align}
    where \eqref{eqn:trace-norm} follows from the fact that $\bB-\log (\bI+\bB)$ is always positive semi-definite, and the inequality $\Tr{\bA\bB}\leq \norm{\bA}_{\op}\Tr{\bB}$ for any pair $\bA,\bB\succeq \bzero$ \cite{von1937some}. This concludes the proof of Theorem \ref{thm:csdp-mix}.
\end{proof}

\subsection{SDPI coefficient of a  Gaussian mixture channel}\label{sub:GMC}
Let $V$ be a latent random variable with distribution $P_V$ taking values in some index set $\mathcal V$, and for each $v\in\mathcal V$ let:
\[
T_{\bY|\bX,V=v}\;=\;\cN\bigl(\bA_v\bX,\;\bI-\bA_v\bA_v^\top\bigr),
\]
so that the joint channel is $T_{\bY,V|\bX}(\by,v|\bx)=P_V(v)\,T_{\bY|\bX,V=v}(\by|\bx)$.  By marginalizing out $V$, one obtains the \emph{Gaussian mixture channel}:
\begin{equation}\label{eq:GMC}
\widetilde T_{\bY|\bX=\bx}
\;=\;
\int_{\mathcal V}
\cN\bigl(\bA_v\bx,\;\bI-\bA_v\bA_v^\top\bigr)\,d P_V(v).
\end{equation}

Following Remark \ref{re:cond-or-uncond} and leveraging the monotonicity of the SDPI coefficient \citep{polyWu2023}, we obtain the inequality $s(P_{\bX},\widetilde{T}_{\bY|\bX}) \leq \norm[1]{\E[\bA_V^\top \bA_V]}_{\op}$. We show that this upper bound for the SDPI of the Gaussian mixture channel is {\em indeed} tight. 
\begin{prop}\label{prop:GMC}
The SDPI coefficient of the Gaussian mixture channel $\widetilde{T}_{\bY|\bX}$ defined in \eqref{eq:GMC} is given by:
\begin{equation}
s(P_{\bX}, \widetilde{T}_{\bY|\bX}) = \|\E[\bA_V^\top \bA_V]\|_{\op}
\end{equation}
\end{prop}

The proof of Proposition \ref{prop:GMC} is stated in Appendix \ref{app:proof_prop_GMC}.

\section{Application: Distributed Covariance Matrix Estimation $(\DCME)$}
\label{sec:problem_def}
In this section, as an application of the conditional SDPI, we introduce the $\DCME$ problem and provide a solution.
\subsection{Problem formulation}
We consider a system including a central server and $K$ agents, as depicted in Fig. \ref{fig:problem_statement}. The central server aims to estimate the covariance matrix $\bC=\E[(\bZ-\E[\bZ])(\bZ-\E[\bZ])^\top]$ of a $d$--dimensional $\sigma$--sub--Gaussian random vector $\bZ\sim P$ (see Definition \ref{def:sub_gaussian_random_vector}), based on $m$ independent and identically distributed (i.i.d.) samples $\bZ^{(1)},\dots,\bZ^{(m)}$. However, the central server lacks direct access to these samples. Instead, each Agent $k$, for all $k\in[K]$, possesses complete knowledge of certain dimensions of all $m$ samples. Specifically, Agent $k$ has access to the dimensions $[1+\sum_{i=1}^{k-1}d_i:\sum_{i=1}^{k}d_i]$ of each data vector. The information accessible to Agent $k$ can be represented by $  \{\bX_k^{(i)}\}_{i=1}^{m}= \{\bZ_{[1+\sum_{i=1}^{k-1}d_i:\sum_{i=1}^{k}d_i]}^{(i)}\}_{i=1}^{m}$.
The central server's objective is to estimate $\bC$ by receiving up to $B_k$ bits from Agent $k$. 

We denote this problem of distributed covariance matrix estimation $(\DCME)$ with parameters $(\sigma,m,d_{1:K},B_{1:K})$ as $\DCME(\sigma,m,d_{1:K},B_{1:K})$. More formally, this problem consists of $K$ encoder functions and one decoder function, defined as follows:
\begin{itemize}
    \item $K$ encoder functions $\cE_k:\R^{d_k \times m}\mapsto [1:2^{B_k}]$ where encoder $k$ maps $\{\bX_k^{(i)}\}_{i=1}^{m}$  to 
    $M_k=\cE_k(\{\bX_k^{(i)}\}_{i=1}^{m})$.
    
    \item A decoder function $\cD:[1:2^{B_1}]\times [1:2^{B_2}]\times\dots\times[1:2^{B_K}] \mapsto \msfS_{+}^{d \times d}$,
    where $\msfS_{+}^{d \times d}$ represents the set of positive semi--definite matrices of dimension $d \times d$. The decoder function maps $(M_1,M_2,\dots,M_K)$  to $\bCh=\cD(M_1,M_2,\dots,M_K)$.
\end{itemize}

The distortion of a $\DCME$ scheme, under the $\dist$ norm (where $\dist$ can be either the operator norm or Frobenius norm), is quantified by the $\dist$ norm of the difference between the estimated covariance matrix $\bCh$ and the true covariance matrix $\bC$; that is,  $\cL_{\dist}(\bCh,\bC)=\norm[1]{\bCh-\bC}_{\dist}$. 
The expected distortion of a $\DCME$ scheme is given by:
\begin{equation}\label{eq:expected_distortion}
    \E\left[\cL_{\dist}(\bCh,\bC)\right]=\E_{\{\bZ^{(i)}\}_{i=1}^{m}\sim P^{\otimes m}}\left[\norm[1]{\bCh-\bC}_{\dist}\right].
\end{equation}
The objective is to design the encoding functions $\left\{\cE_k\right\}_{k=1}^{K}$ and the decoding function $\cD$, to minimize the expected distortion \eqref{eq:expected_distortion} in the worst--case scenario. In other words, we aim to characterize the {\emph minimax} distortion, defined as:
\begin{equation}\label{eq:minimax-distortion}
    \ME(\sigma,m,d_{1:K},B_{1:K}):= \inf_{{\cE_1,\cdots,\cE_K,\cD}} \sup_{P\in\cP}\E\left[\cL_{\dist}(\bCh,\bC)\right],
\end{equation}
where $\cP=\subG^{(d)}(\sigma)$ represents the family of $\sigma$--sub--Gaussian $d$--dimensional distributions.

\subsubsection{Distributed Cross Covariance Matrix Estimation $(\DCCME)$}
A key component of the analysis of the bounds on the minimax distortion \eqref{eq:minimax-distortion} involves investigating cross--covariance estimation for the case of {\em two} users. Specifically, we consider the scenario where there are $K=2$ agents, and the objective is to estimate the {\emph coss--covariance} matrix $\bC_{21}:=\E[(\bX_2-\E[\bX_2])(\bX_1-\E[\bX_1])^\top]$ using the same communication setup as in $\DCME$. In this scenario, the encoders remain as previously defined. The decoder function is given by $\cD_{21}:[1:2^{B_1}]\times[1:2^{B_2}]\mapsto \R^{d_2\times d_1}$, where $\R^{d_2\times d_1}$ is the set of $d_2\times d_1$--matrices. The decoder maps $(M_1,M_2)$ to $\bCh_{21}$. Again, the objective is to characterize the minimax distortion, defined as:
\begin{equation}\label{eq:minimax-distortion-cross}
    \MEc(\sigma,m,d_{1:2},B_{1:2}):= \inf_{{\cE_1,\cE_2,\cD}} \sup_{P\in\cP}\E\left[\cL_{\dist}(\bCh_{21},\bC_{21})\right].
\end{equation}

\begin{figure}[t]
    \centering
    \begin{tikzpicture}[scale=0.55, yscale=1.5]
        \node at (-1,0) (dataset) {$\left\{\bZ^{(i)} = \begin{bmatrix}
             \bX_1^{(i)} \\
             \bX_2^{(i)} \\
             \vdots \\
             \bX_K^{(i)}
        \end{bmatrix}\right\}_{i=1}^{m}$};
        \node at (5,3) (machine1) {$\left\{\bX_1^{(i)}\right\}_{i=1}^{m}$};
        \node at (5,1) (machine2) {$\left\{\bX_2^{(i)}\right\}_{i=1}^{m}$};
        \node at (5,-1) () {$\vdots$};
        \node at (5,-3) (machineK) {$\left\{\bX_K^{(i)}\right\}_{i=1}^{m}$};
        \node[rectangle, draw=green!50!blue, fill=yellow!50!white] at (8,3) (enc1) {$\cE_1$};
        \node[rectangle, draw=green!50!blue, fill=yellow!50!white] at (8,1) (enc2) {$\cE_2$};
        \node at (8,-1) () {$\vdots$};
        \node[rectangle, draw=green!50!blue, fill=yellow!50!white] at (8,-3) (encK) {$\cE_K$};
        \node at (10,3) (M1) {$M_1$};
        \node at (10,1) (M2) {$M_2$};
        \node at (10,-1) () {$\vdots$};
        \node at (10,-3) (MK) {$M_K$};
        \node[rectangle, draw=green!50!blue, , fill=yellow!50!white] at (12,0) (dec) {$\cD$};
        \node at (14,0) (Ch) {$\bCh$};
        \draw[-latex, thick, draw=blue] (dataset)--(machine1);
        \draw[-latex, thick, draw=blue] (machine1)--(enc1);
        \draw[-latex, thick, draw=blue] (enc1)--(M1);
        \draw[-latex, thick, draw=blue] (M1)--(dec);
        \draw[-latex, thick, draw=blue] (dataset)--(machine2);
        \draw[-latex, thick, draw=blue] (machine2)--(enc2);
        \draw[-latex, thick, draw=blue] (enc2)--(M2);
        \draw[-latex, thick, draw=blue] (M2)--(dec);
        \draw[-latex, thick, draw=blue] (dataset)--(machineK);
        \draw[-latex, thick, draw=blue] (machineK)--(encK);
        \draw[-latex, thick, draw=blue] (encK)--(MK);
        \draw[-latex, thick, draw=blue] (MK)--(dec);
        \draw[-latex, thick, draw=blue] (dec)--(Ch);
\end{tikzpicture}
    \caption{Setting of the problem $\DCME(\sigma,m,B_{1:K},d_{1:K})$. $\bZ\in\R^d$ is a $\sigma$--sub--Gaussian random vector with covariance matrix $\bC$. for each $k\in[K]$, $\bX_k$ contains the dimensions $[1+\sum_{i=1}^{k-1}d_i:\sum_{i=1}^{k}d_i]$ of $\bZ$. The $\bCh$ is an estimate of $\bC$, subject to the constraint that $H(M_k)\leq B_k$.}
    \label{fig:problem_statement}
\end{figure}
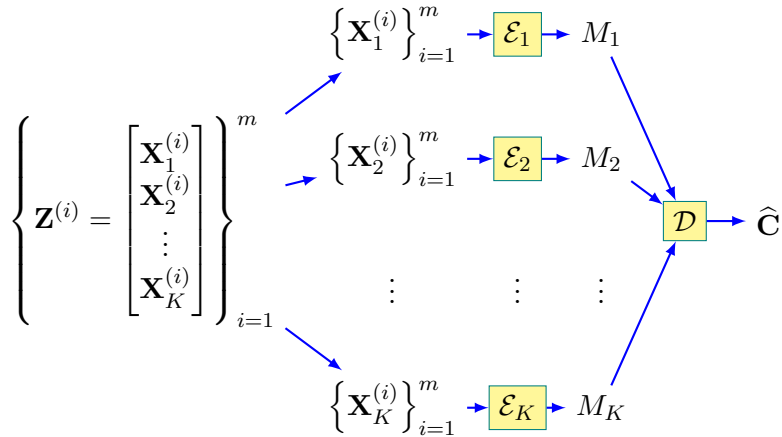

\subsection{Minimax Lower Bounds for the Expected Distortion of a $\DCME$ Problem}
We first establish lower bounds for the case $K=2$ and subsequently derive a lower bound for the general $K$ from the two--user scenario.

We define:
\begin{equation}
\begin{split}
    \alphaopcc&:=\sqrt{\frac{d_1(d_1\vee d_2)}{2B_{1}}}\bigvee\sqrt{\frac{d_2(d_1\vee d_2)}{2B_{2}}}, \\
    \alphaopsc&:=\sqrt{\frac{d}{3m}},\\
    \alphaFcc&:=\sqrt{\frac{d_1d_2}{14}(\frac{d_1}{B_1}\vee\frac{d_2}{B_2})},\\
    \alphaFsccross&:= \sqrt{\frac{dd_{\min}}{42m}},\\
    \alphaFsc&:= \sqrt{\frac{d^2}{42m}},\\
    \alphaFccself&:=\frac{4}{7}\left( \sqrt{d_1}\cdot 2^{\frac{-16B_1}{d_1^2}}\bigvee \sqrt{d_2}\cdot 2^{\frac{-16B_2}{d_2^2}}\right),
\end{split}
\end{equation}

The minimax errors of the $\DCME$ and $\DCCME$ problems (denoted by $\ME$ and $\MEc$) can be expressed using the previously defined quantities. In these definitions, $\op$ refers to the operator norm, and $\Fr$ to the Frobenius norm. The terms $\msf{cc}$ and $\msf{sc}$ represent communication complexity and sample complexity, respectively. 

Accordingly, $\alphaopcc$ quantifies the impact of communication constraints on the minimax error under the operator norm, while $\alphaopsc$ captures the effect of having a limited number of samples under the same norm. These quantities provide lower bounds on $\MEop$ and $\MEopc$.

Similarly, $\alphaFcc$ reflects the influence of communication constraints on the minimax error when measured using the Frobenius norm. The impact of limited samples differs between full covariance estimation and cross--covariance estimation. Specifically, $\alphaFsc$ represents the sample--induced error in estimating the full covariance matrix, while $\alphaFsccross$ pertains to the cross--covariance part. 

An additional error term, $\alphaFccself$, accounts for the effect of communication limitations on the estimation of the self--covariance matrix. This term arises only in the lower bound of $\MEF$.

\begin{thm}\label{thm:lower_bound-cross}
    Consider the $\DCCME(\sigma,m,d_{1:2},B_{1:2})$ problem. Then, $\MEc(\sigma,m,d_{1:2},B_{1:2})$ is lower--bounded as follows: 
    \begin{equation}\label{eqn:ya-LB-op-cr}
    \begin{split}
        \MEopc &\geq \frac{\sigma^2}{32}\left(\left(\alphaopcc\bigvee \alphaopsc \right)\bigwedge2\right)\\
        &=          {\sigma^2}\Omega\left(\left(\sqrt{d\left(\frac{d_1}{B_1}\vee\frac{d_2}{B_2}\right)}\bigvee \sqrt{\frac{d}{m}} \right)\bigwedge1\right),
    \end{split}
    \end{equation}
    and:
    \begin{equation}\label{eqn:ya-LB-Fr-cr}
    \begin{split}
        \MEFc &\geq \frac{\sigma^2}{32}\left(\left(\alphaFcc \bigvee \alphaFsccross\right) \bigwedge\frac{\sqrt{d_{\min}}}{7}\right)\\
        &= {\sigma^2}\Omega\left(\left(\sqrt{{d_1d_2}(\frac{d_1}{B_1}\vee\frac{d_2}{B_2})} \bigvee \sqrt{\frac{dd_{\min}}{m}}\right) \bigwedge{\sqrt{d_{\min}}}\right).
    \end{split}
    \end{equation}
\end{thm}  

\begin{thm}\label{thm:lower_bound}
    Consider the $\DCME(\sigma,m,d_{1:2},B_{1:2})$ problem. Then, $\ME(\sigma,m,d_{1:2},B_{1:2})$ is lower--bounded as follows: 
    \begin{equation}\label{eqn:ya-LB-op}
    \begin{split}
        \MEop\geq\MEopc &\geq \frac{\sigma^2}{32}\left(\left(\alphaopcc \bigvee \alphaopsc\right)\bigwedge2\right)\\
        &= {\sigma^2}\Omega\left(\left(\sqrt{d\left(\frac{d_1}{B_1}\vee\frac{d_2}{B_2}\right)}\bigvee \sqrt{\frac{d}{m}} \right)\bigwedge1\right),
    \end{split}
    \end{equation}
    and:
    \begin{equation}\label{eqn:ya-LB-Fr}
    \begin{split}
        \MEF &\geq \frac{\sigma^2}{32}\left(\left(\left( \alphaFcc\bigvee \alphaFsc\right)\bigwedge \frac{\sqrt{d}}{7}\right)\bigvee \alphaFccself\right)\\
        &= {\sigma^2}\Omega\left(\left(\left( \sqrt{{d_1d_2}(\frac{d_1}{B_1}\vee\frac{d_2}{B_2})}\bigvee \sqrt{\frac{d^2}{m}}\right)\bigwedge {\sqrt{d}}\right)\bigvee \sqrt{d_1}\cdot 2^{\frac{-16B_1}{d_1^2}}\bigvee \sqrt{d_2}\cdot 2^{\frac{-16B_2}{d_2^2}}\right).
    \end{split}
    \end{equation}
\end{thm}
\begin{cor}\label{cor:c-s-L}
Any $\DCME$ (or $\DCCME$) scheme that approximate the covariance matrix within $\varepsilon$ expected--operator norm distortion has communication budget $B_k=\Omega\left(\sigma^4\frac{dd_k}{\varepsilon^2}\right)$ and sample complexity $m=\Omega\left(\sigma^4\frac{d}{\varepsilon^2}\right)$.
Furthermore, any $\DCME$ (or $\DCCME$) scheme that achieves expected Frobenius norm distortion at most $\varepsilon$ requires a communication budget of
$B_k = \Omega\left(\sigma^4 \frac{d_1 d_2 d_k}{\varepsilon^2} \right)$.
In addition, the sample complexity for any $\DCME$ scheme is
$m = \Omega\left(\sigma^4 \frac{d^2}{\varepsilon^2} \right)$,
while for any $\DCCME$ scheme it is
$m = \Omega\left(\sigma^4 \frac{d d_{\min}}{\varepsilon^2} \right)$.
\end{cor}

\begin{remark}\label{rem:samplecomplexity}
The lower bounds on the sample complexity of distributed covariance matrix estimation match those of centralized covariance matrix estimation. Clearly, the sample complexity in the distributed setting must be at least as large as that of the centralized setting, where all data are collected on a central server. The sample complexity lower bound for the operator norm is folklore, while the corresponding lower bound for the Frobenius norm is derived in \citep{ashtiani2020near,devroye2020minimax}.
\end{remark}

The complete proof for Theorem \ref{thm:lower_bound} is provided in detail in Section \ref{sec:Proof_lower_bound}. Within that section, we employ a specialized variant of Fano's method, known as the averaged Fano's method, to establish the theorem. The initial step of this averaged Fano's method involves reducing the estimation problem to a finite hypothesis testing problem, a technique also common in the standard Fano's method. This reduction utilizes a family of distributions $\{P_v\}_{v\in\cV}$ with associated covariance matrices $\{\bC_v\}_{v\in\cV}$, where $\cV$ is a set of finite cardinality. Subsequently, a minimax lower bound for the estimation can be derived, expressed in terms of the error probability of this hypothesis testing problem and the separation between the distinct $\bC_v$s, defined as $\rho_{\dist}:= \inf_{\substack{v,v'\in\cV\\ v\neq v'}} \big\{ \norm[0]{\bC_{v}-\bC_{v'}}_{\dist} \big\}$.

\subsection{Achievability  of the Minimax Lower bounds}
In this section, we propose a $\DCME(\sigma,m,d_{1:2}, B_{1:2})$ scheme and derive
lower bounds on both its sample complexity and communication budgets for $\varepsilon$--approximation of the covariance matrix  under $\|.\|_\dist$ norm. 
A  description of the achievable scheme is provided in Section \ref{sec:proof_achievable_scheme} with detailed  proofs is relegated to Appendix \ref{app:proof_thm_achievable_scheme}. Here, we outline the main idea.

Consider the decomposition bellow of matrix$\bC$:
\begin{equation}\label{eqn:CovDecomp}
    \overline{\bC} = \left[\begin{matrix}
        \overline{\bC}_{11} &  \overline{\bC}_{12}\\
        \overline{\bC}_{12}^\top & \overline{\bC}_{22}
    \end{matrix}\right], 
 \end{equation}
 where $\overline{\bC}_{11}=\overline{\bC}_{[1:d_1,1:d_1]}$, 
$\overline{\bC}_{12}=\overline{\bC}_{[1:d_1,d_1+1:d]}$, and $\overline{\bC}_{22}=\overline{\bC}_{[d_1+1:d,d_1+1:d]}$.
Agent 1 can estimate $\bC_{11}$ using the data points $\{\bX_1^{(i)}\}_{i=1}^{m}$, and Agent 2 can estimate $\bC_{22}$ using the data points $\{\bX_2^{(i)}\}_{i=1}^{m}$. Consequently, they allocate portions of their communication budgets to transmit quantized versions of $\bC_{11}$ and $\bC_{22}$ to the central server. The remaining portions of their communication budgets are then used to transmit quantized versions of $\{\bX_1^{(i)}\}_{i=1}^{m}$ and $\{\bX_2^{(i)}\}_{i=1}^{m}$ to the central server. The central server can then estimate $\bC_{12}$ using this received information and construct an estimate $\bCh$. We present the main result of this subsection in terms of communication budgets and sample complexity.

\begin{thm}\label{thm:achievable_scheme}
    Consider $\DCME (\sigma,m,d_{1:2},B_{1:2})$ and a permissible distortion $\varepsilon\leq\sigma^2$. Assume that the number of samples $m$ and the communication budgets $B_k$ satisfy the following constraints: 
    \begin{align*}
        m &\geq  \tau\frac{d\sigma^4}{\varepsilon^2}\\
        B_k& \geq \tau'\frac{\sigma^4 d_k d}{\varepsilon^2}\cdot \log\Big(\frac{\sigma^2}{\varepsilon}\Big),
    \end{align*}
    for $k=1,2$ and some constants $\tau,\tau'$. Then, there exists a scheme with expected distortions $\E\left[\cL_{\op}(\bCh,\bC)\right]\leq\varepsilon$ and $\E\left[\cL_{\op}(\bCh_{12},\bC_{12})\right]\leq\varepsilon$.
\end{thm}

\begin{thm}\label{thm:achievable_scheme-Fr}
   Consider $\DCME (\sigma,m,d_{1:2},B_{1:2})$ and assume that the number of samples $m$ and the communication budgets $B_k$ satisfy the following constraints:
   
    \begin{align*}
        m &\geq  \tau\frac{d^2\sigma^4}{\varepsilon^2}\\
        B_k& \geq \tau'\frac{\sigma^4 d_k dd_{\min}}{\varepsilon^2}\cdot \log_2\left(\frac{\sigma^2\sqrt{d_{\min}}}{\varepsilon}\right)\bigvee \tau'' d_k^2\log_2\left(\frac{\sigma^2\sqrt{d}}{\varepsilon}\right),
    \end{align*}
    for $k=1,2$ and some constants $\tau,\tau', \tau''$. Then, there exists a scheme with expected distortions $\E\left[\cL_{\Fr}(\bCh,\bC)\right]\leq\varepsilon$. Further if the numbers of sample and communication budget satisfy the following constraint, 
     \begin{align*}
        m &\geq  \tau\frac{\sigma^4d d_{\min} }{\varepsilon^2}\\
        B_k& \geq \tau'\frac{\sigma^4 d_k dd_{\min}}{\varepsilon^2}\cdot \log_2\left(\frac{\sigma^2\sqrt{d_{\min}}}{\varepsilon}\right)
        ,
    \end{align*}
    for $k=1,2$ and some constants $\tau,\tau', \tau''$. Then, there exists a scheme with expected distortions
    $\E\left[\cL_{\Fr}(\bCh_{12},\bC_{12})\right]\leq\varepsilon$.
\end{thm}

\begin{remark}
Comparing Corollary \ref{cor:c-s-L}, Theorem \ref{thm:achievable_scheme} and \ref{thm:achievable_scheme-Fr} yields that the lower bounds and upper bounds are matched up to a logarithmic factor. 
\end{remark}

\begin{remark}
Assume $d_2=1, B_2=\infty$, that is the agent 2 takes the role of central server. In this case, our setting corresponds to the scenario investigated in \cite[Theorem 4]{hadar2019distributed}. That result implies an upper bound on the Frobenius norm error, specifically $\cO\left(\frac{d}{\sqrt{B}}\right)$, under the premise of each agent possessing an infinite number of correlated samples. This bound is achieved via a novel estimator called {\em maximum estimator}. Consequently, we infer that the logarithmic factors present in Theorem 1 are dispensable, implying the tightness of our communication budget's lower bound without requiring an additional logarithmic term, contingent on the availability of infinite or sufficiently numerous samples. It is imperative to underscore that \cite{hadar2019distributed} confines its analysis to Gaussian random vectors and infinite samples, whereas our study encompasses the broader category of sub--Gaussian distributions and accommodates finite samples, thereby operating under more weaker assumptions.
\end{remark}

\begin{remark}[Centralized Covariance Matrix Estimation with Limited Communication Budget]
Consider the centralized covariance matrix estimation problem, a simplified version of the distributed case where only one agent is present. This agent can estimate the entire covariance matrix locally using a sample covariance estimator and then transmit a quantized version of this estimate to a central server.

This approach gives rise to the following lower bound on the expected operator norm distortion:
\[\E\left[\cL_{\op}(\bCh,\bC)\right] = \sigma^2\cO\left(\max\bigg\{\sqrt{\frac{d}{m}},\exp\left(\frac{-c.B}{d^2}\right)\bigg\}\right).\]
Crucially, this result demonstrates that the communication budget for the centralized case, which scales as $B=\cO\Big(d^2\log\big(\frac{1}{\varepsilon}\big)\Big)$, is substantially lower than the total communication budget required for the distributed case, which scales as $B=\Omega\left(\frac{d^2}{\varepsilon^2}\right)$.
\end{remark}

\subsection{Multi--Agent Scenario}
In this section, we analyze the multi agent scenario. Let the number of agents be denoted by $K>2$ and for all $k \in [K]$, Agent $k$ has access to the dimensions $[1 + \sum_{i=1}^{k-1}d_i : \sum_{i=1}^{k}d_i]$ of all $m$ samples $\bZ^{(1)}, \dots, \bZ^{(m)}$. The information available to Agent $k$ can be represented as the set $ \{\bX_k^{(i)}\}_{i=1}^{m} = \{\bZ_{[1 + \sum_{i=1}^{k-1}d_i : \sum_{i=1}^{k}d_i]}^{(i)}\}_{i=1}^{m}$. 

The central server aims to estimate $\bC$ by receiving up to $B_k$ bits of information from Agent $k$. To derive a lower bound in this context, we define a subset of agents, denoted as $\cS \subset [K]$. Then we assume that the agents within $\cS$ collude, and similarly, the agents in the complement set $\cS^{\msfc} = [K] \setminus \cS$ also collude. This leads to the formation of two super--agents, $A$ and $B$. The super--agent $A$ has access to $\sum_{k \in \cS} d_k$ dimensions and a communication budget of $\sum_{k \in \cS} B_k$ bits, while super--agent $B$ has access to $\sum_{k \in \cS^c} d_k$ dimensions and a communication budget of $\sum_{k \in \cS^{\msfc}} B_k$ bits. Theorem \ref{thm:lower_bound} provides a lower bound for this colluded scenario, which also serves as a lower bound for the non--colluded case. By maximizing these lower bounds over the possible choices of subset $\cS$, we conclude that:

\begin{equation}
\label{eqn:LB-MulA}
\begin{aligned}
    \MEop &= \sigma^2\Omega\left(\sqrt{d\cdot\max_{\substack{\cS\subset [K]\\ \cS\neq\emptyset}}\left\{\frac{\sum_{k\in\cS}d_k}{\sum_{k\in\cS}B_k}\right\}}\bigvee \sqrt{\frac{d}{m}}\right)\\
    &=\sigma^2\Omega\left(\sqrt{d\cdot\max_{k\in[K]}\frac{d_k}{B_k}\bigvee \frac{d}{m}}\right).
\end{aligned}
\end{equation}
Specifically, assuming that $K = d$ and that for all $k \in [K]$, $d_k = 1$, we get:
\begin{equation}
\begin{aligned}
    \MEop &= \sigma^2\Omega\left(\sqrt{\frac{d}{\min\limits_{k}B_k\wedge m}}\right).
\end{aligned}
\end{equation}
This result indicates that the performance of the $\DCME$ problem is determined by the agents with the smallest communication budgets. Also \eqref{eqn:LB-MulA} implies that the communication budget of agent $k$ for approximating the covariance matrix within $\varepsilon$ error in operator norm,  should satisfy $B_k=\Omega\left(\frac{dd_k}{\varepsilon^2}\right)$. 
   The following theorem (proof deferred to Appendix \ref{app:proof_thm_achievable_scheme_MulA}) shows that this bound is tight up to a logarithmic factor in $\log d$ and  $\frac{1}{\varepsilon} $.  \begin{thm}\label{thm:achievable_scheme-MulA}
    Consider $\DCME (\sigma,m,d_{1:K},B_{1:K})$ and a permissible distortion $\varepsilon\leq\sigma^2$. Assume that the number of samples $m$ and the communication budgets $B_k$ satisfy the following constraints: 
    \begin{align*}
        m &\geq  \tau\frac{d\sigma^4}{\varepsilon^2},&
        B_k \geq  \tau' \frac{\sigma^4dd_k}{{\varepsilon}^2}\log_2\left(\tau''\frac{\sigma^4}{\varepsilon^2}\log(d\sigma^2/\varepsilon)\right)
    \end{align*}
    for $k\in [1:K]$ and some constants $\tau,\tau',\tau''$. Then, there exists a scheme with expected distortions $\E\left[\cL_{\op}(\bCh,\bC)\right]\leq\varepsilon$.
\end{thm}

Comparing Theorem \ref{thm:achievable_scheme-MulA} with Theorem \ref{thm:achievable_scheme} for $K=2$, it is worth mentioning that to get achievability for general $K$, the logarithmic factor depends on $\log d$ in addition to $\frac{1}{\varepsilon}$.
\section{Proof of Lower Bounds via Conditional SDPI
}
\label{sec:Proof_lower_bound}
In this section, we present the proof of Theorems \ref{thm:lower_bound-cross} and \ref{thm:lower_bound}. We first establish the necessary preliminaries and then proceed with the proof, step by step.

\subsection{Averaged Fano Method}\label{sec:averaged_fano}
We employ a variant of Fano's method, which we term the averaged Fano's method, to lower bound $\ME(\sigma,m,d_{1:2},B_{1:2})$.
Generally, Fano's method reduces an estimation problem to a hypothesis testing problem and subsequently derives a lower bound for the latter through Fano's inequality.

More precisely, let $\cP$ be a family of distributions, and let $\theta:\cP\mapsto \Theta$ be a parameter of interest of the distributions in $\cP$ (e.g., mean, covariance, etc.) residing in a metric space $\Theta$ equipped with the metric $\|.\|$. The objective is to approximate $\theta(P)$ for an unknown $P\in\cP$ using a sample $X$ obtained from $P$. Let this approximation be denoted by $\hat{\theta}(X)$. Consider the set $\cV=\{1,2,\dots,|\cV|\}$ and define $\cP_{\cV} = \{P_1,\cdots,P_{|\cV|}\}$ as a subset of $\cP$. The set $\cP_{\cV}$ is termed $2\delta$--separated, if for each $i,j\in\cV$ with $i\neq j$, we have $\|\theta(P_i)-\theta(P_j)\|\geq 2\delta$. For a given $2\delta$--separated $\cP_{\cV}$, let $V$ be a uniform random variable drawn from $[1:|\cV|]$ and given $V=v$, let the random variable $X$ be a sample from $P_{v}$ and $\hat{\theta}(X)$ representing the corresponding approximation of $\theta(P_v)$. Thus, we have the Markov chain $V\mkv X\mkv \hat{\theta}$. Fano's method establishes the following lower bound on the minimax error of estimation $\hat{\theta}$ of $\theta$:
\begin{equation}\label{eq:Fano-original}
    \min_{\hat{\theta}:\cX\mapsto \Theta}\max_{P\in\cP}\E_P\left[\norm{\hat{\theta}(X)-\theta(P)}\right]\geq \delta\left(1-\frac{I(V;X)+\log 2}{\log|\cV|}\right).  
\end{equation}
We employ Fano's method in conjunction with the conditional SDPI to establish the lower bounds (Theorem \ref{thm:lower_bound-cross} and Theorem \ref{thm:lower_bound}). However, to render the computation of the mutual information in \eqref{eq:Fano-original} tractable, we utilize the following variant of Fano's method.
Here, rather than reducing the approximation to a single hypothesis testing problem, we consider multiple reductions and use the average of the Fano lower bounds as a lower bound for the minimax loss; hence, we refer to it as the averaged Fano method.
More precisely, let $W\sim\pi_W$ be a random variable taking values in $\cW$. For each instance $W=w$, assume that $\cP_{\cV}^{(w)}=\{P_1^{(w)},\cdots,P_{|\cV|}^{(w)}\}$ is a $2\delta$--separated subset of $\cP$. Let $V$ be a uniform random variable, as defined previously, independent of $W$. Given $V=v$ and $W=w$, the random variable $X$ is drawn from $P_v^{(w)}$. The standard Fano lower bound \eqref{eq:Fano-original} then yields:
\begin{equation}\label{eq:Fano-averaged}
    \begin{split}
        \min_{\hat{\theta}:\mathcal{X}\mapsto \Theta}\max_{P\in\mathcal{P}}\E_P\left[\norm{\hat{\theta}(X)-\theta(P)}\right]&\geq \delta\sup_{w\in\cW}\left(1-\frac{I(V;X|W=w)+\log 2}{\log|\cV|}\right)\\
        &\geq \delta\left(1-\frac{I(V;X|W)+\log 2}{\log|\cV|}\right)\\
        &\geq \delta\left(1-\frac{I(V,W;X)+\log 2}{\log|\cV|}\right).
    \end{split}
\end{equation}
The averaged Fano method has been previously employed in \cite[Example 15.19]{wainwright2019high} to derive a concise minimax bound for the PCA problem.

\subsection{Averaged Fano's Method for Covariance Estimation}
Let $W\sim \pi_W$ be a random variable taking value in $\cW$. For each $w\in\cW$, we consider a  family of distributions $\cP_{\cV}^{(w)}=\{P_v^{(w)}\}_{v\in\cV}\subset\subG^{(d)}(\sigma)$ indexed by a finite set $\cV=[1:|\cV|]$. For each $(w,v)$, let $\bC_v^{(w)}:=\E_{\bZ\sim P_v^{(w)}}[(\bZ-\E[\bZ])(\bZ-\E[\bZ])^\top]$ denotes the corresponding covariance matrix. Further, let $\bX_1=\bZ_{[1:d_1]}$ and $\bX_2=\bZ_{[d_1+1:d]}$, represent the partitioning of the $d$--dimensional vector $\bZ$ into a $d_1$--dimensional vector $\bX_1$ and a $d_2$--dimensional vector $\bX_2$. For each $(w,v)$, let $\bD_v^{(w)}:=\bC_{v,21}^{(w)}=\E_{\bZ\sim P_v^{(w)}}[(\bX_2-\E[\bX_2])(\bX_1-\E[\bX_1])^\top]$ denotes the corresponding {\emph cross}--covariance matrix. For this set, we define the separations $\rho$ and $\rho^{(\cross)}$ with respect to the $\dist$ norm metric on the space of covariance matrices as:
\begin{equation}
\begin{split}
    \rho_{\dist} &:= \inf_{\substack{w\in\cW\\(v,v')\in\cV^2,v\neq v'}}\left\{\norm[1]{\bC_{v}^{(w)}-\bC_{v'}^{(w)}}_{\dist}\right\},\\
    \rho_{\dist}^{(\cross)} &:= \inf_{\substack{w\in\cW\\(v,v')\in\cV^2, v\neq v'}}\left\{\norm[1]{\bD_{v}^{(w)}-\bD_{v'}^{(w)}}_{\dist}\right\}.
\end{split}
\end{equation}
We now state Lemma \ref{lem:main_fano}, a direct consequence of the averaged Fano's method \eqref{eq:Fano-averaged}:

\begin{lem}[]\label{lem:main_fano}
    Consider a collection of $\rho_\dist$--separated families of distributions $\cP_{\cV}^{(w)}$ under the $\dist$--norm on covariance matrices, each family consisting of $|\cV|$ distributions $P_v^{(w)}$.  Assume a random variable $V\in\cV$ is chosen uniformly and independently of $W$, and  given $(W,V)=(w,v)$, samples $\{\bZ^{(i)}\}_{i=1}^{m}$ are drawn i.i.d. from $P_v^{(w)}$. Additionally, assume that agents 1 and 2 have access to $\{\bX_1^{(i)}=\bZ_{[1:d_1]}^{(i)}\}_{i=1}^{m}$ and $\{\bX_2^{(i)}=\bZ_{[d_1+1:d]}^{(i)}\}_{i=1}^{m}$, respectively. For any $\DCME$ (respectively, $\DCCME$) scheme with parameters $(\sigma,m,d_{1:2},B_{1:2})$, we have:
    \begin{equation*} 
    \begin{split}
        \inf\limits_{\cE_1,\cE_2,\cD}\sup\limits_{P\in\cP}\E\left[\cL_{\dist}(\bCh,\bC)\right]\geq \dfrac{\rho_{\dist}}{2}\left(1-\dfrac{I(W,V;M_1,M_2)+\log 2}{\log|\cV|}\right),\\
        \inf\limits_{\cE_1,\cE_2,\cD}\sup\limits_{P\in\cP}\E\left[\cL_{\dist}(\bCh_{21},\bC_{21})\right]\geq \dfrac{\rho_{\dist}^{(\cross)}}{2}\left(1-\dfrac{I(W,V;M_1,M_2)+\log 2}{\log|\cV|}\right).
    \end{split}
    \end{equation*}
\end{lem}
We now proceed by first establishing two lemmas that utilize Lemma \ref{lem:main_fano} with specific distribution families $\{P_v^{(w)}\}_{v\in\cV,w\in\cW}$. These lemmas will subsequently pave the way for deriving the main theorem concerning the minimax lower bound.
 
\subsection{Construction of $\rho_\dist$--Separated Families for Cross--Covariance}\label{subse:family}
Consider a set $\cV=[1:|\cV|]$ and a corresponding family of distributions $\cP_{\cV}^{(w)}=\{P_v^{(w)}\}_{v\in\cV}$, where $P_v^{(w)} = \cN(\bzero, \bC_{v}^{(w)})$, and:
\begin{equation}\label{eq:matrix-representation}
    \bC_{v}^{(w)} = \frac{\sigma^2}{2}\begin{bmatrix}
         \bI_{d_1} &  \delta(\bD_v^{(w)})^\top\\
        \delta\bD_v^{(w)} &  \bI_{d_2}
    \end{bmatrix},
\end{equation}
where  $\bD_v^{(w)}$ is a matrix in $\R^{d_2\times d_1}$ with $\norm[1]{\bD_{v}^{(w)}}_\op\leq 1$, and $\delta\leq 1$ is a parameter to be determined subsequently. 
Note that $\bC_{v}^{(w)}$ represents the covariance matrix of a $\sigma$--sub--Gaussian random vector $\bZ$. Consequently, we must have $\bC_{v}\succeq \bzero$. Furthermore, from Definition \ref{def:sub_gaussian_random_vector}, for all vectors $\bu$ with $\norm{\bu}_2=1$, the random variable $\bu^\top\bZ$ is $\sigma$--sub--Gaussian; therefore, $\var[\bu^\top\bZ]\leq\sigma^2$. This implies that for all $\bu$ with $\norm{\bu}_2=1$:
\begin{equation}
    \var[\bu^\top\bZ] = \E[\bu^\top\bZ\bZ^\top\bu] = \bu^\top\bC_{v}^{(w)}\bu \leq\sigma^2.
\end{equation}
Therefore, we must have $\norm{\bC_{v}^{(w)}}_{\op}\leq\sigma^2$.

We can express $\bC_{v}^{(w)}$ as $\bC_{v}^{(w)}=\frac{\sigma^2}{2}\bI + \frac{\sigma^2}{2}\left[\begin{matrix}
    \bzero &  \delta(\bD_{v}^{(w)})^\top\\
    \delta\bD_{v}^{(w)} & \bzero
\end{matrix}\right]$. From Lemma \ref{lem:eigenvals_and_eigenvects_of_matrix}, the eigenvalues of $\bC_{v}^{(w)}$ are given by $\frac{\sigma^2}{2}\left(1\pm\delta\sigma_i(\bD_{v}^{(w)})\right)$. Therefore, if we assume that $\norm[1]{\bD_{v}^{(w)}}_\op\leq 1$ and $\delta\leq 1$, the constraints $\bC_{v}^{(w)}\succeq\bzero$ and $\norm{\bC_{v}^{(w)}}_{\op}\leq\sigma^2$ are satisfied, ensuring that $P_v^{(w)}\in \subG^{(d)}(\sigma)$.

This choice of distributions exhibits the following properties:
\begin{itemize}
    \item We have $\rho_\op=\rho^{(\cross)}_\op$ and $\rho_{\Fr}=\sqrt{2}\rho^{(\cross)}_\Fr$. Consequently, it suffices to derive a lower bound on the minimax error for cross--covariance matrix estimation.
    \item  
    More importantly, the vector $\bZ\sim\cN(\bzero,\bC_{v}^{(w)})$ has the same marginal distribution over the first $d_1$ dimensions and the second $d_2$ dimensions, for all $(w,v)$. Therefore, $\bmsfX_1=\left\{\bX_1^{(i)}\right\}_{i=1}^{m}$ is independent from $(W,V)$. Similarly, $\bmsfX_2=\left\{\bX_2^{(i)}\right\}_{i=1}^{m}$ is independent from $(W,V)$. Subsequently, $M_1$ (and similarly $M_2$) is also independent from $(W,V)$. This implies: 
    \begin{equation}\label{eq:mutual_information_bound1_app}
		\begin{aligned}
			I(V,W;M_1,M_2) &= I(V,W;M_1) + I(V,W;M_2\mid M_1)\\
            &{=} I(V,W;M_2\mid M_1)&\text{since $M_1\indep (W,V)$}\\
            &\leq I(V,W,M_1;M_2)\\
            &{=} I(M_1;M_2\mid V,W)&\text{since $M_2\indep (W,V)$}.\\
		\end{aligned}
	\end{equation}
\end{itemize}

\subsection{Applying the Conditional SDPI}
We now derive an upper bound on $I(M_1;M_2|W,V)$ using conditional SDPI.  Observe that conditioned on any $(W,V)=(w,v)$, the structures of encoders and decoder impose the following Markov chain: $M_1 \mkv \bmsfX_1  \mkv \bmsfX_2  \mkv M_2$. Furthermore, $(M_1,\bmsfX_1)$ is independent of $(W,V)$; thus, the constraints in the definition of conditional SDPI are satisfied. By the data processing inequality, we have:
\begin{equation}\label{eq:cmi-two-terms}   
    I(M_1;M_2|V,W)\leq I(M_1;\bmsfX_2|V,W)\wedge I(M_2;\bmsfX_1|V,W).
\end{equation}
Now, conditioned on $(W,V)=(w,v)$, $(\bmsfX_1,\bmsfX_2)$ is sampled from the normal distribution $\cN(\bzero, \bC_{v}^{(w)})^{\otimes m}$. Thus, the conditional SDPI constant for mixture of Gaussian (Theorem \ref{thm:csdp-mix}) and the tensorization property of conditional SDPI constant (Theorem \ref{thm:tensorization}) yield:
\begin{equation}\label{eq:agent1-sdpi}
\begin{aligned}
  I(M_1;\bmsfX_2|V,W)&\leq \delta^2\norm{\E_{(W,V)}\left[(\bD_{V}^{(W)})^\top \bD_{V}^{(W)}\right]}_\op I(M_1;\bmsfX_1)\\
  &\leq \delta^2\norm{\E_{(W,V)}\left[(\bD_{V}^{(W)})^\top \bD_{V}^{(W)}\right]}_\op B_1,
\end{aligned}
\end{equation}
where we have used the fact that $I(M_1;\bmsfX_1)\le H(M_1)\leq B_1$.
Similarly, we have:
\begin{equation}\label{eq:agent2-sdpi}
\begin{aligned}
  I(M_2;\bmsfX_1|V,W)&\leq \delta^2\norm{\E_{(W,V)}\left[ \bD_{V}^{(W)}(\bD_{V}^{(W)})^\top\right]}_\op  I(M_2;\bmsfX_2) \\
  &\leq \delta^2\norm{\E_{(W,V)}\left[ \bD_{V}^{(W)}(\bD_{V}^{(W)})^\top\right]}_\op  B_2.
\end{aligned}
\end{equation}

\subsection{Evaluating the Conditional SDPI Constant Using Random Signed Permutation Matrices}
Up to this point,  $W$ has not played a specific role, and all preceding steps could have been performed without it. However, the primary challenge in what follows is to obtain a concise upper bound on the conditional SDPI constant $\norm[2]{\E_{(W,V)}\left[ \bD_{V}^{(W)}(\bD_{V}^{(W)})^\top\right]}_\op$. Achieving this without the introduction of $W$ would be arduous, if not infeasible.

For the present analysis, we assume $d_1\geq d_2$. The complementary case follows by symmetry. We note that there exist $2^{d_1} d_1!$ distinct signed permutation matrices in $\R^{d_1}$; thus, we can impose an ordering on these matrices, denoting them as $\left\{\bA_j\right\}_{j=1}^{2^{d_1} d_1!}$. Let $W$ be a random variable taking values uniformly at random in the set  $\{1,2,\dots,2^{d_1} d_1!\}$. Further, let $\cP_{\cV}=\{P_v\}_{v\in\cV}$ be a $\rho_\dist$--separated set of normal distributions, such that $P_v = \cN(\bzero, \bC_{v})$ with:
\begin{equation}
    \bC_{v} = \frac{\sigma^2}{2}\begin{bmatrix}
         \bI_{d_1} &  \delta\bD_{v}^\top\\
        \delta\bD_{v} &  \bI_{d_2}
    \end{bmatrix},
\end{equation} 

where $\norm[1]{\bD_v}_\op\leq 1$.

Now, for each $w\in\{1,2,\dots,2^d_1 d_1!\}$, let $\bD_v^{(w)}=\bD_v \bA_{w}$ in the matrix representation \eqref{eq:matrix-representation} and $\cP_{\cV}^{(w)}=\{P_v^{(w)}\}_{v\in\cV}$. Given that any signed permutation matrix is a unitary matrix, the distance (with respect to either the operator norm or the Frobenius norm) between any two corresponding matrices in different families $\cP_{\cV}^{(w)}$ and $\cP_{\cV}^{(w')}$ are identical, that is, $\norm[1]{P_v^{(w)}-P_{v'}^{(w)}}_\dist=\norm[1]{P_v^{(w')}-P_{v'}^{(w')}}_\dist$. Consequently, all sets $\cP_{\cV}^{(w)}$ are $\rho_\dist$--separated.
Now, by Lemma \ref{lem:SPM}, the following identity holds for any matrix $\bD_v$ satisfying $\norm{\bD_v}_{\op}\leq 1$:
\begin{equation}\label{eq:1CSDPIcomp}
\begin{split}
    \E_W \left[ (\bD_v^{(W)})^\top D_v^{(W)}\right] &= \E_W \left[\bA_{W}^\top \bD_v^\top \bD_v \bA_{W}\right]\\
    &= \frac{1}{d_1} \Tr{\bD_v^\top \bD_v}\bI_{d_1}\\
    &=\frac{1}{d_1}\norm{\bD_v}_{\Fr}^2 \bI_{d_1}\\
    &\preceq  \frac{d_1\wedge d_2}{d_1}\bI_{d_1}.
\end{split}
\end{equation}
Next, consider:
\begin{equation}\label{eq:2CSDPIcomp}
    \E_W\left[ \bD_v^{(W)} (\bD_v^{(W)})^\top\right] = \bD_v^\top \bD_v \preceq  \bI_{d_2}=\frac{d_1\wedge d_2}{d_2}\bI_{d_2}.
\end{equation}

In summary, combining \eqref{eq:mutual_information_bound1_app}, \eqref{eq:cmi-two-terms}, \eqref{eq:agent1-sdpi}, \eqref{eq:agent2-sdpi}, \eqref{eq:1CSDPIcomp}, and \eqref{eq:2CSDPIcomp} yields:
\begin{equation}\label{eq:summary-MI}
\begin{split}
    I(V,W;M_1,M_2) &\leq I(M_1;M_2\mid V,W)\\
    &\leq I(M_1;\bmsfX_2|V,W)\wedge I(M_2;\bmsfX_1|V,W)\\
    &\leq \delta^2 \left(\norm{\E_{(W,V)}\left[(\bD_{V}^{(W)})^\top \bD_{V}^{(W)}\right]}_\op B_1\right) \bigwedge \left(\norm{\E_{(W,V)}\left[ \bD_{V}^{(W)}(\bD_{V}^{(W)})^\top\right]}_\op  B_2\right)\\
    &\leq \delta^2(d_1\wedge d_2)\left(\frac{B_1}{d_1}\bigwedge \frac{B_2}{d_2}\right).
\end{split}
\end{equation}

\subsection{Packing Set for the Operator--Norm Unit Ball with Respect to $\dist$--Norm.}\label{se:packing}
The subsequent step is to determine a lower bound on the cardinality of the $\rho_\dist$--separated set $\cP_{\cV}$ defined in the preceding subsection. In Appendix \ref{app:pack_cover_matrix} we introduce the $\normiii{.}_{\dist}$ norm of a vectorized matrix and discuss the packing and covering sets of the unit $\normiii{.}_{\op}$ ball of matrices under the $\dist$ norm. We define the set $\{\bD_{v}\}_{v\in\cV}$ as the $\epsilon$--packing points of $\cB_{\normiii{.}_{\op}}^{(d_1d_2)}(1)$ (see Equation \eqref{eq:norm_op_ball}), under $\normiii{.}_{\dist}$ norm. Thus, $\inf\limits_{v,v':v\neq v'}\norm{\bD_v-\bD_{v'}}_{\dist} \geq \epsilon$, $\max\limits_{v\in\cV}\{\norm{\bD_{v}}_{\op}^2\} \leq 1$, and from  \eqref{eq:pack_op_ball_op_norm} and \eqref{eq:pack_op_ball_Fr_norm}, we have
$\log_2(|\cV|) \geq d_1d_2\log_2\left(\dfrac{\nu_{\dist}^{(d_1,d_2)}}{\epsilon}\right)$, where $\nu_{\dist}^{(d_1,d_2)} = 1$ if $\dist=\op$ and $\nu_{\dist}^{(d_1,d_2)} = \frac{\sqrt{d_1\wedge d_2}}{14}$ if $\dist=\Fr$.
We set $\epsilon=\dfrac{\nu_{\dist}^{(d_1,d_2)}}{4}$ which yields  $\log_2|\cV|\geq 2d_1d_2$. Furthermore, we note that the set $\cP_{\cV}$ corresponding to the packing $\{\bD_v\}_{v\in\cV}$ is $\rho_\dist$--separated with:
\begin{equation}\label{eq:packing-unit-op-ball}
\begin{aligned}
    \rho_{\op}&=\rho_\op^{(\cross)}=\delta\sigma^2 \dfrac{\nu_\op^{(d_1,d_2)}}{4} \\
    \rho_{\Fr}&=\sqrt{2}\rho_\Fr^{(\cross)}=\delta\sigma^2 \dfrac{\nu_\Fr^{(d_1,d_2)}}{2\sqrt{2}}.
\end{aligned}
\end{equation}
Setting $\delta^2=\left(\frac{d_1\vee d_2}{4}\left(\frac{d_1}{B_1}\bigvee\frac{d_2}{B_2}\right)\right)\bigwedge 1$, and incorporating \eqref{eq:summary-MI}, \eqref{eq:packing-unit-op-ball}, and the inequality $\log_2|\cV|\ge 2d_1d_2$ into Lemma \ref{lem:main_fano}, we obtain:
\begin{equation}
\begin{aligned}
    \MEop&\geq\frac{\sigma^2}{32}\left( \alphaopcc \bigwedge 2\right)&&\MEop^{(\cross)}\geq\frac{\sigma^2}{32}\left( \alphaopcc\bigwedge 2\right)\\
    \MEF&\geq\frac{\sigma^2}{32}\left( \alphaFcc\bigwedge\frac{\sqrt{d_1\wedge d_2}}{7}\right)&&\MEF^{(\cross)}\geq\frac{\sigma^2}{32}\left( \alphaFcc\bigwedge\frac{\sqrt{d_1\wedge d_2}}{7}\right)
\end{aligned}
\end{equation} 

In Theorem \ref{thm:lower_bound-cross} and Theorem \ref{thm:lower_bound}, there exist additional lower bounds pertaining to sample complexity and the limited communication budget for self--covariance estimation. The proofs for these particular lower bounds are deferred to Appendix \ref{app:proof_completion_lower_bound}.


\section{Achievable Scheme: Proof Sketch of Theorem \ref{thm:achievable_scheme}}
\label{sec:proof_achievable_scheme}

In this section, we present a near--optimal achievable Distributed Covariance Matrix Estimation ($\DCME$) scheme and establish an upper bound on its expected distortion. The proposed scheme operates in two distinct phases: one dedicated to approximating the self--covariance matrices $\bC_{11}$ and $\bC_{22}$, and another for the cross--covariance matrix $\bC_{12}$ (see \eqref{eqn:CovDecomp}).

\paragraph{Mean Invariance}
Should the observed random vectors $\{\bZ^{(i)}\}_{i=1}^{m}$ exhibit a non--zero mean, they can be transformed by defining $\bZ^{'(i)} = \frac{1}{\sqrt{2}}(\bZ^{(2i-1)} - \bZ^{(2i)})$. This redefinition ensures that the new vectors possess a zero mean while retaining the identical covariance matrix as the original $\bZ^{(i)}$. Consequently, the set of transformed samples $\{\bZ^{'(i)}\}_{i=1}^{m/2}$ can be equivalently employed in place of the initial samples $\{\bZ^{(i)}\}_{i=1}^{m}$. Therefore, for the purpose of the subsequent analysis, it is permissible to assume, without loss of generality, that $\mathbb{E}[\bZ]=0$.

\paragraph{Empirical Estimation of Self--Covariance Matrices}
Each agent is capable of estimating its respective self--covariance matrix directly from its local data by employing an empirical covariance estimator. Specifically, Agent 1 computes its estimate of $\bC_{11}$ as $\bCt_{11} = \frac{1}{m}\sum_{i=1}^{m}\bX_1^{(i)}\bX_1^{(i)\top}$, while Agent 2 similarly estimates $\bC_{22}$ using $\bCt_{22} = \frac{1}{m}\sum_{i=1}^{m}\bX_2^{(i)}\bX_2^{(i)\top}$.

\paragraph{Quantization of Estimated Self--Covariance Matrices} 
The empirical self--covariance matrix $\bCt_{11}$ lies within the ball $\cB_{\normiii{.}_\op}^{d_1^2}(\tau\sigma^2)$ with high probability for some constant $\tau>0$. To quantize it, Agent 1 finds an $\epsilon$--covering of this ball with $2^{{B_1}/{2}}$ points with smallest possible $\epsilon$. If the empirical estimate $\bCt_{11}$ lies within the ball $\cB_{\normiii{.}_\op}^{d_1^2}(\tau\sigma^2)$, Agent 1 quantizes $\bCt_{11}$ to $B_1/2$ bits by selecting the nearest point in the covering to the empirical estimate. If $\bCt_{11}$ lies outside the ball, Agent 1 declares an error. Agent 2 performs a similar quantization of its empirical estimate.

\paragraph{Quantization of Estimated Self--Covariance Matrices}
The empirical self--covariance matrix 
$\bCt_{11}$ is expected to reside within the ball $\mathcal{B}_{\normiii{.}_\op}^{d_1^2}(\tau\sigma^2)$ with high probability, for some positive constant $\tau$. To quantize this matrix, Agent 1 determines an $\epsilon$--covering of this ball, comprising $2^{{B_1}/{2}}$ points, such that $\epsilon$ is minimized. If the empirical estimate $\bCt_{11}$ falls within this ball, Agent 1 quantizes $\bCt_{11}$ to $B_1/2$ bits by selecting the nearest point from this covering to its empirical value. Conversely, if $\bCt_{11}$ lies outside this specified ball, Agent 1 signals an error. Agent 2 executes an analogous quantization procedure for its own empirical estimate.

\paragraph{Quantization of Data for Approximating the Cross--Covariance}
To approximate the cross--covariance, we first select a subset size $n = \min\left\{\min\{\frac{B_1}{d_1},\frac{B_2}{d_2}\}/2\log_2\left(\frac{6912\sigma^2}{\varepsilon}\right),m\right\}$. We then define data matrices $\bmsfX_1 \in \R^{d_1 \times n}$ and $\bmsfX_2 \in \R^{d_2 \times n}$ by concatenating the first $n$ samples from each agent, such that $\bmsfX_1 = [\bX_1^{(1)}, \bX_1^{(2)}, \dots, \bX_1^{(n)}]$ and $\bmsfX_2 = [\bX_2^{(1)}, \bX_2^{(2)}, \dots, \bX_2^{(n)}]$. The \emph{empirical} estimator for $\bC_{12}$ is then computed using these $n$ samples as $\bCt_{12} = \frac{1}{n} \bmsfX_1 \bmsfX_2^\top$.  For communication, Agent 1 quantizes its entire block of data, $\bmsfX_1$, to $B_1/2$ bits. Agent 2 performs a symmetrical quantization on its data block. It is established that $\bmsfX_1$ is highly likely to reside within the ball $\cB_{\normiii{.}_\op}^{nd_1}(\tau\sigma\sqrt{d_1+n})$. Agent 1 quantizes $\bmsfX_1$ by finding an $\epsilon$--covering of this ball with $2^{{B_1}/{2}}$ points, aiming for the smallest possible $\epsilon$. If $\bmsfX_1$ lies within this ball, Agent 1 selects the closest point $\bmsfXh_1$ from the covering to represent its quantized data, using $B_1/2$ bits. Otherwise, Agent 1 signals an error. Similarly, Agent 2 determines a quantized representation $\bmsfXh_2$ for $\bmsfX_2$ utilizing $B_2/2$ bits.

\paragraph{Cross--Covariance Estimation at the Central Server}
Upon receiving the quantized data $\bmsfXh_1$ and $\bmsfXh_2$, the central server initially estimates $\bC_{12}$ as $\bCh_{12} = \frac{1}{n} \bmsfXh_1 \bmsfXh_2^\top$. Should an error signal be received from either agent, the central server outputs a zero matrix, $\bCh = \bzero$. Otherwise, it proceeds to compute the composite matrix:
\begin{equation}
    \bCh^{\ast} = \begin{bmatrix}
        \bCh_{11} & \frac{1}{n} \bmsfXh_1 \bmsfXh_2^\top \\
        \frac{1}{n} \bmsfXh_2 \bmsfXh_1^\top & \bCh_{22}
    \end{bmatrix}.
\end{equation}
If $\bCh^{\ast}$ is not positive semi--definite, the central server adjusts it to ensure this property. This adjustment is performed by spectrally decomposing $\bCh^{\ast}$ as $\bCh^{\ast} = \sum_{i=1}^{r} \lambda_i \bv_i \bv_i^\top$, and then defining $\bCh^{\ast}_{+}$ by retaining only the non--negative eigenvalues: $\bCh^{\ast}_{+}=\sum_{i=1}^{r}\lambda_i\mathbbm{1}_{\{\lambda_i\geq 0\}}\bv_i\bv_i^\top$. The final estimated covariance matrix returned by the central server is then:
\begin{equation}\label{eq:ach-estimator}
  \bCh=\bCh_+^\ast.  
\end{equation}
The analysis of our $\DCME$ scheme relies on concentration inequalities for random matrices, inspired by but not identical to those in \citep{vershynin2018high} (see Appendix \ref{app:material_useful_achievable} for more details). 
The full proof of the scheme appears in Appendix \ref{app:proof_thm_achievable_scheme}.
\section{Interactive Cross-Covariance Estimation}
\label{sec:Interactive_Case}
In this section, we study the $\DCME$ problem in an interactive setting, which generalizes the interactive correlation estimation explored in \cite{hadar2019communication}. We assume the presence of two agents, Alice and Bob. Alice and Bob observe i.i.d. samples $\bmsfX_1$ and $\bmsfX_2$, respectively. For simplicity and consistency with prior work, we assume that the pair $\bZ=[\bX_1^\top,\bX_2^\top]^\top$ is jointly Gaussian with zero mean and covariance matrix:
\[
\bC = \begin{bmatrix}
\bI_{d_1} & \bC_{12} \\
\bC_{12}^{\top} & \bI_{d_2}
\end{bmatrix},
\]
where $\norm[1]{\bC_{12}}_\op\le 1$.
The objective for both Alice and Bob is to estimate the cross covariance matrix $\bC_{12}$ via rate--limited interactive communication. The correlation estimation problem in \cite{hadar2019communication} corresponds to the special case where $d_1 = d_2 = 1$. Additionally, the case $(d_1 = d, d_2 = 1)$ is examined in \citep{sahasranand2021communication}.

More formally, consider a shared board on which Alice and Bob post their messages during communication. The interactive protocol proceeds as follows:
Alice first writes message $M_1 = \cE_1(\bmsfX_1)$, based on her data. Bob then responds with $M_2 = \cE_2(\bmsfX_2, M_1)$. Next, Alice writes $M_3 = \cE_3(\bmsfX_1, M_1, M_2)$, and this exchange continues alternately. The sequence of messages written on the board is denoted by $\Pi = (M_1, M_2, M_3, \dots)$. Ultimately, Alice and Bob compute an estimate $\bCh_{12} = \mathcal{D}(\Pi)$ of the cross covariance matrix $\bC_{12}$. The communication protocol is subject to a rate constraint given by:
\begin{equation}
H(\Pi) = H(M_1, M_2, M_3, \dots) \leq B.
\end{equation}
We are interested in the expected distortion of the protocol, defined as:
\begin{equation}\label{eq:expected_distortion-int}
\E\left[\cL_{\dist}(\bCh_{12}, \bC_{12})\right] = \E_{\{\bZ^{(i)}\}_{i=1}^{m} \sim P^{\otimes m}_\bZ} \left[\norm[1]{\bCh_{12} - \bC_{12}}_{\dist}\right],
\end{equation}
where $\dist$ denotes either the operator norm or the Frobenius norm. Our goal is to characterize the minimax expected distortion, defined by:
\begin{equation}\label{eq:minimax-distortion-int}
\ME^{\mathsf{int}}(m, d_1, d_2, B) := \inf_{(\Pi, \cD): H(\Pi) \le B} ~~\sup_{\|\bC_{12}\|_\op \le 1} \E\left[\cL_{\dist}(\bCh_{12}, \bC_{12})\right].
\end{equation}
\subsection{Upper bound}
We'll now explain the upper bound for the distance in terms of the operator norm. The approach for the \Frob norm is similar, so we'll omit it for brevity.

\subsection*{Scenario 1: Alice has More Dimensions ($d_1 \ge d_2$)}

Imagine Alice has access to more data dimensions than Bob. In this case, the process unfolds as follows:

\begin{enumerate}
    \item \textbf{Alice's First Round:} Alice doesn't write anything on the board.
    \item \textbf{Bob's Action:} Bob writes the same message he would in a distributed setting.
    \item \textbf{Alice as Server:} Alice effectively acts as a central server. She uses her data and Bob's message to estimate the cross--covariance matrix.
\end{enumerate}

This situation is like a distributed cross--covariance estimation where Alice can provide all her samples without any restrictions (like having an infinite communication budget, $B_1 = \infty$). Based on Theorem \ref{thm:achievable_scheme}, Alice can estimate the cross--covariance within an error of $\frac{\varepsilon}{2}$ if:

\begin{itemize}
    \item The number of samples, $m$, is sufficient: $m = \cO\left(\frac{d}{\varepsilon^2}\right)$
    \item Bob's communication budget, $B_2$, is sufficient: $B_2=\widetilde{\cO}\left(\frac{d_1 d_2}{\varepsilon^2}\right)$
\end{itemize}

To get an estimate $\widehat{\bC}_{12}$ of the true cross--covariance $\bC_{12}$, Alice uses a covering argument (explained in Appendix \ref{app:pack_cover_matrix}) to get a quantized version, $\widetilde{\bC}_{12}$. This quantized version satisfies the condition $\norm[1]{ \widehat{\bC}_{12}-\widetilde{\bC}_{12}}_{\text{op}}\leq \frac{\varepsilon}{2}$.

This quantization requires $\mathcal{O}\left(d_1d_2\log \frac{8}{\varepsilon}\right)$, (which is less than $\mathcal{O}\left(\frac{d_1d_2}{\varepsilon^2}\right)$) bits as detailed in Equation \eqref{eq:pack_op_ball_op_norm} of Appendix  \ref{app:pack_cover_matrix}. Alice then writes these bits on the board. This ensures that both Alice and Bob can compute $\widetilde{\bC}_{12}$, which is within an $\varepsilon$ error of the actual cross--covariance matrix $\bC_{12}$.

The total number of bits consumed in this entire process is:
\begin{equation}\label{eqn:interactive-Up}
B=\mathcal{\widetilde{O}}\left(\frac{d_1 d_2}{\varepsilon^2}\right). 
\end{equation}
\subsection*{Scenario 2: Bob has More Dimensions ($d_2 > d_1$)}

If Bob has more dimensions than Alice, their roles are simply reversed, and we arrive at the same total number of bits consumed.

\subsection*{Frobenius Norm Considerations}
A similar line of reasoning applies when considering the Frobenius norm distance. For an interactive approximation of the cross--covariance matrix within an $\varepsilon$ error in the Frobenius norm, the following constraints on the number of samples ($m$) and the total communication budget ($B$) are sufficient:

\begin{equation}\label{eqn:interactive-up-fr}
\begin{aligned}
m=\cO\left(\frac{d^2}{\varepsilon^2}\right), \qquad
B= \mathcal{\widetilde{O}}\left(\frac{d_1d_2 d_{\min}}{\varepsilon^2}\right).
\end{aligned}
\end{equation}
\begin{remark}
In \citep[Theorem 6]{sahasranand2021communication}, the scenario where $d_1=d-1$ and $d_2=1$ is analyzed, asserting a communication budget of $\Theta\left(\frac{d^2}{\varepsilon^2}\right)$ for approximating the cross--covariance matrix. Contrary to this, our preceding protocol demonstrates that this can be accomplished with a substantially reduced communication budget of $\widetilde{\cO}\left(\frac{d}{\varepsilon^2}\right)$. The root of this difference lies in a misapplication of a certain generalization of SDPI, a point we will elaborate on in the subsequent subsection.

\end{remark}
\subsection{Lower bound}
We now proceed to prove the tightness of the upper bound \eqref{eqn:interactive-Up}. The tightness of the upper bound \eqref{eqn:interactive-up-fr}  can be established using a similar argument and is therefore omitted for brevity. The proof relies on the concept of {\em ``symmetric--SDPI"} introduced in \citep{jingbo2017} and further investigated in \citep{hadar2019communication}. 
\subsubsection{Overview of symmetric--SDPI} For a pair of random variables $(X,Y)\sim P_{XY}$, the symmetric--SDPI coefficient is defined as the minimum number $s_\infty$ such that the following inequality holds for any integer number $T$:
\begin{equation}
    I(X;Y)-I(X;Y|U_1,\cdots,U_T)\leq s_\infty I(U_1,\cdots,U_T;X,Y) 
\end{equation}
where $U_1,\cdots,U_T$ satisfies
\begin{equation}\label{eqn:s-SDPI-Markov}
\begin{aligned}
    &U_i \mkv (X,U^{i-1}) \mkv Y, & i\in\N\setminus 2\N,\\
    &U_i \mkv (Y,U^{i-1}) \mkv X, & i\in\N\cap 2\N.
\end{aligned}
\end{equation}

\begin{lem}[Tensorization of symmetric--SDPI {\citep[Lemma 9.3]{hadar2019communication}}]\label{lem:sym-SDPI-ten}Let $(X^n,Y^n)\sim \otimes_{i=1}^n P_{X_iY_i}$. Then
$s_\infty(X^n,Y^n)=\max_{1\le i\le n}s_\infty(X_i,Y_i)
$.    
\end{lem}
\begin{lem}[symmetric--SDPI for Gaussian,{\citep[Lemma 9.4]{hadar2019communication}}]
Let $(X,Y)\sim \cN\left(\bzero,\begin{bmatrix}
1&\rho \\
\rho & 1
\end{bmatrix}\right)$. Then $s_\infty(X,Y)=\rho^2$.
    
\end{lem}

\begin{cor}[symmetric--SDPI for vector Gaussian]\label{cor:s-SDPI}
    Let $\begin{bmatrix}
    \bX_1\\
    \bX_2
    \end{bmatrix}$ be a zero mean Gaussian vector with covariance matrix $\bC=\begin{bmatrix}
    \bC_{11} & \bC_{12}\\
    \bC_{12}^{\top} & \bC_{22}
    \end{bmatrix}$, such that $\bC_{11}$ and $\bC_{22}$ are non--singular. 
    Then 
    \begin{equation*}
        s_\infty(\bX_1,\bX_2)= \norm{\bC_{11}^{-1/2}\bC_{12}\bC_{22}^{-1/2}}_{\op}^2.
    \end{equation*}
\end{cor}
\begin{remark}
In \citep[Lemma 16]{sahasranand2021communication}, the symmetric–SDPI for the pair $(\bX, Y)$ with covariance matrix  
\(
\bC = \begin{bmatrix} \bI_d & \rho \\ \rho^\top & 1 \end{bmatrix}, \quad \rho \in \mathbb{R}^d,  
\)  
is used as an intermediate step in deriving a lower bound for the $\DCME$. The authors claim (without rigorous proof) that the symmetric–SDPI is given by $\max_i \rho_i^2$, where $\rho_i$ denotes the $i$--th coordinate of $\rho$. However, the correct symmetric–SDPI for this pair is $\|\rho\|^2$, which can be significantly larger than the claimed bound.
\end{remark}
\begin{proof} Let $\bA=\bC_{22}^{-1/2}\bC_{12}^{\top}\bC_{11}^{-1/2}$ and consider its singular value decomposition as $\bA=\bU\bD\bV^\top$, where $\bU$ and $\bV$ are unitaries and $\bD$ is diagonal with diagonal entries $(D_{1},\cdots,D_r)$, where $r\le\min\{d_1,d_2\}$. Define $\bXb_1 = \bV^\top\bC_{11}^{-{1}/{2}} \bX_1$, $\bXb_2 = \bU^\top\bC_{11}^{-{1}/{2}} \bX_2$. Observe that $\bXb_1$ and $\bXb_2$  are in one--to--one correspondence with $\bX_1$ and $\bX_2$, respectively. Thus $s_\infty(\bXb_1,\bXb_2)=s_\infty(\bX_1,\bX_2)$. Also it can  be readily verified that 
    \begin{equation}
        \begin{bmatrix}
        \bXb_1\\
        \bXb_2
        \end{bmatrix} \sim \cN\left(\bzero, \begin{bmatrix}
        \bI_{d_1} & \bD\\
        \bD^\top & \bI_{d_2}
        \end{bmatrix}\right),
    \end{equation}
     Consider the components of $\bXb_1$ and  $\bXb_2$ as  $\bXb_1=(\bar{X}_{11},\cdots,\bar{X}_{1d_1})$ and $\bXb_2=(\bar{X}_{21},\cdots,\bar{X}_{2d_2})$. Then we observe that $(\bar{X}_{11},\bar{X}_{21}),\cdots,(\bar{X}_{1r},\bar{X}_{2r}),\{\bar{X}_{1i}\}_{i=r+1}^{d_1},\{\bar{X}_{2i}\}_{i=r+1}^{d_2}$ are mutually independent. Thus by Lemma \ref{lem:sym-SDPI-ten}, we obtain:
     \begin{equation}
         s_\infty(\bXb_1,\bXb_2)=\max_{1\le i\le r}s_{\infty}(\bar{X}_{1i},\bar{X}_{2i}).
     \end{equation}
    Further $(\bar{X}_{1i},\bar{X}_{2i})\sim \cN\left(\bzero,\begin{bmatrix}
1&D_i \\
D_i & 1
\end{bmatrix}\right)$. Therefore $s_{\infty}(\bar{X}_{1i},\bar{X}_{2i})=D_i^2$.  Subsequently, we have $s_\infty(\bX_1,\bX_2)=\max_{1\le i\le r}D_i^2=\|\bA\|_\op^2.$
\end{proof}

   \subsection{Minimax Lower bound} 
    The argument is again based on Fano's method. Similar to the proof of Theorem \ref{thm:lower_bound}, we consider the family of normal distributions $\cN\left(\bzero,\bC_v\right)$ with the covariance matrices $\bC_v$ structured as:
\begin{equation*}
    \bC_v = \begin{bmatrix}
    \bI_{d_1} & \delta\bD_v^\top\\
    \delta\bD_v & \bI_{d_2}
    \end{bmatrix},
\end{equation*}
where $\bD_v$ is a matrix in $\R^{d_2\times d_1}$ with $\norm[1]{\bD_{v}}_\op\leq 1$, and $\delta\leq 1$ is a parameter which will determined subsequently. Now, we have:
\begin{equation}
\begin{aligned}
    \rho_{\op} &= \inf_{(v,v')\in\cV^2,v\neq v'}\left\{\norm[1]{\bC_{v}-\bC_{v'}}_{\dist}\right\}\\
    &= \frac{\delta}{2}\inf_{(v,v')\in\cV^2,v\neq v'}\left\{\norm[1]{\bD_{v}-\bD_{v'}}_{\dist}\right\},
\end{aligned}
\end{equation}

We consider the family $\{\bD_v\}_{v\in\cV}$ as the $\frac{1}{4}$--packing of $\cB_{\normiii{.}_{\op}}^{(d_1d_2)}(1)$ (see Equation \eqref{eq:norm_op_ball}), under the $\normiii{.}_{\op}$ norm. Thus, $\inf\limits_{v,v':v\neq v'}\norm{\bD_v-\bD_{v'}}_{\dist} \geq \frac{1}{4}$, $\max\limits_{v\in\cV}\{\norm{\bD_{v}}_{\op}^2\} \leq 1$, and from  \eqref{eq:pack_op_ball_op_norm} and \eqref{eq:pack_op_ball_Fr_norm}, we have
$\log_2(|\cV|) \geq 2d_1d_2$. Invoking  Fano's method yields:
\begin{equation}\label{eqn:s-SDPI-Fano}
\begin{aligned}
    \inf\limits_{\Pi,\cD}\sup\limits_{P\in\cP}\E\left[\cL_{\dist}(\bCh,\bC)\right] &\geq \dfrac{\rho_{\dist}}{2}\left(1-\dfrac{I(V;\Pi)+\log 2}{2d_1d_2}\right)
\end{aligned}
\end{equation}
Thus we need to upper bound $I(V;\Pi)$. We do this using symmetric--SDPI. Defining $\widetilde{P}_{\Pi,\bmsfX_1,\bmsfX_1} = P_{\bmsfX_1}P_{\bmsfX_2}P_{\Pi|\bmsfX_1,\bmsfX_2}$, where $P_{\bmsfX_i}=\cN(\bzero,\bI_{d_i})^{\otimes m}$ for $i=1,2$, we have: 
\begin{equation}\label{eqn:s-SDPI-L-MI}
\begin{aligned}
    I(V;\Pi) &\stackrel{\text{(a)}}{\leq} \E_{V}\left[\Dkl\left(P_{\Pi}^{(V)}\|\widetilde{P}_{\Pi}\right)\right]\\
    &\stackrel{\text{(b)}}{\leq} I(\bmsfX_1;\bmsfX_2|V) - I(\bmsfX_1;\bmsfX_2\mid \Pi,V),
\end{aligned}
\end{equation}
where $P_{\Pi}^(v)$ is the marginal distribution of $\Pi$ when $(\bmsfX_1,\bmsfX_2)$ is resulted from the pair with covariance matrix $\bC_v$, (a) follows from \cite[Corollary 4.2]{polyWu2023}, and (b) follows from \cite[Theorem 7.1]{hadar2019communication}. Now, observe that $\Pi=(M_1,M_2,\cdots)$ satisfies the Markov chains in the definition of symmetric--SDPI \eqref{eqn:s-SDPI-Markov}(with $U_i$ is replaced with $M_i$), thus we can further upper bound the r.h.s. of \eqref{eqn:s-SDPI-L-MI} as,
\begin{equation}\label{eqn:s-SDPI-L-MI-2}
\begin{aligned}
    I(V;\bsM) &{\leq} \left(\max_v s_{\infty,v}(\bX_1,\bX_2)\right)I(\bX_1,\bX_2;\Pi|V)\le \delta^2 H(\Pi)\le \delta^2 B,
\end{aligned}
\end{equation}
where we have used Corollary \ref{cor:s-SDPI} to get $s_{\infty,v}(\bX_1,\bX_2)=
    \norm{\bC_{11}^{-1/2}\bC_{12}\bC_{22}^{-1/2}}_{\op}^2 = \delta^2\norm{\bD_v}_{\op}^2\le \delta^2
$.

Substituting \eqref{eqn:s-SDPI-L-MI-2} in \eqref{eqn:s-SDPI-Fano} with the choice $\delta=\sqrt{\frac{d_1d_2}{2B}}$, we conclude:
\begin{equation}
    \inf\limits_{\bsM,\cD}\sup\limits_{P\in\cP}\E\left[\cL_{\dist}(\bCh,\bC)\right] =\Omega\left( \sqrt{\frac{d_1d_2}{B}}\right).
\end{equation}

\subsection{Interaction Reduces the communication budget.} We now compare the total communication budgets required for cross--covariance estimation in the non--interactive and interactive settings. By Corollary \ref{cor:c-s-L} and Theorem \ref{thm:achievable_scheme}, the total communication budget needed to estimate the cross--covariance matrix $\bC_{12}$ up to an error of $\varepsilon$ in the non--interactive setting is $B_1+B_2 = \widetilde{\Theta}\bigl(\frac{d^2}{\varepsilon^2}\bigr)$. In contrast, allowing interaction reduces this budget to $\widetilde{\Theta}\bigl(\frac{d_1 d_2}{\varepsilon^2}\bigr)$, which can be significantly smaller than $\widetilde{\Theta}\bigl(\frac{d^2}{\varepsilon^2}\bigr)$ when the dimensions $d_1$ and $d_2$ are imbalanced—for example, when $d_1=1$ and $d_2=d-1$. Thus, interaction can significantly reduce the total communication budget for the $\DCCME$ task. This phenomenon has also been observed previously in the context of distributed nonparametric estimation \cite{liu2023few}.

\section{Conclusion}
\label{sec:conclusion}

This paper rigorously investigated the fundamental limits and achievable performance for distributed covariance matrix estimation $(\DCME)$ in a feature--split setting under communication constraints. Our core contribution is the development of the Conditional Strong Data Processing Inequality (C--SDPI), a novel theoretical framework that enabled us to derive near--optimal minimax lower bounds for the $\DCME$ problem. These bounds precisely quantify the trade--offs between sample complexity, communication budgets, data dimensionality, and estimation error, highlighting the inherent constraints on accuracy. We also designed and analyzed an explicit estimation scheme that achieves these theoretical limits up to logarithmic factors, confirming the near--optimality of our approach.

Furthermore, our analysis extended to interactive settings, revealing that interaction can significantly reduce the total communication budget for cross--covariance estimation, particularly in scenarios with imbalanced agent dimensions. This work provides a comprehensive information--theoretic foundation for distributed covariance matrix estimation, offering both theoretical insights into its limits and practical, near--optimal solutions for this challenging problem.

\subsubsection*{Acknowledgements}
We sincerely thank Mohammad Ali Maddah-Ali and Amin Gohari for their valuable comments and insightful discussions.



\newcommand{\etalchar}[1]{$^{#1}$}

\addcontentsline{toc}{section}{Appendices}

\appendix
\section{Some Preliminary Lemmas, Corollaries, and Propositions}
\label{app:preliminaries}

\subsection{Proof of Lemma \ref{lem:SPM}}
\label{app:proof_lem_SPM}
\begin{proof}
Let $\bD$ denote $\E\left[\bA^\top \bB\bA\right]$. Consider an arbitrary signed permutation matrix $\tilde{\bA} \in \cP_d$. Since $\cP_d$ forms a group under matrix multiplication, the random matrix $\bA\tilde{\bA}$ is also uniformly distributed over $\cP_d$. Consequently, for any $\tilde{\bA}\in\cP_d$, we have:
\begin{equation}
  \bD=\E\left[\bA^\top \bB\bA\right]=\E\left[\tilde{\bA}^\top\bA^\top \bB\bA\tilde{\bA}\right]= \tilde{\bA}^\top \bD\tilde{\bA} 
\end{equation}
We now demonstrate that $\bD$ must be a scalar multiple of the identity matrix. For any distinct pair of indices $(i,j)$, define the matrices $\bA^{(ij)}_{+}=\be_i\be_j^\top + \be_j\be_i^\top + \sum_{k\neq i,j}\be_k\be_k^\top$ and $\bA^{(ij)}_{-}=\be_i\be_j^\top - \be_j\be_i^\top + \sum_{k\neq i,j}\be_k\be_k^\top$. Both $\bA^{(ij)}_{+}$ and $\bA^{(ij)}_{-}$ are signed permutation matrices, and thus $(\bA^{(ij)}_\pm)^\top \bD \bA^{(ij)}_\pm=\bD$. This implies the following relationships:
\begin{equation}
\begin{aligned}
    D_{ii}=\be_i^\top \bD\be_i=\be_i^\top (\bA_+^{(ij)})^\top \bD \bA_+^{(ij)}\be_i=\be_j^\top \bD\be_j=D_{jj}\\
    D_{ij}=\be_i^\top \bD\be_j=\be_i^\top (\bA_+^{(ij)})^\top \bD \bA_+^{(ij)}\be_j=\be_j^\top \bD\be_i=D_{ji}\\
    D_{ij}=\be_i^\top \bD\be_j=\be_i^\top (\bA_-^{(ij)})^\top \bD \bA_-^{(ij)}\be_j=-\be_j^\top \bD\be_i=-D_{ji}
\end{aligned}
\end{equation}
From these relations, it follows that the off--diagonal entry $D_{ij}$ must be zero, and all diagonal entries are equal. Therefore, $\bD$ is a scalar multiple of the identity matrix; that is, $\bD=\alpha \bI_d$ for some scalar $\alpha$. Recalling that $\alpha \bI_d=\bD=\E[\bA^\top \bB\bA]$, taking the trace of both sides of this identity yields:
\begin{equation}
   \alpha d=\Tr{\E[\bA^\top \bB\bA]}=\E\left[\Tr{\bA\bA^\top \bB}\right]=\Tr{\bB}, 
\end{equation}
where we have utilized the orthogonality property of signed permutation matrices (i.e., $\bA\bA^\top = \bI$). Consequently, $\alpha=\frac{\Tr{\bB}}{d}$, completing the proof.
\end{proof}

\subsection{A Lemma from Linear Algebra}
\label{app:lem_eigenvals_and_eigenvects_of_matrix}

\begin{lem}\label{lem:eigenvals_and_eigenvects_of_matrix}
  Consider the matrix $\bA\in\R^{m\times n}$ and define the matrix $\bB\in\R^{(m+n)\times (m+n)}$ as follows:
  \[\bB = \left[
		\begin{matrix}
			\mathbf{0} & \bA\\
			\bA^\top & \mathbf{0}
		\end{matrix}\right].\]
	If we denote the singular value decomposition (SVD) of $\bA$ as $\bA = \sum_{i=1}^{r}\sigma_i\bu_i\bv_i^\top$, then the eigenvalues and eigenvectors of $\bB$ are:
	\[\left\{\pm\sigma_i\right\}_{i=1}^{r},\qquad
		\left\{\frac{1}{\sqrt{2}}\left[\begin{matrix}
			\pm \bu_i\\ \bv_i
		\end{matrix}\right]\right\}_{i=1}^{r}\]
\end{lem}

\begin{proof}
    From the singular value decomposition of $\bA$, we have:
    \begin{equation}
        \bA\bv_i = \sigma_i\bu_i,\qquad \bA^\top\bu_i = \sigma_i\bv_i.
    \end{equation}
    We write:
    \begin{equation}
    \begin{split}
        \frac{1}{\sqrt{2}}\bB \left[\begin{matrix}
		\pm\bu_i\\ \bv_i
	    \end{matrix}\right] &= \frac{1}{\sqrt{2}}\left[
	    \begin{matrix}
		\mathbf{0} & \bA\\
		\bA^\top & \mathbf{0}
	    \end{matrix}\right]\left[\begin{matrix}
		\pm\bu_i\\ \bv_i
	    \end{matrix}\right]\\
	    &= \frac{1}{\sqrt{2}} \left[\begin{matrix}
		\bA\bv_i\\ \pm\bA^\top\bu_i
	    \end{matrix}\right]\\
	    &= \frac{\sigma_i}{\sqrt{2}} \left[\begin{matrix}
		\bu_i\\ \pm\bv_i
	    \end{matrix}\right]\\
	    &= \frac{\pm\sigma_i}{\sqrt{2}} \left[\begin{matrix}
		\pm\bu_i\\ \bv_i
	    \end{matrix}\right].
    \end{split}
    \end{equation}
    This completes the proof.
\end{proof}

\subsection{Some Properties of Sub--Gaussian Random Variables}
\label{app:prop_sub_gaussian_gamma}
To study some properties of sub--Gaussian random variables, familiarity with another family of random variables is necessary. This family extends the class of sub--Gaussian random variables and is called sub--Gamma random variables.

\begin{dfn}[{\citep[~Chapter 2.4]{boucheron2013concentration}}]\label{def:Sub_Gamma_random_variable}
  A random variable $X$ is called $(\sigma,\alpha)$--sub--Gamma, if: 
  \[
      \E \Big[e^{\lambda(X - \E[X])}\Big] \leq \exp\Big(\frac{\lambda^2\sigma^2}{2(1-\alpha|\lambda|)}\Big),
  \]
  for all $\lambda$ such that: $|\lambda|<\frac{1}{\alpha}$.
\end{dfn}
We state and prove some properties of sub--Gaussian and sub--Gamma random variables.

\begin{lem}[{\citep{boucheron2013concentration}}]\label{lem:sum_independent_sub_Gaussian&Gamma}
    Consider an independent sequence $\{X_i\}_{i=1}^{m}$ of random variables,
    \begin{itemize}
        \item if $X_i$, $i\in[m]$ is a $\sigma_i$--sub--Gaussian random variable, then $\sum\limits_{i=1}^n X_i$ is $\sqrt{\sum\limits_{i=1}^n \sigma_i^2}$--sub--Gaussian.
        \item if $X_i$, $i\in[m]$ is a $(\sigma_i,\alpha_i)$--sub--Gamma random variable, then $\sum\limits_{i=1}^n X_i$ is $\Big(\sqrt{\sum\limits_{i=1}^n \sigma_i^2},\max\limits_i \{\alpha_i\}\Big)$--sub--Gamma.
    \end{itemize}
\end{lem}

\begin{lem}\label{lem:tail_gamma}
    Any $(\sigma,\alpha)$--sub--Gamma random variable $X$ satisfies the following inequality:
    \begin{equation*}
        \begin{split}
            \Pp\left[X\ge t\right]&\leq
            \exp\Big(\frac{-t^2}{2(\sigma^2+\alpha t)}\Big)\\
            &\leq \exp\Big(\frac{1}{2(\sigma^2+\alpha)}\min\{t,t^2\}\Big)
        \end{split}
    \end{equation*}
\end{lem}

\begin{proof}
    Some variations of this lemma are presented in different papers. For completeness, we provide a proof here. We write:
    \begin{equation}\notag
    \begin{split}
        \Pp[X\geq t] &\overset{\text{(a)}}{=} \Pp\left[e^{\lambda X} \geq e^{\lambda t}\right]\\
        &\overset{\text{(b)}}{\leq} e^{-\lambda t}\E\left[e^{\lambda X}\right]\\
        &\overset{\text{(c)}}{\leq} \exp\Big(-\lambda t + \frac{\lambda^2\sigma^2}{2(1-\alpha|\lambda|)}\Big).
    \end{split}
    \end{equation}
    Note that (a) holds when $\lambda>0$, (b) is derived from Markov's inequality, and (c) follows from Definition \ref{def:Sub_Gamma_random_variable}, assuming  $|\lambda|<\frac{1}{\alpha}$. Now we set $\lambda=\frac{t}{\sigma^2+t\alpha}$, which satisfies the condition $0<\lambda<\frac{1}{\alpha}$. Thus:
    \begin{equation}\notag
    \begin{split}
        \Pp[X\geq t] &\leq \left.\exp\Big(-\lambda t + \frac{\lambda^2\sigma^2}{2(1-\alpha|\lambda|)}\Big)\right|_{\lambda=\frac{t}{\sigma^2+t\alpha}}\\
        &= \exp\Big(\frac{-t^2}{2(\sigma^2+\alpha t)}\Big).
    \end{split}
    \end{equation}
    Note that if $t\leq 1$, we have: $\sigma^2+\alpha\geq\sigma^2+\alpha t$; therefore:
    \begin{equation}\notag
        \frac{t^2}{2(\sigma^2+\alpha t)}\geq \frac{t^2}{2(\sigma^2+\alpha)}\qquad (0<t\leq1).
    \end{equation}
    On the other hand, if $t\geq 1$, we have: $t(\sigma^2+\alpha)\geq \sigma^2+\alpha$; therefore:
    \begin{equation}\notag
        \frac{t^2}{2(\sigma^2+\alpha t)}\geq \frac{t}{2(\sigma^2+\alpha)}\qquad (t\geq 1).
    \end{equation}
    Thus:
    \begin{equation}\notag
        \frac{t^2}{2(\sigma^2+\alpha t)}\geq \frac{1}{2(\sigma^2+\alpha)}\min\{t,t^2\}.
    \end{equation}
    and the second inequality is proved.
\end{proof}

\begin{lem}[A maximal inequality for sub--Gamma Random Variables {\citep[~Corollary 2.6]{boucheron2013concentration}}]\label{le:max-Gamma}
   Let $\{X_i\}_{i=1}^{n}$ be  a sequence of centered sub--Gamma random variables with the same parameters $(\sigma,\alpha)$. Then:
   \begin{equation}
       \E\Big[\max\limits_{i\in[n]}X_i\Big]\leq \sigma\sqrt{2\ln(n)}+\alpha \ln(n).
   \end{equation}
\end{lem}

\begin{lem}\label{lem:multiple_of_sub_Gaussian_Gamma} 
  Assume that $X$ and  $Y$ are centered sub--Gaussian random variables with parameters $\sigma_1$  and  $\sigma_2$, respectively. Then $XY-\E[XY]$ is a sub--Gamma random variable with parameters $(5\sigma_1\sigma_2,2.5\sigma_1\sigma_2)$.
\end{lem}

\begin{proof}
    We write:
    \begin{align}
            \E\left[e^{\lambda (XY-\E[XY])}\right] &= 1 + \lambda\E\left[ (XY-\E[XY])\right] + \sum_{k=2}^{+\infty} \frac{\lambda^k\E\left[ (XY-\E[XY])^k\right]}{k!}\nonumber\\
            &= 1 + \sum_{k=2}^{+\infty} \frac{\lambda^k}{k!}\E\left[ (XY-\E[XY])^k\right]\nonumber\\
            &\le 1 + \sum_{k=2}^{+\infty} \frac{|\lambda|^k}{k!}\E\left[|XY-\E[XY]|^k\right]
            \nonumber\\
            &= 1 + \sum_{k=2}^{+\infty} \frac{|\lambda|^k}{k!}\|XY-\E[XY]\|_k^k
           \label{eqn:y-kth-norm} \\
            &\leq 1 + \sum_{k=2}^{+\infty} \frac{|\lambda|^k}{k!}\left(\|XY\|_k+\|\E[XY]\|_k\right)^k
            \label{eqn:y-minkovsky}\\
            &= 1 + \sum_{k=2}^{+\infty} \frac{|\lambda|^k}{k!}\left(\Big(\E\left[|XY|^k\right]\Big)^{\frac{1}{k}}+|\E[XY]|\right)^k
            \nonumber\\
            &\leq 1 + \sum_{k=2}^{+\infty} \frac{|\lambda|^k}{k!}\left(\Big(\E\left[X^{2k}\right]\E\left[Y^{2k}\right]\Big)^{\frac{1}{2k}}+\sqrt{\E[X^2]\E[Y^2]}\right)^k
            \label{eqn:y-CSineq}\\
            &\leq 1 + \sum_{k=2}^{+\infty} (|\lambda|\sigma_1\sigma_2)^k\frac{\left(\left(2^{k+1}k!\right)^{\frac{1}{k}}+1 \right)^k}{k!}\label{eqn:y-subg-moment}
            \\
            &\leq 1 + \sum_{k=2}^{+\infty} (|\lambda|\sigma_1\sigma_2)^k. 2. (2.5)^k\label{eqn:y-ineq-diagram}
            \\
            &= 1+\frac{25(\lambda\sigma_1\sigma_2)^2}{2(1-2.5|\lambda|\sigma_1\sigma_2)}\\ 
            &\leq \exp\left(\frac{25(\lambda\sigma_1\sigma_2)^2}{2(1-2.5|\lambda|\sigma_1\sigma_2)}\right),\label{eqn-y-f-SE}
    \end{align}
    where:
    \begin{itemize}
        \item in \eqref{eqn:y-kth-norm}, for a random variable $Z$, $\|Z\|_{k}:=\E^{1/k}[|Z|^k]$ is the $L_k$ norm of the random variable $Z$, 
        \item
        \eqref{eqn:y-minkovsky} follows directly from the application of Minkowski's inequality (also known as the triangle inequality) to the $L_k$ norm.
        \item
        \eqref{eqn:y-CSineq} follows from Cauchy--Schwarz inequality,
        \item
        in \eqref{eqn:y-subg-moment}, we use the following upper bound for the $2k$--th moment of a $\sigma$--sub--Gaussian random variable $Z$ (see \citep[Theorem 2.1]{boucheron2013concentration}):
        \[
        \E[Z^{2k}]\leq 2(2\sigma^2)^k k!,
        \]
        \item
        \eqref{eqn:y-ineq-diagram} follows from the fact that the function $h[k]:=\frac{\Big(\left(2^{k+1}k!\right)^{\frac{1}{k}}+1 \Big)^k}{(2.5)^kk!}$ is a decreasing function on $\{2,3,\cdots\}$ and takes its maximum at $k=2$, which is equal to 2 (see Figure \ref{fig-y-diagram}).
    \end{itemize}
    Finally, \eqref{eqn-y-f-SE} implies that $XY-\E[XY]$ is a $(5\sigma_1\sigma_2,2.5\sigma_1\sigma_2)$--sub--Gamma random variable.
\end{proof}
{\pgfplotsset{
    standard/.style={
        axis x line=middle,
        axis y line=middle,
        enlarge x limits=0.15,
        enlarge y limits=0.15,
        every axis x label/.style={at={(current axis.right of origin)},anchor=north west},
        every axis y label/.style={at={(current axis.above origin)},anchor=north east},
        every axis plot post/.style={mark options={fill=blue,draw=blue}}
        }
    }
    \begin{figure}[t]
    \begin{center}
        \begin{tikzpicture}
            \begin{axis}[%
                standard,
                domain = 2:20,
                samples = 19,
                xlabel={$k$},
                ylabel={$h[k]$},
                ymin=0,
                ymax=2]
                \addplot+[ycomb,blue,thick] {(((2^(x+1)*(x!))^(1/x)+1)^x)/(2.5^x*x!)};
            \end{axis}
        \end{tikzpicture}
        \caption{Diagram of the function $h[k]:=\frac{\Big(\left(2^{k+1}k!\right)^{\frac{1}{k}}+1 \Big)^k}{(2.5)^kk!}$.}\label{fig-y-diagram}
  \end{center}  \end{figure}
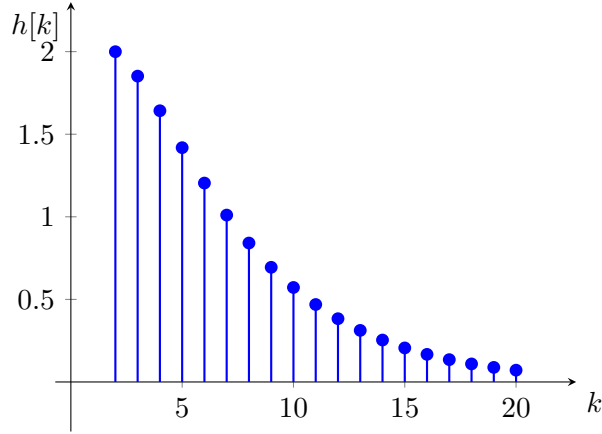
}

\begin{cor}\label{cor:Hoeffding_multiple_sub_Gaussian}
Let $\{(X_i,Y_i)\}_{i=1}^m$  be a sequence of i.i.d. pairs of random variables where $X_i$'s and $Y_i$'s are $\sigma_1$--sub--Gaussian and $\sigma_2$--sub--Gaussian, respectively. If we define $Z_i = X_iY_i-\E[X_iY_i]$, then we have:
\begin{equation*}
    \begin{split}
        \Pp\bigg[\frac{1}{m}\sum_{i=1}^m Z_i\geq 10\sigma_1\sigma_2t\bigg]\leq\exp\left({-m}.\min\{t,t^2\}\right).
    \end{split}
\end{equation*}
\end{cor}

\begin{proof}
    From Lemma \ref{lem:multiple_of_sub_Gaussian_Gamma}, we know that $Z_i = X_iY_i - \E[X_iY_i]$ is a $(5\sigma_1\sigma_2,2.5\sigma_1\sigma_2)$--sub--Gamma random variable. Therefore, using Lemma \ref{lem:sum_independent_sub_Gaussian&Gamma}, we  conclude that $\sum\limits_{i=1}^m Z_i$ is a $(5\sigma_1\sigma_2\sqrt{m}, 2.5\sigma_1\sigma_2)$--sub--Gamma random variable. Thus:
    \begin{equation}
    \begin{split}
        \Pp\bigg[\frac{1}{m}\sum_{i=1}^m (X_iY_i-\E[X_iY_i])\geq 10\sigma_1\sigma_2t\bigg] &= \Pp\bigg[\sum_{i=1}^m Z_i\geq 10m\sigma_1\sigma_2t\bigg]\\
        &\leq \exp\Big(\frac{-100m^2\sigma_1^2\sigma_2^2t^2}{2(25\sigma_1^2\sigma_2^2m+25 \sigma_1^2\sigma_2^2 mt)}\Big)\\
        &= \exp\Big(\frac{-2mt^2}{1+t}\Big)\\
        &\leq \exp\Big({-m}.\min\{t,t^2\}\Big).
    \end{split}
    \end{equation}
\end{proof}

\subsection{An Important Relation Between the Packing and the Covering Numbers of a Set}

The packing and covering numbers are defined in Section \ref{sec:pack_cover_nums}. There is an important relationship between the packing and covering numbers of a set, stated in the following lemma:

\begin{lem}[{\citep[~Lemma 5.5]{wainwright2019high}}]\label{lem:packing_covering_numbers}
  For all $\epsilon>0$, the packing and covering numbers are related as follows:
  \[\cM(\cK,\msfd,2\epsilon)\leq \cN(\cK,\msfd,\epsilon)\leq \cM(\cK,\msfd,\epsilon).\]
\end{lem}

\subsection{Finding Upper Bound on Operator Norm of Matrices, Using Covering Nets}

The following lemma is useful in finding an upper bound for the operator norm of a random matrix.

\begin{lem}[{\citep[~Exercise 4.4.3]{vershynin2018high}}]\label{lem:norm_upperbound_covering}
    Let $\bA$ be a $m\times n$ matrix. We define the sets $\cS^{m-1}=\{\bu\in\R^{m-1}:\norm{\bu}=1\}$ and $\cS^{n-1}=\{\bv\in\R^{n-1}:\norm{\bv}=1\}$. We fix an arbitrary $\epsilon>0$ and denote an $\epsilon$--covering set of $\cS^{m-1}$ by $\cN_{\epsilon}^{(m)}$ and an $\epsilon$--covering set of $\cS^{n-1}$ by $\cN_{\epsilon}^{(n)}$. We have:
    \[\norm{\bA}_{\op}  \leq \frac{1}{1-2\epsilon} \max\limits_{\bu\in\cN_{\epsilon}^{(m)},\bv\in\cN_{\epsilon}^{(n)}}\left\{\bu^{\top} \bA \bv\right\}.\]
\end{lem}

\subsection{Packing and Covering in Matrix Spaces}
\label{app:pack_cover_matrix}
Consider the family of matrices defined as follows:
\begin{equation}\notag
    \cA = \{\bA\in\R^{m\times n}:\norm{\bA}_{\op}\leq r\}.
\end{equation}
We  vectorize each member of this family as:
\begin{equation}\notag
    \ba = \vect(\bA) = [A_{11},A_{12},\dots,A_{1n},A_{21},\dots,A_{mn}]^\top.
\end{equation}
We  convert the $\dist$ norm on the matrix space $\R^{m\times n}$ to a norm on $\R^{m n}$  
via $\normiii{\ba}_{\dist}=\norm{\bA}_{\dist}$. 
Now we define the ball of radius $r$ under norm $\normiii{.}_{\op}$ as:
\begin{equation}\label{eq:norm_op_ball}
    \cB_{\normiii{.}_{\op}}^{(mn)}(r)=\left\{\bx\in\R^{mn}:\normiii{\bx}_{\op}\leq r\right\} = \left\{\vect(\bA):\bA\in\R^{m\times n}, \norm{\bA}_{\op}\leq r\right\}.
\end{equation}
We consider an $\epsilon$--covering net for $\cB_{\normiii{.}_{\op}}^{(mn)}(r)$ under the norm $\normiii{.}_{\dist}$, where $\dist$ can denote \Frob or operator norm. 

\begin{itemize}
    \item Consider the case $\dist=\op$, in this case, from \citep[~Lemma 5.7]{wainwright2019high}, we have:
  \begin{equation}\notag
      \left(\frac{r}{\epsilon}\right)^{mn} \leq \cN(\cB_{\normiii{.}_{\op}}^{(mn)}(r),\normiii{.}_{\op},\epsilon) \leq
      \left(1+\frac{2r}{\epsilon}\right)^{mn}
      \leq\left(\frac{3r}{\epsilon}\right)^{mn}.
  \end{equation}
  From Lemma \ref{lem:packing_covering_numbers} we conclude:
  \begin{equation}\label{eq:pack_op_ball_op_norm}
    \left(\frac{r}{\epsilon}\right)^{mn}\leq
    \cN(\cB_{\normiii{.}_{\op}}^{(mn)}(r),\normiii{.}_{\op},\epsilon)\leq
    \cM(\cB_{\normiii{.}_{\op}}^{(mn)}(r),\normiii{.}_{\op},\epsilon)\leq \cN(\cB_{\normiii{.}_{\op}}^{(mn)}(r),\normiii{.}_{\op},\frac{\epsilon}{2})\leq \left(1+\frac{4r}{\epsilon}\right)^{mn}
  \end{equation}
  \subsubsection{Matrix quantization scheme} We quantize matrix $\bA$, whose operator norm is at most $r$, under the norm $\normiii{.}_{\op}$, with the matrices corresponding to the covering points of $\cB_{\normiii{.}_{\op}}^{(mn)}(r)$ . 
  Note that the number of these points is less than $\left(\frac{3r}{\epsilon}\right)^{mn}$, so we can send the index of the quantized matrix using at most $mn\log_2(\frac{3r}{\epsilon})$ bits. Furthermore, if we denote the output of the quantization with $Q_{\op}(\bA)$, we have:
  \begin{equation}\notag
      \norm{\bA-Q_{\op}(\bA)}_{\op} \leq \epsilon.
  \end{equation}
  \item In the case $\dist=\Fr$, we only find a lower bound on the packing number  $\cM(\cB_{\normiii{.}_{\op}}^{(mn)}(r),\normiii{.}_{\Fr},\epsilon)$.

  \begin{lem}\label{lem:packing_op_ball_Fr_dist}
For $\cM(\cB_{\normiii{.}_{\op}}^{(mn)}(1),\normiii{.}_{\Fr},\epsilon)$, we have:
\[
\cM(\cB_{\normiii{.}_{\op}}^{(mn)}(1),\normiii{.}_{\Fr},\epsilon)\geq \left(\frac{ \sqrt{\min\{m,n\}}}{14\epsilon}
\right)^{mn}
\]
\end{lem}
\begin{proof}
    Let $\cA=\{\bA_1,\cdots,\bA_M\}$ be a \emph{maximal $\epsilon$--packing} of the ball $\cB_{\normiii{.}_{\op}}^{(mn)}(1)$. Then $\cA$ is also an $\epsilon$--covering. In particular:
    \[
    \cB_{\normiii{.}_{\op}}^{(mn)}(1)\subseteq\bigcup_{i=1}^M\cB_{\normiii{.}_{\Fr}}^{(mn)}(\bA_i;\epsilon),
    \]
    where $\cB_{\normiii{.}_{\Fr}}^{(mn)}(\bA;\epsilon)=\{\vect(\bB):\norm[1]{\bB-\bA}_\Fr\leq \epsilon\}$ is the \Frob ball with center $\bA$ and radius $\epsilon$. The proof follows a probabilistic argument which is similar to the volume argument usually used to prove packing numbers. 

    Let $\bsG=[g_{ij}]_{m\times n}$ be a random matrix with independent $g_{ij}\sim\cN\left(0,\frac{1}{4(\sqrt{m}+\sqrt{n})^2}\right)$ elements.
    It follows from \citep[Theorem 7.3.1]{vershynin2018high} that $\E\left[\norm[0]{\bsG}_\op\right]\le \frac{1}{2}$. Thus  Markov inequality yields:
    \begin{equation}\label{eq:y-half} 
     \Pp\left[\vect(\bsG)\in\cB_{\normiii{.}_{\op}}^{(mn)}(1)\right] = \Pp\left[\norm{\bsG}_{\op}\leq 1\right] = 1-\Pp\left[\norm{\bsG}_{\op}> 1\right] \geq 1-\frac{1}{2} = \frac{1}{2}.  
    \end{equation}

    On the other side, union bound gives:
    \begin{equation}\label{eq:y-union-fr}
        \Pp\left[\vect(\bsG) \in\cB_{\normiii{.}_{\op}}^{(mn)}(1) \right]\leq \sum_{i=1}^M \Pp\left[\vect(\bsG)\in\cB_{\normiii{.}_{\Fr}}^{(mn)}(\bA_i;\epsilon)\right].
    \end{equation}
  We now proceed to find an upper bound on the inner term in the summation. Observe:

  \begin{equation}\label{eq:y-vol}
    \begin{split}
    \Pp\left[\vect(\bsG) \in\cB_{\normiii{.}_{\Fr}}^{(mn)}(\bA_i;\epsilon)\right]&=\int_{\cB_{\normiii{.}_{\Fr}}^{(mn)}(\bA_i;\epsilon)} \left(\frac{4(\sqrt{m}+\sqrt{n})^2}{2\pi}\right)^{\frac{mn}{2}}\exp\left(-{2(\sqrt{m}+\sqrt{n})^2\norm[0]{\bG-\bA_i}_\Fr^2}\right)d\bG\\
     &\leq \int_{\cB_{\normiii{.}_{\Fr}}^{(mn)}} \left(\frac{4(\sqrt{m}+\sqrt{n})^2}{2\pi}\right)^{\frac{mn}{2}}d\bG\\
     &= \left(\frac{4(\sqrt{m}+\sqrt{n})^2}{2\pi}\right)^{\frac{mn}{2}}\Vol\left(\cB_{\normiii{.}_{\Fr}}^{(mn)}(\bA_i;\epsilon)\right)\\
     &= \left(\frac{2\epsilon^2(\sqrt{m}+\sqrt{n})^2}{\pi}\right)^{\frac{mn}{2}}\mathsf{Vol}\left(\cB_{\normiii{.}_{\Fr}}^{(mn)}(\bA_i;1)\right),
    \end{split}
  \end{equation}
  where we have used the density formula for normal distribution. Now we view $\cB_{\normiii{.}_{\Fr}}^{(mn)}(\bA_i;1)$ as a $mn$--dimensional euclidean ball. It is well known that the volume of this ball is given by
  \begin{equation*}
      \Vol\left(\cB_{\normiii{.}_{\Fr}}^{(mn)}(\bA_i;1)\right)=\frac{\pi^{\frac{mn}{2}}}{\Gamma(1+\frac{mn}{2})}.
  \end{equation*}
  Using the bound $\Gamma(1+x)>>(\frac{x}{e})^x$, in \eqref{eq:y-vol}, we obtain:
  \begin{equation}\label{eq:y-pp-upper}
    \begin{split}
        \Pp\left[\vect(\bsG) \in\cB_{\normiii{.}_{\Fr}}^{(mn)}(\bA_i;\epsilon)\right]
     &\leq \left(\frac{4e\epsilon^2(\sqrt{m}+\sqrt{n})^2}{mn}\right)^{\frac{mn}{2}}\\
     &\leq 
     \left(\frac{16e\epsilon^2\max\{m,n\}}{mn}\right)^{\frac{mn}{2}}\\
     &=\left(\frac{16e\epsilon^2}{\min\{m,n\}}\right)^{\frac{mn}{2}}.
    \end{split}
  \end{equation}
  Putting \eqref{eq:y-half}, \eqref{eq:y-union-fr} and \eqref{eq:y-pp-upper} together yields:
  \begin{equation}
      M\geq \frac{1}{2} \left(\frac{\min\{m,n\}}{16e\epsilon^2}\right)^{\frac{mn}{2}}\geq \left(\frac{\sqrt{\min\{m,n\}}}{8\sqrt{e}\epsilon}\right)^{mn}\geq \left(\frac{\sqrt{\min\{m,n\}}}{14\epsilon}\right)^{mn}.
  \end{equation}
  This concludes the proof.
\end{proof}
    From Lemma \ref{lem:packing_op_ball_Fr_dist}, we conclude that:
    \begin{equation}\label{eq:pack_op_ball_Fr_norm}
        \cM(\cB_{\normiii{.}_{\op}}^{(mn)}(r),\normiii{.}_{\Fr},\epsilon)\geq \left(\frac{r\cdot \sqrt{\min\{m,n\}}}{14\epsilon}
        \right)^{mn}.
    \end{equation}
\end{itemize}
\section{Proof of Lemma \ref{le:identity-from-CR}}
\label{app:proof_lemma_CSDPI}
Here, we state the proof of Lemma \ref{le:identity-from-CR}.

\begin{proof}
    Applying the chain rule for KL--divergence yields the following sequence of equalities:
    \begin{equation}
    \begin{aligned}
        \Dkl(Q_{Y_1,Y_2|V}\|P_{Y_1|V}P_{Y_2|V}|P_V)&\stackrel{\text{(a)}}{=}\Dkl(Q_{Y_1|V}\|P_{Y_1|V}|P_V)+\Dkl(Q_{Y_2|Y_1,V}\|P_{Y_2|V}|Q_{Y_1,V})\\
    	&\stackrel{\text{(b)}}{=}\Dkl(Q_{Y_1|V}\|P_{Y_1|V}|P_V)+\Dkl(Q_{Y_2,X_1|Y_1,V}\|Q_{X_1|Y_1,V}P_{Y_2|V}|Q_{Y_1,V})\\
    	&\quad-\Dkl(Q_{X_1|Y_1,Y_2,V}\|Q_{X_1|Y_1,V}|Q_{Y_1,Y_2,V})\\
    	&\stackrel{\text{(c)}}{=}\Dkl(Q_{Y_1|V}\|P_{Y_1|V}|P_V)+\Dkl(Q_{Y_2,X_1|Y_1,V}\|Q_{X_1|Y_1,V}P_{Y_2|V}|Q_{Y_1,V})\\
    	&\quad-I_Q(X_1;Y_2|Y_1,V)\\
    	&\stackrel{\text{(d)}}{=}\Dkl(Q_{Y_1|V}\|P_{Y_1|V}|P_V)+\Dkl(Q_{Y_2|X_1,Y_1,V}\|P_{Y_2|V}|Q_{X_1,Y_1,V})\\
    	&\quad-I_Q(X_1;Y_2|Y_1,V),
    \end{aligned}
    \end{equation}
    where (a) follows from the chain rule for KL--divergence applied to $\Dkl(Q_{Y_1,Y_2|V}\|P_{Y_1|V}P_{Y_2|V}|P_V)$, (b) follows from the chain rule for KL--divergence applied to $\Dkl(Q_{Y_2,X_1|Y_1,V}\|Q_{X_1|Y_1,V}P_{Y_2|V}|Q_{Y_1,V})$, (c) follows from the definition of conditional mutual information, and (d) follows from the chain rule for KL--divergence applied to $\Dkl(Q_{Y_2,X_1|Y_1,V}\|Q_{X_1|Y_1,V}P_{Y_2|V}|Q_{Y_1,V})$ and the fact that \\$\Dkl(Q_{X_1|Y_1,V}\|Q_{X_1|Y_1,V}|Q_{Y_1,V})=0$.
    
    The final identity follows from the fact that the product channel $T_{Y_1|X_1,V}T_{Y_2|X_2,V}$ induces the Markov chain $Y_1\mkv(X_1,V)\mkv Y_2$. This implies:
     \[
     \Dkl(Q_{Y_2|X_1,Y_1,V}\|P_{Y_2|V}|Q_{X_1,Y_1,V})=\Dkl(Q_{Y_2|X_1,V}\|P_{Y_2|V}|Q_{X_1,V}).
     \]
\end{proof}
     
\section{Gaussian Optimality}
\label{app:gaussian_optimality}
The goal of this section is to prove Lemma \ref{le:normal-opt}, which asserts the Gaussian optimality for the optimization problem \eqref{eqn:inf-normal}. For technical reason, we first consider a smoothed version of the optimization problem \eqref{eqn:inf-normal} and show the Gaussian optimality for the smoothed version. Then we argue that the desired result can be deduced from the result for the smoothed version.

For a positive value $\delta$, define
 \begin{align}
 	\msfg_\delta(Q_{\bX})&:=\msfss \Dkl(Q_{\bX}\|P_{\bX})-\Dkl(Q_{\bY|V}\ast\cN(0,{\delta}\bI)\|P_{\bY|V}|P_V).
 \end{align}
where $P_{\bX}=\cN(0,\bI_{d_{\bX}}),P_{\bY}=P_{\bY|V}=\cN(0,\bI_{d_{\bY}})$ and $\ast$ represent the {\em convolution} operator. That is the output $\bY$ is smoothed by adding a noise $\sqrt{\delta}\bW$ (where $\bW\sim\cN(0, \bI_{d_{\bY}})$) which is independent from all other random variables.

\begin{lemma}\label{le:normal-opt-s}
Let $\FG$ be the set of centered normal distribution $Q_{\bX}$ with the property $\Dkl(Q_{\bX}\|P_{\bX})$ is finite. Then:
\begin{equation}\label{eqn:inf-normal-s}
    \inf_{\substack{Q_{\bX}:\E_Q[\|\bX\|^2]\leq t\\ \Dkl(Q_{\bX}\|P_{\bX})<\infty}}\msfg_\delta(Q_{\bX})=\inf_{\substack{Q_{\bX}:\E_Q[\|\bX\|^2]\leq t\\Q_{\bX}\in\FG}}\msfg_\delta(Q_{\bX}).
\end{equation}
\end{lemma}

The proof of Lemma \ref{le:normal-opt} is divided to three steps:
\begin{enumerate}
    \item{\bf Existence of an optimizer}: We first consider a constrained variation of the optimization \eqref{eqn:inf-normal-s} and show that there exists a minimizer for which infimum is attained.
    
    \item {\bf Gaussian optimality for the optimizer}: In the second step, we prove that any minimizer from the previous step should be a centered normal distribution. From here, we deduce Lemma \ref{le:normal-opt-s}. The argument closely follows the doubling--trick argument in \cite{GengNair2014}, in which we will show that the minimizer should be rotationally invariant, thus it should be Gaussian distribution.
    \item We argue that Lemma \ref{le:normal-opt-s} yields Lemma \ref{le:normal-opt}.
\end{enumerate}

\subsection{Existence of optimizer}

Consider the following constrained version of the optimization \eqref{eqn:inf-normal-s},
\begin{align}\label{eqn:delta-opt}
 	\msfv_\delta(t)&:=\inf_{\substack{Q_{\bX}:\E_Q[\|\bX\|^2]\leq t,\\ \Dkl(Q_{\bX}\|P_{\bX})<\infty}}\msfg_\delta(Q_{\bX}).
 \end{align}
Further, let $\msfG_\delta(Q_{\bX})$ be the convex envelop of the function $\msfg_\delta(Q_{\bX})$. Then it can be computed using the following expression 
\begin{equation}
\begin{split}
    \msfG_\delta(Q_{\bX})&:=\inf_{n\ge 1}\inf_{\substack{\{p_i\}_{i=1}^n,\{Q^{(i)}_{\bX}\}_{i=1}^{n}: p_i\geq0,\sum_i p_i=1\\ \sum_i p_iQ^{(i)}_{\bX}=Q_{\bX}}}\sum_{i=1}^n p_i\msfg_\delta(Q^{(i)}_{\bX})\\
    &=\inf_{\widetilde{Q}_{U{\bX}}:\widetilde{Q}_{\bX}=Q_{\bX}} \E_{U}\left[\msfg_\delta(\widetilde{Q}_{\bX|U}(.|U))\right],
\end{split}
\end{equation}
where $U$ is an auxiliary random variable.
 
The main statement of this section is the following:
\begin{proposition}\label{prop:VDelta}
Let:
\begin{equation}
\begin{split}
	\msfV_\delta(t)&:=\inf_{\substack{Q_{\bX}:\E_Q[\|\bX\|^2]\leq t,\\ \Dkl(Q_{\bX}\|P_{\bX})<\infty}}\msfG_\delta(Q_{\bX})\\
	&=\inf_{\substack{Q_{U{\bX}}:\E_Q[\|\bX\|^2]\leq t,\\ \Dkl(Q_{\bX}\|P_{\bX})<\infty}}\E_{U}\left[\msfg_\delta\big({Q}_{\bX|U}(.|U)\big)\right].	
\end{split}
\end{equation}
Then there exists a {\em binary} random variable $U\in\{1,2\}$ and a probability measure $Q^{\ast}_{U\bX}$ such that:
\begin{equation}
\begin{split}
  \E_{U\sim Q^{\ast}_U}\left[\msfg_\delta({Q}^{\ast}_{\bX|U}(.|U))\right]&=\msfV_\delta(t),\\
  \E_{Q^{\ast}}[\bX|U]&=0,\\
 \E_{Q^{\ast}}[\|\bX\|^2]&\leq t,\\
 \Dkl(Q^{\ast}_{\bX}\|P_{\bX})&<\infty.
\end{split}
\end{equation} 
In other words, $Q^{\ast}_{U\bX}$ is a minimizer. Further, $\bX$ is a centered random variable under $Q^{\ast}_{\bX}$.  
\end{proposition}
We defer the proof of this proposition to the end of this section. We now investigate some properties of the functions $\msfg_\delta$ and $\msfv_\delta$ that will be used in the proof of Proposition \ref{prop:VDelta}.

The function $\msfg_\delta(Q_{\bX})$ can be represented using differential entropy as follows:
\begin{equation}\label{eqn:sdelta-DE}
\begin{split}
    \msfg_\delta(Q_{\bX})&=h_Q(\bY+\sqrt{\delta} \bW|V)-\msfss h_Q(\bX)+\frac{1}{2}\left(\E\left[\msfss\|\bX\|^2-\|\bY+\sqrt{\delta}\bW\|^2\right]\right)+C_1\\
    &=h_Q(\bY+\sqrt{\delta} \bW|V)-\msfss h_Q(\bX)+\frac{1}{2}\left(\E\left[\bX^\top(\msfss \bI-\bA_V\bA_V^\top)\bX\right]-\delta d_{\bY}\right)+C_1\\
    &=h_Q(\bY+\sqrt{\delta} \bW|V)-\msfss h_Q(\bX)+\frac{1}{2}\left(\E_{\bX}\left[\bX^\top(\msfss \bI-\E_V[\bA_V\bA_V^\top])\bX)\right]\right)+C\\
    &=h_Q(\bY+\sqrt{\delta} \bW|V)-\msfss h_Q(\bX)+\frac{1}{2}\E_{\bX}\left[\bX^\top\bB\bX\right]+C,
\end{split}
\end{equation}
where $\bX\sim Q_{\bX}$, $\bY=\bA_V \bX+\bZ_V$,  $C_1=\frac{1}{2}\left(\msfss d_{\bX}-d_{\bY}\right)\log(2\pi)$, $C=C_1-\frac{\delta d_{\bY}}{2}$ and $\bB:= \msfss \bI-\E_V[\bA_V\bA_V^\top]$. Observe that $\bB\succeq 0$ by the definition of $\msfss$.
\begin{lemma}\label{le:liminf-s}
	Suppose $Q_X^{(i)}\Rightarrow Q^{\ast}_X$ be a weakly convergent sequence of probability measure satisfying the moment constraint $\E_{Q^{(i)}}[\|X\|^2]<b$ for all $i$ and some positive value $b$, then:
	\begin{equation*}
		\liminf_{n\to\infty}\msfg_\delta(Q^{(i)}_X)\geq \msfg_\delta(Q_X^{\ast}).
	\end{equation*}
\end{lemma}
The proof is  based on the following fact about weaky converegen sequence of probability measures \cite{Billingsley1999},
\begin{fact}\label{fact:1}
	For any weakly convergent sequence $Q^{(i)}_{\bX}$ and any lower semi--continuous and bounded from below function $\varphi(\bX)$, we have:
	$$\liminf\limits_{n\to\infty} \E_{Q^{(i)}_{\bX}}\left[\varphi(\bX)\right]\geq \E_{Q^{\ast}_{\bX}}\left[\varphi(\bX)\right].$$
\end{fact}
\begin{proof}[Proof of Lemma \ref{le:liminf-s}]
 We show that each term in the expression \eqref{eqn:sdelta-DE} satisfies the desired limit behavior, separately. 
\begin{enumerate}
	\item Using \cite[Proposition 18]{GengNair2014} (see also \cite[Corollary 4]{polyanskiy2016wasserstein}), we have $\lim\limits_{n\rightarrow\infty}h_{Q^{(i)}}(\bY+\sqrt{\delta} \bW|V=v)=h_{Q^{\ast}}(\bY+\sqrt{\delta} \bW|V=v)$. Further, $h_Q(\bY+\sqrt{\delta} \bW|V=v)\geq h(\sqrt{\delta} \bW)$. Thus by Fatou's lemma, we have:
	\begin{equation}
	\begin{split}
	    \liminf_{n\to\infty}h_{Q^{(i)}}(\bY+\sqrt{\delta} \bW|V)&=\liminf_{n\to\infty}\E_{P_V}[h_{Q^{(i)}}(\bY+\sqrt{\delta} \bW|V=v)]\\
		&\geq \E_{P_V}[\lim_{n\to\infty}h_{Q^{(i)}}(\bY+\sqrt{\delta} \bW|V=v)]\\
		&= h_{Q^{\ast}}(\bY+\sqrt{\delta} \bW|V).
	\end{split}
	\end{equation}
	\item  It is shown in \cite[Lemma A2]{CourtadeLiu}, that any weakly convergent sequence $Q_{\bX}^{(i)}\Rightarrow Q^{\ast}_{\bX}$ with moment constraint $\E_Q^{(i)}\big[\bX\bX^\top\big]\preceq \bC$ satisfies the following inequality:
\begin{equation}\label{eqn:limsup}
	\limsup_{n\to\infty} h_{Q^{(i)}}(\bX)\leq h_{Q^{\ast}}(\bX).
\end{equation}
However, inspecting the proof in \cite{CourtadeLiu} shows that \eqref{eqn:limsup} also holds under weaker constraint $\E_Q^{(i)}\big[\|\bX\|^2\big]\leq t$.
   	
\item Since $\bB\succeq 0$, the function $f(\bx)=\bx^\top\bB\bx$ is bounded from below. Thus utilizing Fact \ref{fact:1} yields:
\begin{equation*}
	\liminf_{n\to\infty}\E_{Q^{(i)}}\big[\bX^\top\bB\bX\big]\geq \E_{Q^{\ast}}\big[\bX^\top\bB\bX\big].
\end{equation*}
\end{enumerate}

\end{proof}

\begin{proposition}\label{prop:v-delta}
The  optimization problem \eqref{eqn:delta-opt} has a minimizer. More precisely, there exists $Q^{\ast}_{\bX}$ satisfying the constraints  $\E_{Q^{\ast}}\big[\|\bX\|^2\big]\leq t,
$ and $\Dkl(Q_{\bX}^{\ast}\|P_{\bX})<\infty$ such that  $\msfv_\delta(t)=\msfg_\delta(Q_{\bX}^{\ast}).$
\end{proposition}

\begin{proof}
Take a sequence of probability measure $\{Q^{(i)}_{\bX}\}_{i=1}^n$ along which $\msfg_\delta(Q^{(i)}_{\bX})$ approaches $\msfv_\delta(t)$. The moment constraint $\E_{Q^{(i)}}[\|\bX\|^2]\leq t$ ensures that  $\{Q^{(i)}_{\bX}\}_{i=1}^n$ is a {\em tight} sequence \cite{Billingsley1999}. Thus  we can extract a subsequence of it that weakly converges to a probability measure $Q^{\ast}_{\bX}$. Without loss of generality, assume that the sequence $\{Q^{(i)}_{\bX}\}_{i=1}^n$ is itself convergent to $Q^{\ast}_{\bX}$. We show that $Q^{\ast}_{\bX}$ is a minimizer.
\begin{itemize}
\item Utilizing Fact \ref{fact:1} with $\varphi(\bx)=\|\bx\|^2$ implies that   $Q^{\ast}_{\bX}$ satisfies the constraint $\E_{Q^{\ast}}[\|\bX\|^2]\leq t$.  
Also from the semi--lower continuity of KL--divergence \cite{polyWu2023}, we have $\Dkl(Q^{\ast}_{\bX}\|P_{\bX})\leq \liminf\limits_{n\to\infty}$ $\Dkl(Q^{(i)}_{\bX}\|P_{\bX})<\infty.$ Thus $Q^{\ast}_{\bX}$ satisfies the constraint of \eqref{eqn:delta-opt}, so $\msfg_\delta(Q^{\ast}_{\bX})\geq \msfv_\delta(t)$. 
\item Conversely, Lemma \ref{le:liminf-s} says $\msfv_\delta(t)=\liminf\limits_{n\to\infty}\msfg_\delta(Q^{(i)}_{\bX})\geq \msfg_\delta(Q^{\ast}_{\bX})$. Therefore $Q^{\ast}_{\bX}$ must be the minimizer.
\end{itemize}
\end{proof}

\begin{proof}[Proof of Proposition \ref{prop:VDelta}]
Utilizing Proposition \ref{prop:v-delta}, we can express $\msfV_\delta(t)$ in the following way:
\begin{align}\label{eqn:cara}
	\msfV_\delta(t)=\inf_{n\geq 1}\inf_{\substack{ \{p_i\}_{i=1}^n,t_i: p_i\geq0,\sum_i p_i=1,\\ \sum_i p_it_i\leq t}}\sum_{i=1}^n p_i\msfv_\delta(t).
\end{align}
By Bunt--Carath\'eodor\'y theorem \cite{bunt1934bijdrage,caratheodory1911variabilitatsbereich}, for any point on the boundary of convex envelop of the set $\cM:=\left\{\big(t_i,\msfv_\delta(t_i)\big):1\leq i\leq n\right\}$ can be expressed as the convex combination of at most two points of $\cM$. Thus \eqref{eqn:cara} can be rewritten as:
\begin{align}\label{eqn:cara2}
	\msfV_\delta(t)=\inf_{\substack{ \alpha,t_1,t_2: 0\leq\alpha\leq \frac{1}{2},\\ \alpha t_1+\bar{\alpha}t_2\leq t}}\alpha \msfv_\delta(t_1)+\bar{\alpha}\msfv_\delta(t_2),
\end{align} 
where $\bar{\alpha}=1-\alpha$.
Let $\{(\alpha_i,t_{1i},t_{2i} )\}_{i\geq 1}$ be a sequence satisfying $\alpha_i t_{1i}+\bar{\alpha_i}t_{2i}\leq t$ such that $\alpha_i \msfv_\delta(t_{1i})+\bar{\alpha_i}\msfv_\delta(t_{2i})$ converges to $\msfV_\delta(t)$. Also let $Q^{(i)}_1$ and $Q_2^{(i)}$ be probability measures on $\cX$ such that $\msfv_\delta(t_{ki})=\msfg_\delta(Q^{(i)}_k)$ for $k=1,2$. Existence of such $Q^{(i)}_k$ with the property $\E_{Q^{(i)}_k}[\|\bX\|^2]\leq t_{ki}$ is guaranteed by Proposition \ref{prop:v-delta}. Since there exists a convergent subsequence of $\{\alpha_i\}_{i\geq1}$, we can assume that the sequence $\{\alpha_i\}_{i\geq1}$ is itself convergent to some $\alpha^{\ast}\in[0,\frac{1}{2}]$. We investigate the cases $\alpha^{\ast}>0$ and $\alpha^{\ast}=0$, separately.
\begin{itemize}
	\item {\bf Case I: $\alpha^{\ast}>0$.}\\
	After discarding some initial terms in the sequence $\{\alpha_i\}_{i\geq1}$, we can assume that  for all $i>0$, we have: $\alpha_i>\frac{\alpha^{\ast}}{2}$. This implies that for all $i>0$, $t_{1i}\leq \frac{2t}{\alpha^{\ast}}$ and $t_{2i}\leq 2t$. Thus $Q_1^{(i)}$ and $Q_2^{(i)}$ have common finite second moments, thus there exist a  subsequence $\{i_j\}_{j=1}^{\infty}$ for which $Q_1^{(i_j)}\Rightarrow Q^{\ast}_1$ and $Q_2^{(i_j)}\Rightarrow Q^{\ast}_2$. Rename the sequence $\{i_j:\}_{j\geq 1}$ to $\{1,2,\dots\}$.  Let $U\in\{1,2\}$ be a binary random variable and define $Q^{\ast}_U(U=1)=\alpha^{\ast}$ and $Q^{\ast}_{\bX|U=k}=Q^{\ast}_k$ for $k=1,2$. Also let $Q^{(i)}_U(U=1)=\alpha_i$ and $Q^{(i)}_{\bX|U=k}=Q^{(i)}_k$ for $k=1,2$. We have $Q^{(i)}_{\bX}\Rightarrow Q^{\ast}_{\bX}$. 
	Now consider:
	\begin{align}\label{eqn:second-moment-now}
		\E_{Q_{\bX}^{\ast}}\left[\|\bX\|^2\right]&\leq\liminf_{i\to\infty}\E_{Q_{\bX}^{(i)}}\left[\|\bX\|^2\right]=\liminf_{i\to\infty}\alpha_i\E_{Q_1^{(i)}}\left[\|\bX\|^2\right]+\bar{\alpha_i}\E_{Q_2^{(i)}}\left[\|\bX\|^2\right]\leq t. 
	\end{align}
	Also $Q^{\ast}_{\bX}$ satisfies the constraint $\Dkl(Q^{\ast}_{\bX}\|P_{\bX})<\infty$, by the semi--lower continuity of KL--divergence. Thus:
	\begin{align}\label{eqn:second-moment-now-1}
		\msfV_\delta(t)\leq \msfG_\delta(Q^{\ast}_{\bX})\leq \E\left[\msfg_\delta\big(Q^{\ast}_{\bX|U}(.|U)\big)\right]
	\end{align}
	However, Lemma \ref{le:liminf-s} asserts:
	\begin{equation}
	\begin{split}
	    \E\left[\msfg_\delta\big(Q^{\ast}_{\bX|U}(.|U)\big)\right]&=\alpha^*\msfg_\delta(Q^{\ast}_{1})+\bar{\alpha^{\ast}}\msfg_\delta(Q^{\ast}_{2})\\
	    &\leq \liminf_{i\to\infty}\alpha_i\msfg_\delta(Q^{(i)}_{1})+\liminf_{i\to\infty}\bar{\alpha_i}\msfg_\delta(Q^{(i)}_{2})
		\\
		&\leq \msfV_\delta(t).
	\end{split}
	\end{equation}
	Thus $Q^{\ast}_{U\bX}$ is the desired minimizer.
	\item{\bf Case II: $\alpha^{\ast}=0$.}\\
	We have again $t_{2i}\leq 2t$. Thus we can assume that $Q_2^{(i)}\Rightarrow Q_2^{\ast}$. The assumption $\alpha_i\to 0$ implies $\alpha_i Q^{(i)}_1+\bar{\alpha_i} Q_2^{(i)}\Rightarrow Q^{\ast}_2= Q^{\ast}_{\bX}$. In this case, the inequality \eqref{eqn:second-moment-now}  is still holds. Further we have:
	\begin{align}\label{eqn:second-moment-now-2}
		\msfV_\delta(t)\leq \msfG_\delta(Q^{\ast}_{\bX})\leq \msfg_\delta(Q^{\ast}_{\bX}).
	\end{align}
	In the other side, we show that $\msfV_\delta(t) \geq \msfg_\delta(Q^{\ast}_{\bX})$. 
	The function $\msfg_\delta(Q_1^{(i)})$ can be bounded from below as:
	\begin{equation}\label{eqn:boundedness}
	\begin{split}
		\msfg_\delta(Q_1^{(i)})&=h_{Q_1^{(i)}}(\bY+\sqrt{\delta} \bW|V)-\msfss h_{Q_1^{(i)}}(\bX)+\frac{1}{2}\E_{\bX\sim {Q_1^{(i)}}}\left[\bX^\top\bB\bX\right]+C\\
		&\geq h(\sqrt{\delta}\bW)-\msfss h_{\bX\sim\cN(0,\frac{t_{1i}}{d})}(\bX)+C\\
		&=D-\frac{\msfss}{2}\log(t_{1i}),
	\end{split}
	\end{equation}
    where we used the facts that $h(\bY+\bZ)\geq h(\bZ)$ for any independent random variable $\bY$ and $\bZ$, and the normal distribution maximizes the entropy under second moment constraint. We also used the positive--definiteness of $\bB$. In the last line, $D$ is a constant which is not depending on $t$.\\
    Since $\msfG_\delta$ is a convex envelop of $\msfg_\delta$ and the right hand side of \eqref{eqn:boundedness} is a convex function, the  inequality \eqref{eqn:boundedness} continues to hold for it. Now observe $t_{1i}\leq \frac{t}{\alpha_i}$. Inequality \eqref{eqn:boundedness} implies $\liminf\limits_{i\to\infty}\alpha_i\msfg_\delta(Q^{(i)}_{1})\geq 0$. Thus by Lemma \ref{le:liminf-s} we have:
	\begin{equation}
	\begin{split}
	    \msfg_\delta(Q^{\ast}_{\bX})&=\msfg_\delta(Q^{\ast}_{2})\\
	    &\leq \liminf_{i\to\infty}\bar{\alpha_i}\msfg_\delta(Q^{(i)}_{2})\\
		&\leq \liminf_{i\to\infty}\alpha_i\msfg_\delta(Q^{(i)}_{1})+\liminf_{i\to\infty}\bar{\alpha_i}\msfg_\delta(Q^{(i)}_{2})\\
		&\leq \msfV_\delta(t).
	\end{split}
	\end{equation}
    Therefore $\msfV_\delta(t)=\msfg_\delta(Q^{\ast}_{\bX})$. Hence in this case, we can assume that $U$ is a constant random variable and $Q^{\ast}_{\bX}$ is the desired minimizer. 
\end{itemize}
{\em Proof of $\E_{Q^{\ast}}[\bX|U]=0$.}

Let $\bmu_u=\E_{Q^{\ast}}[\bX|U=u]$ for $u\in\{1,2\}$. We show that for an optimal $Q^{\ast}_{U\bX}$, $\bX$ must be a conditionally centered random variable. Take $\tilde{\bX}:=\bX-\bmu_U$ and let $\tilde{\bY}$ be the output of the Markov kernel (channel) $T_{\bY|\bX,V}$ given the input $\tilde{\bX}$, that is, given $V=v$,  $\tilde{\bY}=\bA_v\tilde{\bX}+\bZ_v$. Suppose $(\bmu_1^\top,\bmu_2^\top)\neq (\bzero^\top,\bzero^\top)$. We show that $Q^{\ast}_{U\tilde{\bX}}$ achieves smaller $\E[\msfg(Q_{\bX|U})]$ than $Q^{\ast}_{U\bX}$, contradicting the assumption that $Q^{\ast}_{U\bX}$ is the minimizer. 
The definition of $\tilde{\bX}$ implies:
\begin{align*}
	\E\left[\bX^\top\bB\bX\right]=\E\left[\tilde{\bX}^\top\bB\tilde{\bX}\right]+\E_U\left[\bmu_U^\top\bB\bmu_U\right]\geq \E\left[\tilde{\bX}^\top\bB\tilde{\bX}\right],
\end{align*}
where we used the fact that $\bB$ is a positive semi--definite matrix.
Using this and \eqref{eqn:sdelta-DE}, we obtain:
\begin{align*}
    \E[\msfg_\delta(Q^{\ast}_{\tilde{\bX}|U})]&=\E_{Q^{\ast}_U}\left[h(\tilde{\bY}+\sqrt{\delta} \bW|V,U=u)-\msfss h(\tilde{\bX}|U=u)\right]+\frac{1}{2}\E\left[\tilde{\bX}^\top\bB\tilde{\bX}\right]+C\\
   &\stackrel{\text{(a)}}{=} \E_{Q^{\ast}_U}\left[h(\bY+\sqrt{\delta} \bW|V,U=u)-\msfss h(\bX|U=u)\right]+\frac{1}{2}\E\left[\tilde{\bX}^\top\bB\tilde{\bX}\right]+C\\
   &\leq  \E_{Q^{\ast}_U}\left[h(\bY+\sqrt{\delta} \bW|V,U=u)-\msfss h(\bX|U=u)\right]+\frac{1}{2}\E\left[\bX^\top\bB\bX\right]+C\\
   &=\E\left[\msfg_\delta(Q^{\ast}_{\bX|U})\right],
\end{align*}
where (a) follows from the fact that differential entropy is invariant under the shift.
\end{proof}
\subsection{Gaussian Optimality--Proof of Lemma \ref{le:normal-opt-s}}
 
For the sake of brevity, let $\msfg_\delta(Q_{U\bX}):=\E_U[\msfg_\delta(Q_{\bX|U})]$, $\olbY=\bY+\sqrt{\delta}\bW$ and $T_{\olbY|V,\bX}=T_{{\bY}|V,\bX}\ast \cN(0,\delta \bI_{d_{\bY}})$.

Let $Q^{\ast}_{U\bX}$ be the minimizer for which $U\in\{1,2\}$, 
$\msfg_\delta(Q_{U\bX}^{\ast})=\msfV_\delta(t)$ and $\E_{Q^{\ast}}[\bX|U]=0$.
Take $(U_1,\bX_1),(U_2,\bX_2)$ be two i.i.d. samples drawn from $Q^{\ast}_{U\bX}$. Let $Q_{U_1U_2\bX_1\bX_2}(u_1,u_2,\bx_1,\bx_2)=Q^{\ast}_{U\bX}(u_1,\bx_1)Q^{\ast}_{U\bX}(u_2,\bx_2)$ be the joint distribution of $(U_1,U_2,\bX_1,\bX_2)$. Set $\bU=[U_1,U_2]^\top$ and define:
\begin{equation}
\begin{aligned}
    \bX_+=\frac{\bX_1+\bX_2}{\sqrt{2}}, &&\bX_-=\frac{\bX_1-\bX_2}{\sqrt{2}},\\
    \olbY_+=\frac{\olbY_1+\olbY_2}{\sqrt{2}}, &&\olbY_-=\frac{\olbY_1-\olbY_2}{\sqrt{2}}.
\end{aligned}
\end{equation}
The random variables $(\bU,V,\bX_+,\bX_-,\bY_+,\bY_-)$ has the following properties:
\begin{enumerate}
    \item The joint distribution of $(\bU,V,\bX_+,\bX_-,\bY_+,\bY_-)$ factors as:  
\begin{equation}
 P_VQ_{\bU,\bX_+,\bX_-}T_{\bY_+|V,\bX_+}T_{\bY_-|V,\bX_-}.
\end{equation} 
\item The rotational invariance of the Gaussian noises $\bZ_v$ and $\bW$ implies $T_{\olbY_+|V,\bX_+}$ and $T_{\olbY_-|V,\bX_-}$ are two \emph{independent} copies of the channel $T_{\olbY|V,\bX}$, that is:
\begin{equation}
    T_{\olbY_+|V,\bX_+}(\by|v,\bx)=T_{\olbY_-|V,\bX_-}(\by|v,\bx)=T_{\olbY|V,\bX}(\by|v,\bx).
\end{equation}
\item The following moment constraint holds:
\begin{equation*}
    \E[\|\bX_+\|^2]=\E[\|\bX_-\|^2]=\E[\|\bX_1\|^2]\leq t.
\end{equation*}
\end{enumerate}

We now proceed to show $\bX_+\mkv (\olbY_+,V)\mkv \olbY_-$. To accomplish this, consider:
\begin{align}
    2\msfV_\delta(t)&=\msfg_\delta(Q_{U_1\bX_1}^{\ast})+\msfg_\delta(Q_{U_2\bX_2}^{\ast})\notag\\
    &=\msfss \Dkl(Q_{\bX_1\bX_2|U_1,U_2}\|P_{\bX}^{\otimes 2}|Q_{U_1U_2})-\Dkl(Q_{\olbY_1\olbY_2|U_1,U_2,V}\|P_{\bY}^{\otimes 2}|Q_{U_1U_2V})\label{eqn:plusexpansion1}\\
    &=\msfss \Dkl(Q_{\bX_+\bX_-|U_1,U_2}\|P_{\bX}^{\otimes 2}|Q_{U_1U_2})-\Dkl(Q_{\olbY_+\olbY_-|U_1,U_2,V}\|P_{\bY}^{\otimes 2}|Q_{U_1U_2V})\label{eqn:plusexpansion}\\
    &=\msfss \Dkl(Q_{\bX_+|\bU}\|P_{\bX}|Q_{\bU})-D(Q_{\olbY_+|\bU V}\|P_{\bY}|Q_{\bU V})\nonumber\\
        &\quad+\msfss \Dkl(Q_{\bX_-|\bU\bX_+}\|P_{\bX}|Q_{\bU\bX_+})-\Dkl(Q_{\olbY_-|\bU V\bX_+}\|P_{\bY}|Q_{\bU V\bX_+})\nonumber\\
        &\quad+{I_Q(\bX_+;\olbY_-|\bU V\olbY_+)}\label{eqn:plusexpansion2-IS}\\
    & \geq \E_{\bU}\left[\msfg_\delta(Q_{\bX_+|\bU})\right]+\E_{\bU,\bX_+}\left[\msfg_\delta(Q_{\bX_-|\bU\bX_+})\right] \label{eqn:plusexpansion2-PMI}\\
    & \geq 2\msfV_\delta(t), \label{eqn:plusexpansion2}
\end{align} 

where \eqref{eqn:plusexpansion2-IS} follows from the identity in Lemma \ref{le:identity-from-CR}, \eqref{eqn:plusexpansion2-PMI} follows because mutual information term is non--negative, and \eqref{eqn:plusexpansion2} is due to the definition of  $\msfV_\delta(t)$.

The chain of inequalities \eqref{eqn:plusexpansion1}--\eqref{eqn:plusexpansion2} shows:
\begin{itemize}
    \item $Q_{\bU\bX_+}$ is also a minimizer attaining: 
    \begin{equation}\label{eqn:plusoptimal} 
       \msfg_\delta (Q_{\bU\bX_+})=\msfV_\delta(t).
    \end{equation}
    Similarly $Q_{\bU\bX_-}$ is a minimizer.
    \item More importantly, $I_Q(\bX_+;\bY_-|\bU V\bY_+)=0$ which is equivalent to $\bX_+\mkv (\olbY_+,V,\bU)\mkv \olbY_-$. However, we also have the trivial Markov chain $\olbY_+\mkv (\bX_+,V)\-- (\olbY_-,\bU)$. It is easy to show that if $A\mkv B\mkv C$ and $B\mkv A\mkv C$ simultaneously hold and $P_{AB}\ll\gg P_AP_B$, then $C$ is independent of $(A,B)$ (see e.g. \cite{AJC2022})\footnote{A simple information theoretic argument is as follows: The two Markov chains imply $\Dkl(P_{C|A=a}\|P_{C|B=b})=0$, a.s. $P_{AB}$. The condition $P_{AB}\gg P_{A}P_B$ ensures that this also holds a.s. $P_AP_B$. Thus by Jensen inequality $I(A;C)=\Dkl(P_{C|A}\|P_C)\leq \E_B[\Dkl(P_{C|A}\|P_{C|B})]=0$. This and the Markov chain $B\mkv A\mkv C$ imply $(A,B)\bot C$}. Observe that for each $V=v$, $Q_{\olbY_+,\bX_+|V=v}\ll\gg Q_{\olbY_+|V=v}Q_{\bX_+}$ (due to the smoothing term $\bW$ in the definition of $\olbY$), hence we can deduce:
    \begin{equation}\label{eqn:partial-independence}
        I_Q(\bX_+\olbY_+;\olbY_-|\bU,V)=I_Q(X_-\olbY_-;\olbY_+|\bU,V)=0,
    \end{equation}
    where the second equality comes from swapping the role of $+$ and $-$. 
\end{itemize}

We now show that \eqref{eqn:partial-independence} yields $I_Q(\bX_+;\bX_-|\bU)=0$. To accomplish this, we invoke the following identities to expand \eqref{eqn:plusexpansion} in a different way:
\begin{equation}
\begin{aligned}
    &\Dkl(Q_{\bX_+\bX_-|\bU}\|P_{\bX}^{\otimes 2}|Q_{\bU}) &= \Dkl(Q_{\bX_+|\bU}\|P_{\bX}|Q_{\bU}) +\Dkl(Q_{\bX_-|\bU}\|P_{\bX}|Q_{\bU})\\
    &\quad+I_Q(\bX_+;\bX_-|\bU)
\end{aligned}
\end{equation}
\begin{equation}
\begin{aligned}
    \Dkl(Q_{\olbY_+\olbY_-|\bU V}\|P_{\bY}^{\otimes 2}|Q_{\bU V}) &=\Dkl(Q_{\olbY_+|\bU V}\|P_{\bY}|Q_{\bU V})+\Dkl(Q_{\olbY_-|\bU V}\|P_{\bY}|Q_{\bU V})\\
    &\quad+I_Q(\bY_+;\bY_-|\bU V)\\
    &\stackrel{\eqref{eqn:partial-independence}}{=}\Dkl(Q_{\olbY_+|\bU V}\|P_{\bY}|Q_{\bU V})+\Dkl(Q_{\olbY_-|\bU V}\|P_{\bY}|Q_{\bU V}).
\end{aligned}
\end{equation}
By substituting this in \eqref{eqn:plusexpansion}, we obtain:
\begin{equation}
\begin{aligned}
    2\msfV_\delta(t)
    &=\msfss \Dkl(Q_{\bX_+\bX_-|\bU}\|P_{\bX}^{\otimes 2}|Q_{\bU})-\Dkl(Q_{\olbY_+\olbY_-|\bU,V}\|P_{\bY}^{\otimes 2}|Q_{\bU V})\nonumber\\
    &=\msfg_\delta(Q_{\bU\bX_+})+\msfg_\delta(Q_{\bU\bX_-})+\msfss I_Q(\bX_+;\bX_-|\bU)\geq 2\msfV_\delta(t),
\end{aligned}
\end{equation}
   where we have used \eqref{eqn:plusoptimal} in the last inequality. Hence we should have $I_Q(\bX_+;\bX_-|\bU)=0$.
   
   In summary, $(U_1,U_2,\bX_+,\bX_-,\bX_1,\bX_2)$ has the following properties:
   \begin{itemize}
       \item Given $\bU=[u_1,u_2]^\top$, $\bX_1$ and $\bX_2$ are independent.
       \item Given $\bU=[u_1,u_2]^\top$, $\bX_+$ and $\bX_-$ are independent.
       \item  $U_i\in\{1,2\}, i=1,2$. Moreover, $(U_1,\bX_1)$ and $(U_2,\bX_2)$ are independent and have common distribution $Q^{\ast}_{U\bX}$.
   \end{itemize}
  The first two items and the multivariate extension of Darmois--Skitovich theorem \cite{ghurye1962characterization} imply $\bX_1$ and $\bX_2$ have to be Gaussian with the same covariance matrix. In particular for $u_1=1, u_2=2$, $Q_{\bX_1|U_1=1,U_2=2}=Q_{\bX_2|U_1=1,U_2=2}=\cN(0,\bC)$. This and the third item yield $Q^{\ast}_{\bX|U=1}=Q^{\ast}_{\bX|U=2}=\cN(0,\bC)$, thus $\bX$ and $U$ are independent under $Q^{\ast}$. Hence $\msfV_\delta(t)=\msfg_\delta(Q^{\ast}_{\bX})$, where $Q^{\ast}_{\bX}=\cN(0,\bC)$. Therefore the infimum in \eqref{eqn:inf-normal-s} is attained by a Gaussian distribution. This concludes the proof of Lemma \ref{le:normal-opt-s}.  

\subsection{Proof of Lemma \ref{le:normal-opt}}
Let $B_{\msfG}=\inf\limits_{\substack{Q_{\bX}:Q_{\bX}\in\FG,\\ \Dkl(Q_{\bX}\|P_{\bX})<\infty}} \msfg_0(Q_{\bX})$, where $\msfg_0(Q_X)=\msfss \Dkl(Q_{\bX}\|P_{\bX})-\Dkl(Q_{\bY|V}\|P_{\bY}|P_V)$. It suffices to show that $\msfg_0(Q_{\bX})\geq B_{\msfG}$, for any $Q_{\bX}$ with $\Dkl(Q_{\bX}\|P_{\bX})< \infty$. 
We need the following observation about the finiteness of the second moment of $\bX\sim Q_{\bX}$.
\begin{claim}
Every distribution $Q_{\bX}$ with $\Dkl(Q_{\bX}\|P_{\bX})<\infty$, has finite second moments in the sense that $\E_Q\left[\|\bX\|^2\right]<\infty$.
\end{claim}
\begin{proof}
The claim follows from  the Donsker--Varadhan variational representation of KL divergence \cite{donsker1983asymptotic}:
\begin{equation}
    \Dkl(Q_{\bX}\|P_{\bX})=\sup_{f:\cX\mapsto \R} \E_Q\left[f(\bX)\right]-\log\E_P\left[\exp(f(\bX))\right].
\end{equation}
Thus the finiteness assumption of KL divergence yields that the function inside the suprimum is always finite. In particular, for $f(\bx)=\frac{\|\bx\|^2}{4}$ we have:
\begin{equation}
    \frac{1}{4}\E_Q\left[\|\bX\|^2\right]-\log\E_P\left[\exp(\frac{\|\bX\|^2}{4})\right]<\infty.
\end{equation}
However, for $P=\cN(0,\bI_{d_{\bX}})$ we have $\E_P\Big[\exp(\frac{\|\bX\|^2}{4})\Big]<\infty$, therefore the second moment $\E_Q[\|\bX\|^2]$ should be finite.
\end{proof}
Assume $\E_Q[\|\bX\|^2]=t<\infty$. Lemma \ref{le:normal-opt-s} shows that for any $\delta>0$, there exists a normal distribution $Q^{(\delta)}_{\bX}$ with $\E_{Q^{(\delta)}}\left[\|\bX\|^2\right]\leq t$ such that $\msfg_\delta(Q^{(\delta)}_{\bX})\leq \msfg_\delta(Q_{\bX})$. Invoking \eqref{eqn:sdelta-DE}, we can relate the function $\msfg_\delta(Q_{\bX})$ to $\msfg_0(Q_{\bX})$  as follows:
\begin{equation}
\begin{aligned}
    \msfg_\delta(Q_{\bX})&= h(\bY+\sqrt{\delta} \bW|V)-\msfss h(\bX)+\frac{1}{2}\E_{\bX}\left[\bX^\top\bB\bX\right]+C_1-\frac{\delta d_{\bY}}{2}\\
    &\geq h(\bY|V)-\msfss h(\bX)+\frac{1}{2}\E_{\bX}\left[\bX^\top\bB\bX\right]+C_1-\frac{\delta d_{\bY}}{2}\\
    &=\msfg_0(Q_{\bX})-\frac{\delta d_{\bY}}{2},
\end{aligned}   
\end{equation}
where we used the fact that adding independent noise $\sqrt{\delta}\bW$ increases the differential entropy. This yields:
\begin{align}\label{eqn:smoothLB}
  \msfg_\delta(Q_{\bX})&\geq  \msfg_\delta(Q^{(\delta)}_{\bX})  \geq \msfg_0(Q^{(\delta)}_{\bX})-\frac{\delta d_{\bY}}{2}\geq  B_{\msfG}-\frac{\delta d_{\bY}}{2},
\end{align}
where we used the definition of $B_{\msfG}$ and the assumption that $Q^{(\delta)}_{\bX}$ is a normal distribution. 
\begin{claim}[Continuity of the function $\msfg_\delta$.]\label{cl:continuity}
For any $Q_{\bX}$ with finite second moments, the following continuity property holds:
\begin{equation}
    \lim_{\delta\downarrow 0} \msfg_\delta(Q_{\bX})=\mathsf{s}_0(Q_{\bX}).
\end{equation}

\end{claim}
We defer the proof of this claim to the end of this section. Putting \eqref{eqn:smoothLB} and the continuity of $\msfg_{\delta}$ at $\delta=0$ together implies:
\begin{equation}
  \msfg_0(Q_{\bX})\geq B_{\msfG}-\lim_{\delta\downarrow 0} \frac{\delta d_{\bY}}{2}=B_{\msfG}. 
\end{equation}
This concludes the proof of Lemma \ref{le:normal-opt}.

\begin{proof}[Proof of Claim \ref{cl:continuity}]
The proof is based on the following continuity lemma.
\begin{lemma}[Lemma 12 in {\cite{AJC2022}}]\label{le:lemma12AJC}
For any $\R^n$--valued random variable $\bU$ with finite second moment (i.e. $\E[\|\bU\|^2]$) and random variable $\bN\sim\cN(0,\bI_n)$, the following continuity property hold:
\begin{equation}
    \lim_{\delta\downarrow 0} h(\bU+\sqrt{\delta}\bN)=h(\bU).
\end{equation}
\end{lemma}
Define $g(v;\delta)=h(\bY+\sqrt{\delta}\bW|V=v)$. For each $v$, the function $g(v;\delta)$ is an increasing function of $\delta$. Further, for $\delta<1$ and conditioned on any $V=v$, $\bY+\sqrt{\delta}\bW$ has a universal upper bound on the second moment. To see this, recall that $\bY=\bA_v \bX+\bZ_v$, where $\|\bA_v\|_{\op}\leq 1$ and $\bZ_v\sim\cN(0,\bI_{d_{\bY}}-\bA_v\bA_v^\top)$. Thus for $\delta<1$:
\begin{equation}
\begin{aligned}
    \E_Q\left[\|\bY+\sqrt{\delta}\bW\|^2|V=v\right]&=\E_Q\left[\|\bA_v\bX\|^2]+\E_Q[\|\bZ_v\|^2\right]+\delta d_{\bY}\\
    &\leq \E_Q\left[\|\bX\|^2\right]+(1+\delta) d_{\bY}<K,
\end{aligned} 
\end{equation}
where $K=\E_Q\left[\|\bX\|^2\right]+2d_{\bY}<\infty$ does not depend on $v$. This implies $g(v;\delta)\leq h(\bA)=\frac{d_{\bY}}{2}\log\frac{2\pi e K}{d_{\bY}}<\infty$, where $\bA\sim\cN(0,\frac{K}{d_{\bY}}\bI_{d_{\bY}})$. 

Also, by Lemma \ref{le:lemma12AJC}, we have  $\lim\limits_{\delta\downarrow 0}g(v;\delta)=g(v,0)$. Now the monotone convergence theorem \cite{rudin1987real}  implies:
\begin{equation}
  \lim_{\delta\downarrow 0} h(\bY+\sqrt{\delta}\bW|V)= \lim_{\delta\downarrow 0} \E_V\left[g(V;\delta)\right]=\E_V\left[g(V;0)\right]=h(\bY|V).
\end{equation}
Now, we are ready to complete the proof of the Claim. Invoking \eqref{eqn:sdelta-DE}, we have:
\begin{equation}
\begin{aligned}
    \lim_{\delta\downarrow 0}\msfg_\delta(Q_{\bX}) &=\lim_{\delta\downarrow 0} h(\bY+\sqrt{\delta} \bW|V)-\msfss h(\bX)+\frac{1}{2}\E_{\bX}\left[\bX^\top\bB\bX\right]+C_1-\frac{\delta d_{\bY}}{2}\\
    &= h(\bY|V)-\msfss h(\bX)+\frac{1}{2}\E_{\bX}\left[\bX^\top\bB\bX\right]+C_1\\
    &=\msfg_0(Q_{\bX}),
\end{aligned}   
\end{equation}
which is the statement of Claim \ref{cl:continuity}.
\end{proof}
\section{SDPI Coefficient of The Gaussian Mixture Channel}\label{app:proof_prop_GMC}
We state the proof of Proposition \ref{prop:GMC} here.

\begin{proof}[Proof of Proposition \ref{prop:GMC}]
The direction $s(P_{\bX}, \widetilde{T}_{\bY|\bX}) \leq \|\E[\bA_V^\top \bA_V]\|_{\op}$ has been established in Subsection \ref{sub:GMC}. Therefore, our primary objective is to prove the converse: $s(P_{\bX}, \widetilde{T}_{\bY|\bX}) \geq \|\E[\bA_V^\top \bA_V]\|_{\op}$.

For a scalar $\beta \ge 1$ and a positive semi--definite matrix $\bB$ such that $\bI \preceq \bB$, let $Q_{\bX}^{(\beta,\bB)}$ be a Gaussian distribution defined as $Q_{\bX}^{(\beta,\bB)} = \cN(\mathbf{0}, \bI+\beta\bB)$. The SDPI coefficient can then be bounded as: 
\[
s(P_{\bX},\widetilde{T}_{\bY|\bX}) \geq \sup_{\beta,\bB}\frac{\Dkl(Q_{\bY}^{(\beta,\bB)}\|P_{\bY})}{\Dkl(Q_{\bX}^{(\beta,\bB)}\|P_{\bX})} \geq \sup_{\bB}\lim_{\beta\to \infty}\frac{\Dkl(Q_{\bY}^{(\beta,\bB)}\|P_{\bY})}{\Dkl(Q_{\bX}^{(\beta,\bB)}\|P_{\bX})},
\]
where $Q_{\bY}^{(\beta,\bB)}$ is the output distribution of the channel $\widetilde{T}_{\bY|\bX}$ when the input distribution is $Q_{\bX}^{(\beta,\bB)}$, and $P_{\bX} = \cN(\mathbf{0}, \bI_{d_{\bX}})$, $P_{\bY} = \cN(\mathbf{0}, \bI_{d_{\bY}})$ are the reference input and output distributions (standard normal distributions in $d_{\bX}$ and $d_{\bY}$ dimensions, respectively).
We first focus on evaluating the limit $\lim\limits_{\beta\to \infty}\frac{\Dkl(Q_{\bY}^{(\beta,\bB)}\|P_{\bY})}{\Dkl(Q_{\bX}^{(\beta,\bB)}\|P_{\bX})}$.

\paragraph*{Analyzing $\Dkl(Q_{\bX}^{(\beta,\bB)}\|P_{\bX})$:}
The KL divergence between $Q_{\bX}^{(\beta,\bB)} = \cN(\mathbf{0}, \bI+\beta\bB)$ and $P_{\bX} = \cN(\mathbf{0}, \bI)$ is given by:
\begin{align}
    \Dkl(Q_{\bX}^{(\beta,\bB)}\|P_{\bX})&=\frac{1}{2}\Tr{\beta\bB-\log(\bI+\beta\bB)}\\
    &=\frac{1}{2}\left(\beta\Tr{\bB}-\sum_{k=1}^{d_{\bX}}\log(1+\beta\lambda_i(\bB))\right)\\
    &=\frac{1}{2}\left(\beta\Tr{\bB}-\cO\left(\log \beta\right)\right)
\end{align}
where $\lambda_i(\bB)$ are the eigenvalues of $\bB$. Since $\bI \preceq \bB$, all eigenvalues $\lambda_i(\bB) \ge 1$. For large $\beta$, $\log(1+\beta\lambda_i(\bB)) \approx \log(\beta\lambda_i(\bB)) = \log\beta + \log\lambda_i(\bB)$. Thus, the sum can be approximated as $\sum_{i=1}^{d_{\bX}}(\log\beta + \log\lambda_i(\bB)) = d_{\bX} \log\beta + \sum_{i=1}^{d_X}\log\lambda_i(\bB)$.
Therefore, for large $\beta$:
$$ \Dkl(Q_{\bX}^{(\beta,\bB)}\|P_{\bX}) = \frac{1}{2}\left(\beta\Tr{\bB} - \cO(\log \beta)\right) $$
The $\cO(\log \beta)$ term encompasses constants depending on $d_{\bX}$ and $\bB$.

\paragraph*{Analyzing $\Dkl(Q_{\bY}^{(\beta,\bB)}\|P_{\bY})$:}
The KL divergence for the output distribution is given by:
\begin{align*}
    \Dkl(Q_{\bY}^{(\beta,\bB)}\|P_{\bY}) &= -h(Q_{\bY}^{(\beta,\bB)}) + \E_{\bY \sim Q_{\bY}^{(\beta,\bB)}}\left[\log\frac{1}{P_{\bY}(\bY)}\right] \\
    &= -h(Q_{\bY}^{(\beta,\bB)}) + \frac{d_{\bY}}{2}\log(2\pi) + \frac{1}{2}\E_{\bY \sim Q_{\bY}^{(\beta,\bB)}}\left[\|\bY\|^2\right] \\
    &= -h(Q_{\bY}^{(\beta,\bB)}) + \frac{d_{\bY}}{2}\log(2\pi) + \frac{1}{2}\Tr{\E_{\bY \sim Q_{\bY}^{(\beta,\bB)}}\left[\bY\bY^\top\right]},
\end{align*}
where $h(Q_{\bY}^{(\beta,\bB)})$ is the differential entropy of $Q_{\bY}^{(\beta,\bB)}$.

Let's simplify each term. First, consider the expected outer product $\E_{\bY \sim Q_{\bY}^{(\beta,\bB)}}\left[\bY\bY^\top\right]$. We can use the law of total expectation:
\begin{align*}
    \E_{\bY \sim Q_{\bY}^{(\beta,\bB)}}\left[\bY\bY^\top\right] &= \E_V\left[\E_{\bY|V,\bX \sim Q_{\bX}^{(\beta,\bB)}}\left[\bY\bY^\top|V\right]\right] \\
    &= \E_V\left[\cov(\bY|V,\bX) + \E[\bY|V,\bX]\E[\bY|V,\bX]^\top\right] \\
    &= \E_V\left[(\bI-\bA_V\bA_V^\top) + \bA_V\E[\bX\bX^\top]\bA_V^\top\right].
\end{align*}
Since $\bX \sim Q_{\bX}^{(\beta,\bB)} = \cN(\mathbf{0}, \bI+\beta\bB)$, we have $\E[\bX\bX^\top] = \bI+\beta\bB$. Substituting this:
\begin{align*}
    \E_{\bY \sim Q_{\bY}^{(\beta,\bB)}}\left[\bY\bY^\top\right] &= \E_V\left[\bI-\bA_V\bA_V^\top + \bA_V(\bI+\beta\bB)\bA_V^\top\right] \\
    &= \bI + \beta\E_V\left[\bA_V\bB\bA_V^\top\right].
\end{align*}
Taking the trace:
$$ \Tr{\E_{\bY \sim Q_{\bY}^{(\beta,\bB)}}\left[\bY\bY^\top\right]} = \Tr{\bI + \beta\E_V\left[\bA_V\bB\bA_V^\top\right]} = d_{\bY} + \beta\Tr{\E_V\left[\bA_V^\top\bA_V\right]\bB}.$$

Next, we need to analyze the differential entropy $h(Q_{\bY}^{(\beta,\bB)})$.
We know that the differential entropy of a distribution is bounded by the entropy of a Gaussian distribution with the same covariance matrix \cite[Theorem 8.6.5]{cover1999elements}. The covariance matrix of $Q_{\bY}^{(\beta,\bB)}$ is $\cov_Q(\bY) = \E_{Q_{\bY}^{(\beta,\bB)}}\left[\bY\bY^\top\right] = \bI + \beta\E_V\left[\bA_V\bB\bA_V^\top\right]$.
The trace of this covariance matrix is $\Tr{\cov_Q(\bY)} = d_{\bY} + \beta\Tr{\E_V\left[\bA_V^\top\bA_V\right]\bB}$.

We can establish bounds for $h(Q_{\bY}^{(\beta,\bB)})$:
\begin{enumerate}
    \item Lower Bound: By the conditioning inequality for entropy, $h(\bY) \ge h(\bY|V)$. Since $\bY|\{V,\bX=\bx\} \sim \cN(\bA_V \bx, \bI-\bA_V\bA_V^\top)$ and $\bX \sim Q_{\bX}^{(\beta,\bB)}$, we have $\bY|\{V=v\} \sim  \cN(\mathbf{0}, \bI + \beta\bA_v\bB\bA_v^\top)$.
    Thus:
    \begin{align*} h(\bY|V=v) &= \frac{1}{2}\log\det{(2\pi e (\bI + \beta\bA_v\bB\bA_v^\top))} \\ &\ge \frac{1}{2}\log\det{(2\pi e \bI) }= \frac{d_{\bY}}{2}\log(2\pi e) = \cO(1),
    \end{align*}
    where we have used the fact that $\bA_v\bB\bA_v^\top\succeq \bzero$, since $\bB\succeq \bzero$.
    This implies $h(Q_{\bY}^{(\beta,\bB)}) \geq \E_V[h(\bY|V)] = \cO(1)$.
    \item Upper Bound: Let $G_{\bY} = \cN(\mathbf{0}, \cov_Q(\bY))$ be a Gaussian distribution with the same covariance as $Q_{\bY}^{(\beta,\bB)}$. Then $h(Q_{\bY}^{(\beta,\bB)}) \le h(G_{\bY})$. The entropy of $G_{\bY}$ is:
    $$ h(G_{\bY}) = \frac{1}{2}\log\det{(2\pi e \cov_Q(\bY)) }.$$
    Since $\cov_Q(\bY) = \bI + \beta\E_V\left[\bA_V\bB\bA_V^\top\right]$, and $\bA_V\bB\bA_V^\top$ is positive semi--definite (as $\bB$ is positive semi--definite), the eigenvalues of $\cov_Q(\bY)$ will grow linearly with $\beta$. Therefore, $\log\det{(\cov_Q(\bY)) }= \cO(\log \beta)$.
    Thus, $h(Q_{\bY}^{(\beta,\bB)}) \le h(G_{\bY}) = \cO(\log \beta)$.
\end{enumerate}
Combining the lower and upper bounds, we conclude that $|h(Q_{\bY}^{(\beta,\bB)})| = \cO(\log \beta)$.

Now, substituting these findings back into the expression for $\Dkl(Q_{\bY}^{(\beta,\bB)}\|P_{\bY})$:
$$ \Dkl(Q_{\bY}^{(\beta,\bB)}\|P_{\bY}) = \frac{1}{2}\beta \Tr{\E_V\left[\bA_V^\top\bA_V\right]\bB} + \cO(\log\beta).$$

Finally, we take the limit of the ratio:
$$ \lim_{\beta\to \infty}\frac{\Dkl(Q_{\bY}^{(\beta,\bB)}\|P_{\bY})}{\Dkl(Q_{\bX}^{(\beta,\bB)}\|P_{\bX})} =\lim_{\beta\to \infty} \frac{\frac{1}{2}\beta \Tr{\E_V\left[\bA_V^\top\bA_V\right]\bB} + \cO(\log\beta)}{\frac{1}{2}\left(\beta\Tr{\bB} - \cO(\log \beta)\right)}.$$
Dividing both numerator and denominator by $\beta$ and taking the limit as $\beta \to \infty$:
$$ \lim_{\beta\to \infty}\frac{\Dkl(Q_{\bY}^{(\beta,\bB)}\|P_{\bY})}{\Dkl(Q_{\bX}^{(\beta,\bB)}\|P_{\bX})} = \frac{\Tr{\E_V\left[\bA_V^\top\bA_V\right]\bB}}{\Tr{\bB}}.$$

To complete the proof, we need to maximize this expression over all valid matrices $\bB$:
$$ \sup_{\bB: \bI \preceq \bB} \frac{\Tr{\E_V\left[\bA_V^\top\bA_V\right]\bB}}{\Tr{\bB}} $$
Let $\bM = \E_V\left[\bA_V^\top\bA_V\right]$. We are maximizing $\frac{\Tr{\bM\bB}}{\Tr{\bB}}$. It is a known result in matrix analysis (e.g., related to Rayleigh quotients and Courant--Fischer theorem for matrices) that for a positive semi--definite matrix $\bM$ and a positive definite matrix $\bB$, this ratio is maximized when $\bB$ is chosen to align with the dominant eigenvector of $\bM$. Specifically, the supremum of this ratio is equal to the largest eigenvalue of $\bM$, which is its operator norm.
$$ \sup_{\bB: \bI \preceq \bB} \frac{\Tr{\bM\bB}}{\Tr{\bB}} = \|\bM\|_{\op} = \norm[2]{\E_V\left[\bA_V^\top\bA_V\right]}_{\op}.$$
Thus, we have shown that $s(P_{\bX}, \widetilde{T}_{\bY|\bX}) \geq \norm[1]{\E_V[\bA_V^\top \bA_V]}_{\op}$. Combining this with the previously established upper bound, the equality holds.

\end{proof}
\section{Proof Completion for Theorems \ref{thm:lower_bound-cross} and \ref{thm:lower_bound}}
\label{app:proof_completion_lower_bound}

In this section, we complete the proof of Theorems \ref{thm:lower_bound-cross} and \ref{thm:lower_bound}. Note that the main part of proof is stated in Section \ref{sec:Proof_lower_bound} and we complete the proof here.

\subsection{Lower Bounds  Related to Sample Complexity}\label{sec:sample_complexity}
Up to this point, we have established bounds related to the constraints on the communication budget. We now turn our attention to the bounds related to the sample size $m$.
Here, we employ the standard Fano method. We analyze the cross--covariance case and the full covariance case separately.
 
\subsubsection{Cross Covariance}
We utilize the same family of distributions employed in Subsections \ref{subse:family} and \ref{se:packing}. The sole difference is the omission of the random variable $W$, which was used in the averaged Fano's method. More precisely, consider a set $\cV=[1:|\cV|]$ and a corresponding family of distributions $\cP_{\cV}=\{P_v\}_{v\in\cV}$, where $P_v = \cN(\bzero, \bC_{v})$, and:
\begin{equation}\label{eq:matrix-representation-n-cr}
    \bC_{v} = \frac{\sigma^2}{2}\begin{bmatrix}
         \bI_{d_1} &  \delta\bD_{v}^\top\\
        \delta\bD_{v} &  \bI_{d_2}
    \end{bmatrix},
\end{equation}
where  $\bD_{v}$ is some matrix in $\R^{d_1\times d_2}$ with $\norm{\bD_{v}}_\op\leq 1$, and $\delta\leq 1$ is a parameter to be determined subsequently. Analogous to \eqref{eq:mutual_information_bound1_app}, we have $I(V;M_1,M_2)\leq I(M_1;M_2|V)$. Furthermore, for a fixed $V=v$, we have the following Markov chain: $M_1\mkv\bmsfX_1\mkv\bmsfX_2\mkv M_2$. The data processing inequality then implies $I(M_1;M_2|V)\leq I(\bmsfX_1;\bmsfX_2|V)$. In summary, we have:
\begin{equation}
    I(V;M_1,M_2)\leq I(\bmsfX_1;\bmsfX_2|V).
\end{equation}
We now proceed to derive an upper bound on $I(\bmsfX_1;\bmsfX_2|V)$:
\begin{equation}\label{eq:mutual_information_bound5_app}
\begin{split}
    I(\left.\bmsfX_1;\bmsfX_2\right|V=v) 
    &= m\; I\left(\bX_1 ; \bX_2| V=v\right)\\
    &= m \left[h(\bX_1|V=v) + h(\bX_2|V=v) - h(\bX_1,\bX_2|V=v)\right]\\
    &\overset{\text{(a)}}{=} \frac{m}{2} \log_2\left(\frac{\det{\frac{\sigma^2}{2} \bI_{d1}}\det{\frac{\sigma^2}{2} \bI_{d2}}}{\det{\bC_{v}}}\right)\\
    &= \frac{m}{2}\left(2r_{v}\log_2(\frac{\sigma^2}{2}) - \sum_{i=1}^{r_{v}}\log_2\left(\frac{\sigma^4}{4}\left(1-\delta^2\sigma_i^2(\bD_{v}\right)\right)\right)\\
   &= -\frac{m}{2}\left( \sum_{i=1}^{r_{v}}\log_2\left(1-\delta^2\sigma_i^2(\bD_{v}\right)\right)\\
   &\overset{\text{(b)}}{\leq} \frac{-m r_{v}}{2} \log_2\left(1-{\delta^2}\right)\\
    &\overset{\text{(c)}}{\leq} {m r_{v}\delta^2}\\
    &\overset{\text{(d)}}{\leq} {m (d_1\wedge d_2)\delta^2},
\end{split}
\end{equation}
where (a) follows from \citep[~Theorem 8.4.1]{cover1999elements} and $r_{v}=\rank(\bD_{v})$. Furthermore, (b) holds because $g(x)=-\log_2(1-x)=\log_2(\frac{1}{1-x})$ is an increasing function, and $\norm{\bD}_\op\le 1$. (c) follows from the fact that for all $x\in[0,1/2]$, we have $\log_2(\frac{1}{1-x})\le {2x}$, and we assume that $\delta\leq \frac{1}{\sqrt{2}}$.    
    
Then,  Lemma \ref{lem:main_fano} yields:
\begin{equation}\label{eqn:y-8kh-1}
\begin{split}
    \ME &\geq \frac{\rho_{\dist}^{(\cross)}}{2}\left[1-\frac{I(V;M_1,M_2)+1}{\log_2(|\cV|)}\right]\\
    &\geq \dfrac{\delta\sigma^2\nu_{\dist}^{(d_1,d_2)}}{8}\left[1-\frac{1+m(d_1\wedge d_2)\delta^2}{2d_1d_2}\right],
\end{split}
\end{equation}
where $\nu_\dist^{(d_1,d_2)}$ was defined earlier, and we have used \eqref{eq:packing-unit-op-ball}. 
Setting $\delta=\sqrt{\frac{d_1\vee d_2}{2m}\wedge \frac{1}{2}} $ implies: 
\begin{equation}\label{eqn:y-8kh-2}
\begin{split}
     \MEopc&\geq\frac{\sigma^2}{32}\left( \alphaopsc\bigwedge 2\right)\\
     \MEFc&\geq\frac{\sigma^2}{32}\left( \alphaFsccross\bigwedge\frac{\sqrt{d_1\wedge d_2}}{7}\right)
\end{split}
\end{equation}
	   
\subsubsection{Full Covariance}

As mentioned in Remark~\ref{rem:samplecomplexity}, any lower bound in the centralized setting—where all data are aggregated on a central server—also applies to the distributed setting. It is folklore that the operator norm distortion in the centralized setting has a lower bound of $\Omega\left(\sqrt{\frac{d}{m}} \wedge 1\right)$, while the Frobenius norm distortion admits a lower bound of $\Omega\left(\sqrt{\frac{d^2}{m}} \wedge \sqrt{d}\right)$~\citep{ashtiani2020near,devroye2020minimax}. The arguments in~\citep{ashtiani2020near} and~\citep{devroye2020minimax} are for $d \geq 9$ and $d \geq 5$, respectively. However, the argument used for the distributed setting also applies to the centralized case for $d \geq 2$.

To establish this, we split the vector $\bmsfZ = \{\bZ^{(i)}\}_{i=1}^{m}$ into two equal--length subvectors:  
$\bmsfX'_1 = \{\bZ_{[1:d/2]}^{(i)}\}_{i=1}^{m}$ and $\bmsfX'_2 = \{\bZ_{[d/2+1:d]}^{(i)}\}_{i=1}^{m}$.  
We then use the same distributed argument with a new family of distributions:  
$P_v = \mathcal{N}(\mathbf{0},  \bC_v)$, where the covariance matrix has the form
\begin{equation} \label{eq:matrix-representation-n-full}
    \bC_{v} = \frac{\sigma^2}{2}\begin{bmatrix}
         \bI_{d/2} & \delta \bD_{v}^\top \\
         \delta \bD_{v} & \bI_{d/2}
    \end{bmatrix}.
\end{equation}

Note that within the normal distributions characterized by \eqref{eq:matrix-representation-n-full}, $\bmsfX'_1$ and $V$ are independent.  Similarly, $\bmsfX'_2$ and $V$ are independent. By the Markov chain $V \mkv \bmsfZ = (\bmsfX'_1, \bmsfX'_2) \mkv (M_1,M_2)$ and the data processing inequality (DPI), we obtain:
\begin{equation} \label{eq:mutual_information_bound4_app-y}
    \begin{split}
        I(V; M_1, M_2) &\leq I(V; \bmsfX'_1, \bmsfX'_2) \\
        &= I(V; \bmsfX'_1) + I(V; \bmsfX'_2)+I( \bmsfX'_1;\bmsfX'_2|V)-I(\bmsfX'_1;\bmsfX'_2) \\
        &\le I(\bmsfX'; \bmsfY' \mid V).
    \end{split}
\end{equation}

The rest of the proof follows similarly to the distributed case. Specifically, one can show
\[
I(\bmsfX'_1; \bmsfX'_2 \mid V) \leq \frac{md\delta^2}{2},
\]
as in~\eqref{eq:mutual_information_bound5_app}. Then, applying the same inequalities used in~\eqref{eqn:y-8kh-1} and~\eqref{eqn:y-8kh-2}, we obtain:
 \begin{equation}\label{eqn:y-8kh-3}
\begin{split}
     \MEopc&\geq\frac{\sigma^2}{32}\left( \alphaopsc\bigwedge 2\right)\\
     \MEFc&\geq\frac{\sigma^2}{32}\left( \alphaFsc\bigwedge\frac{\sqrt{d_1\wedge d_2}}{7}\right)
\end{split}
\end{equation}

\subsection{Lower bounds related to communication budget for self--covariance estimation.}
In this section, we utilize the second family of Gaussian distributions. For the set $\cU$, consider the set of distributions $\{P_u\}_{u\in\cU}$, where $P_u = \cN(\bzero, \bC_{u})$ and:
\[\resizebox{!}{1cm}{$\bC_{u} = \frac{\sigma^2}{2}\left[\begin{matrix}
     \bI_{d_1/2} &  \delta\bD_{u} & \bzero & \bzero\\
    \delta\bD_{u}^{\top} &  \bI_{d_1/2} & \bzero & \bzero\\
    \bzero & \bzero &  \bI_{d_2/2} &  \bzero\\
    \bzero & \bzero & \bzero &  \bI_{d_2/2}
\end{matrix}\right],$}\]
where $\bD_{u}$ is a matrix in $\R^{d_1/2\times d_1/2}$.  

As in the previous case, we must have $\bC_{u}\succeq\bzero$ and $\norm{\bC_{u}}_{\op}\leq\sigma^2$. Note that the eigenvalues of $\bC_{u}$ are $\left\{\frac{\sigma^2}{2}\left(1\pm\delta\sigma_i(\bD_{u})\right)\right\}_{i=1}^{\rank(\bD_{u})}$. Thus, if we assume that $\norm{\bC_{u}}_{\op}\leq1$ and $\delta\leq 1$, the conditions $\bC_{u}\succeq\bzero$ and $\norm{\bC_{u}}_{\op}\leq\sigma^2$ are satisfied.

We then write:
\begin{equation}
\begin{split}
  \rho_{\dist} &= \inf\limits_{u,u':u\neq u'}\norm{\bC_u-\bC_{u'}}_{\dist}\\
    &= \frac{\sigma^2}{2}\inf\limits_{u,u':u\neq u'}\norm{\left[\begin{matrix}
        \bzero &  \delta(\bD_{u}-\bD_{u'}) & \bzero & \bzero\\
        \delta(\bD_{u}-\bD_{u'})^{\top} & \bzero & \bzero & \bzero\\
        \bzero & \bzero & \bzero &  \bzero\\
        \bzero & \bzero & \bzero &  \bzero
    \end{matrix}\right]}_{\dist}\\
    &= \frac{\sigma^2}{2}\inf\limits_{u,u':u\neq u'} \norm{\left[\begin{matrix}
        \bzero &  \delta(\bD_{u}-\bD_{u'}) \\
        \delta(\bD_{u}-\bD_{u'})^{\top} & \bzero
    \end{matrix}\right]}_{\dist}\\
    &\overset{\text{(a)}}{=} \sqrt{1+\ind{\dist=\Fr}}\frac{\sigma^2\delta}{2}\inf\limits_{u,u':u\neq u'}\norm{\bD_u-\bD_{u'}}_{\dist},
\end{split}
\end{equation}
where (a) follows from Lemma \ref{lem:eigenvals_and_eigenvects_of_matrix}.

We also derive an upper bound for $I(U;M_1,M_2)$:
\begin{equation}
\begin{split}
	I(U;M_1,M_2) &= I(U;M_1) + I(U;M_2|M_1)\\
    &\leq I(U;M_1) + I(U;M_2|M_1) + I(M_1;M_2)\\
    &= I(U;M_1) + I(U,M_1;M_2)\\
    &= I(U;M_1) + I(U;M_2) + I(M_1;M_2|U)\\
    &\leq I(U;M_1) + I(U;\bY) + I(\bX;\bY|U)\\
    &= I(U;M_1)\\
    &\leq B_1.
\end{split}
\end{equation}
Next, if we define the set $\{\bD_{u}\}_{u\in\cU}$ as the $\epsilon$--packing points of $\cB_{\normiii{.}_{\op}}^{(d_1^2/4)}(1)$, under the $\normiii{.}_{\dist}$ norm, we have $\inf\limits_{u,u':u\neq u'}\norm{\bD_u-\bD_{u'}}_{\dist} \geq \epsilon$, $\max\limits_{u\in\cU}\left\{\big\|\bD_{u}\big\|_{\op}\right\}\leq 1$, and $\log_2(|\cU|) \geq\frac{d_1^2}{4}\log_2\left(\frac{\nu_{\dist}^{(d_1/2)}}{\epsilon}\right)$, where:
\begin{equation}
    \nu_{\dist}^{(d_1/2)}=\begin{cases}
      1 & \text{if } \dist=\op\\
      \frac{\sqrt{d_1}}{14\sqrt{2}} & \text{if } \dist=\Fr
    \end{cases}.
\end{equation}
Now, if we set $\epsilon = \nu_{\dist}^{(d_1/2)}\cdot2^{\frac{-16B_1}{d_1^2}}$ and $\delta=1$, from Lemma \ref{lem:main_fano} we have the following minimax lower bound:
\begin{equation}
\begin{split}
    \ME(\sigma,B_1,B_2,d_1,d_2,m)&\geq \frac{\rho_{\dist}}{2}\left[1-\frac{I(U;M_1,M_2)+1}{\log_2(|\cU|)}\right]\\
    &\geq \frac{\sqrt{1+\ind{\dist=\Fr}}\sigma^2}{2}\left[1-\frac{B_1+1}{\log_2(|\cU|)}\right]\inf\limits_{u,u':u\neq u'}\norm{\bD_u-\bD_{u'}}_{\op}\\
    &\geq \frac{\sqrt{1+\ind{\dist=\Fr}}\sigma^2}{2}\left[1-\frac{B_1+1}{\frac{d_1^2}{4}\log_2\left(\frac{\nu_{\dist}^{(d_1/2)}}{\epsilon}\right)}\right]\epsilon\\
    &= \frac{\sigma^2}{4}\sqrt{1+\ind{\dist=\Fr}}\nu_{\dist}^{(d_1/2)}\cdot2^{\frac{-16B_1}{d_1^2}}.
\end{split}
\end{equation}
Therefore, we have:
\begin{equation}
\begin{aligned}
    \MEop(\sigma,B_1,B_2,d_1,d_2,m) &\geq \frac{\sigma^2}{4}\cdot\left(2^{\frac{-16B_1}{d_1^2}}\bigvee 2^{\frac{-16B_2}{d_2^2}}\right),\\
    \MEF(\sigma,B_1,B_2,d_1,d_2,m) &\geq \frac{\sigma^2}{56}\left(\sqrt{d_1}\cdot2^{\frac{-16B_1}{d_1^2}}\bigvee \sqrt{d_2}\cdot2^{\frac{-16B_2}{d_2^2}}\right).\\
\end{aligned}
\end{equation}

\section{Some Concentration Inequalities for Random Matrices}
\label{app:material_useful_achievable}

In this section, we obtain two lemmas and one proposition that are useful in proving Theorem \ref{thm:achievable_scheme}.

\begin{lem}\label{lem:norm_upperbound_selfcovar}
    Assume that $\bX\in\R^{d_1}$ is a zero mean, sub--Gaussian vector with parameter $\sigma_1$, and we have $m$ i.i.d. samples from $\bX$ as $\{\bX^{(i)}\}_{i=1}^{m}$. Also assume that $\bY\in\R^{d_2}$ is a zero mean, sub--Gaussian vector with parameter $\sigma_2$, and we have $m$ i.i.d. samples from $\bY$ as $\{\bY^{(i)}\}_{i=1}^{m}$. Consider the cross--covariance matrix $\bC_{\bX\bY}\in\R^{d_1\times d_2}$ as $\bC_{\bX\bY}=\E[\bX\bY^\top]$ and assume that we use the estimator $\bCt_{\bX\bY}=\frac{1}{m}\sum\limits_{i=1}^{m}\bX^{(i)}\bY^{(i)\top}$. Then we have: 
    \begin{equation*}
        \begin{split}
            \Pp\Big[\norm[1]{\bCt_{\bX\bY}-\bC_{\bX\bY}}_{\op}&\geq 10\sigma_1\sigma_2t\Big]\\
            &\leq  (9)^{d_1+d_2} \exp\Big(-m.\min\big\{t,t^2\big\}\Big),
        \end{split}
    \end{equation*}
    and:
    \begin{equation*}
        \begin{split}
            \Pp\Big[\norm[1]{\bCt_{\bX\bY}}_{\op}&\geq 11\sigma_1\sigma_2\Big]\\
            &\leq \min\left\{1, \exp\big(3(d_1+d_2)-m\big)\right\}.
        \end{split}
    \end{equation*}
\end{lem}

\begin{proof}
    We use Lemma \ref{lem:norm_upperbound_covering} with $\epsilon=\frac{1}{4}$ and write:
    \begin{equation}\label{eq:norm_Ch_upperbound1}
        \begin{split}
            \Pp\left[\norm[1]{\bCt_{\bX\bY}-\bC_{\bX\bY}}_{\op}\geq t\right] &\leq \Pp\bigg[\max\limits_{\bu\in\cN_{1/4}^{(d_1)},\bv\in\cN_{1/4}^{(d_2)}}\bu^{\top} (\bCt_{\bX\bY}-\bC_{\bX\bY}) \bv\geq \frac{t}{2}\bigg]\\
            &\leq \sum\limits_{j=1}^{|\cN_{1/4}^{(d_1)}|} \sum\limits_{k=1}^{|\cN_{1/4}^{(d_2)}|} \Pp\bigg[\bu^{(j)\top} (\bCt_{\bX\bY}-\bC_{\bX\bY}) \bv^{(k)}\geq \frac{t}{2}\bigg],
        \end{split}
    \end{equation}
    where we denote the $1/4$--covering points of $\cS^{d_1-1}$ by $\{\bu^{(j)}\}_{j=1}^{|\cN_{1/4}^{(d_1)}|}$ and the $1/4$--covering points of $\cS^{d_2-1}$ by $\{\bv^{(k)}\}_{k=1}^{|\cN_{1/4}^{(d_2)}|}$. We also know from \citep[~Corollary 4.2.13]{vershynin2018high} that $\cN_{1/4}^{(d)}\leq 9^d$.

    We have: 
    \begin{equation}\label{eq:norm_Ch_upperbound2}
        \begin{split}
            &\Pp\bigg[\bu^{(j)\top} (\bCt_{\bX\bY}-\bC_{\bX\bY}) \bv^{(k)}\geq \frac{t}{2}\bigg] = \Pp\bigg[\bu^{(j)\top} (\frac{1}{m}\sum\limits_{i=1}^{m}\bX^{(i)}\bY^{(i)\top}-\E[\bX\bY^{\top}]) \bv^{(k)}\geq \frac{t}{2}\bigg]\\
            &= \Pp\bigg[\frac{1}{m} \sum\limits_{i=1}^{m}(\bu^{(j)\top}\bX^{(i)})(\bv^{(k)\top}\bY^{(i)})-\E\Big[(\bu^{(j)\top}\bX^{(i)})(\bv^{(k)\top}\bY^{(i)})\Big]\geq \frac{t}{2}\bigg].
        \end{split}
    \end{equation}
    We know that $\bX^{(i)}$ is a $\sigma_1$--sub--Gaussian vector, therefore, from Definition \ref{def:sub_gaussian_random_vector}, we conclude that $U_i=\bu^{(j)\top}\bX^{(i)}$ is a $\sigma_1$--sub--Gaussian random variable. Similarly we conclude that $V_i=\bv^{(k)\top}\bY^{(i)}$ is a $\sigma_2$--sub--Gaussian random variable.  Therefore, from Lemma \ref{lem:multiple_of_sub_Gaussian_Gamma}, $U_iV_i-\E[U_iV_i]$ is a $(\sigma=5\sigma_1\sigma_2,\alpha=2.5\sigma_1\sigma_2)$--sub--Gamma random variable. Corollary \ref{cor:Hoeffding_multiple_sub_Gaussian} yields:
    \begin{equation}\label{eq:norm_Ch_upperbound3}
        \begin{split}
            &\Pp\bigg[\bu^{(j)\top} (\bCt_{\bX\bY}-\bC_{\bX\bY}) \bv^{(k)}\geq \frac{t}{2}\bigg]\\
            &= \Pp\bigg[\frac{1}{m} \sum\limits_{i=1}^{m}(\bu^{(j)\top}\bX^{(i)})(\bv^{(k)\top}\bY^{(i)})-\E\Big[(\bu^{(j)\top}\bX^{(i)})(\bv^{(k)\top}\bY^{(i)})\Big]\geq \frac{t}{2}\bigg]\\
            &= \Pp\bigg[\frac{1}{m} \sum\limits_{i=1}^{m}(U_iV_i-\E[U_iV_i])\geq \frac{t}{2}\bigg]\\
            &\leq \exp\bigg(-m.\min\Big\{\frac{t}{10\sigma_1\sigma_2},\big(\frac{t}{10\sigma_1\sigma_2}\big)^2\Big\}\bigg).
        \end{split}
    \end{equation}
    Therefore we combine \eqref{eq:norm_Ch_upperbound1}, \eqref{eq:norm_Ch_upperbound2}, and \eqref{eq:norm_Ch_upperbound3} and write:
    \begin{equation}
        \begin{split}
            \Pp\left[\norm[1]{\bCt_{\bX\bY}-\bC_{\bX\bY}}_{\op}\geq 10\sigma_1\sigma_2t\right] &\leq \sum\limits_{j=1}^{|\cN_{1/4}^{(d_1)}|} \sum\limits_{k=1}^{|\cN_{1/4}^{(d_2)}|} \Pp\bigg[\bu^{(j)\top} (\bCt_{\bX\bY}-\bC_{\bX\bY}) \bv^{(k)}\geq 5\sigma_1\sigma_2 t\bigg]\\
            &\leq (9)^{d_1+d_2} \exp\bigg(-m.\min\Big\{t,t^2\Big\}\bigg).
        \end{split}
    \end{equation}
    Thus:
    \begin{equation}
        \begin{split}
            \Pp\left[\norm[1]{\bCt_{\bX\bY}}_{\op}\geq 11\sigma_1\sigma_2\right] &\leq \Pp\left[\norm[1]{\bCt_{\bX\bY}-\bC_{\bX\bY}}_{\op} + \norm[1]{\bC_{\bX\bY}}_{\op}\geq 11\sigma_1\sigma_2\right]\\
            &\leq \Pp\left[\norm[1]{\bCt_{\bX\bY}-\bC_{\bX\bY}}_{\op}\geq 10\sigma_1\sigma_2\right]\\
            &\leq \min\left\{1, \exp\big((d_1+d_2)\ln(9)-m\big)\right\}\\
            &\leq \min\left\{1, \exp\big(3(d_1+d_2)-m\big)\right\}.
        \end{split}
    \end{equation}
\end{proof}

We have the following proposition, directly from Lemma \ref{lem:norm_upperbound_selfcovar}.

\begin{prop}\label{prop:norm_upperbound_selfcovar}
    Assume that $\bX\in\R^{d_1}$ is a zero mean, $\sigma_1^2$--sub--Gaussian random vector, and we have $m$ i.i.d. samples from $\bX$ as $\{\bX^{(i)}\}_{i=1}^{m}$. Also assume that $\bY\in\R^{d_2}$ is a zero mean, $\sigma_2^2$--sub--Gaussian random vector, and we have $m$ i.i.d. samples from $\bY$ as $\{\bY^{(i)}\}_{i=1}^{m}$. Consider the cross--covariance matrix $\bC_{\bX\bY}\in\R^{d_1\times d_2}$ as $\bC_{\bX\bY}=\E[\bX\bY^\top]$ and assume that we use the estimator $\bCt_{\bX\bY}=\frac{1}{m}\sum\limits_{i=1}^{m}\bX^{(i)}\bY^{(i)\top}$. Then we have: 
    $$
    \E\Big[\norm[1]{\bCt_{\bX\bY}-\bC_{\bX\bY}}_{\op}\Big] \leq 32\sigma_1\sigma_2 \max\left\{ \sqrt{\frac{d_1+d_2}{m}},
    \frac{d_1+d_2}{m}
    \right\}.
    $$
\end{prop}

\begin{proof}
Lemma \ref{lem:norm_upperbound_covering} with $\epsilon=\frac{1}{4}$ implies that:
\begin{equation}
\begin{aligned}
 \E\left[\norm[1]{\bCt_{\bX\bY}-\bC_{\bX\bY}}_{\op}\right] &\leq 2 \E\left[\max\limits_{\bu\in\cN_{1/4}^{(d_1)},\bv\in\cN_{1/4}^{(d_2)}}\bu^{\top} (\bCt_{\bX\bY}-\bC_{\bX\bY})\bv\right]\\
 &= 2 \E\left[\max\limits_{\bu\in\cN_{1/4}^{(d_1)},\bv\in\cN_{1/4}^{(d_2)}}
 \frac{1}{m} \sum\limits_{i=1}^{m}\left\{(\bu^{\top}\bX^{(i)})(\bv^{\top}\bY^{(i)})-\E\Big[(\bu^{\top}\bX^{(i)})(\bv^{\top}\bY^{(i)})\Big]\right\}
 \right].
\end{aligned}
\end{equation}
Let: 
\[
Z_{\bu,\bv}=\frac{1}{m} \sum\limits_{i=1}^{m}\left\{(\bu^{\top}\bX^{(i)})(\bv^{\top}\bY^{(i)})-\E\Big[(\bu^{\top}\bX^{(i)})(\bv^{\top}\bY^{(i)})\Big]\right\}.
\]
Using similar reasoning to the one used in establishing \eqref{eq:norm_Ch_upperbound3}, we conclude that $Z_{\bu,\bv}$ is a $(\frac{5\sigma_1\sigma_2}{\sqrt{m}},\frac{2.5\sigma_1\sigma_2}{m})$--sub--Gamma random variable. Now, we invoke Lemma \ref{le:max-Gamma}  to obtain:
\begin{equation}
\begin{aligned}
 \E\left[\norm[1]{\bCt_{\bX\bY}-\bC_{\bX\bY}}_{\op}\right] &\leq 2 \E\left[\max\limits_{\bu\in\cN_{1/4}^{(d_1)},\bv\in\cN_{1/4}^{(d_2)}}
 Z_{\bu,\bv}
 \right] 
 \\
 &\leq
 10\sigma_1\sigma_2\sqrt{\frac{2\ln\big( |\cN_{1/4}^{(d_1)}|.|\cN_{1/4}^{(d_2)}|\big)}{m}}+5\sigma_1\sigma_2 \frac{\ln\big( |\cN_{1/4}^{(d_1)}|.|\cN_{1/4}^{(d_2)}|\big)}{m}
 \\
 &\leq
 10\sigma_1\sigma_2\sqrt{\frac{2(d_1+d_2)\ln(9)}{m}}+5\sigma_1\sigma_2 \frac{(d_1+d_2)\ln(9)}{m}
 \\
 &\leq 32\sigma_1\sigma_2\max\left\{\sqrt{\frac{d_1+d_2}{m}},\frac{d_1+d_2}{m}\right\}
\end{aligned}
\end{equation}
\end{proof}

\begin{lem}\label{lem:norm_upperbound_crosscovar}
    Let $\bA$ be a $d\times n$ random matrix whose columns $\bA_i$ are independent, mean zero, $\sigma$--sub--Gaussian random vectors, then we have:
    \[\Pp\left[\norm{\bA}_{\op}\geq 6\sigma\sqrt{d+n}\right]\leq\exp\big(-2(d+n)\big).\]
    Moreover, for $q\in\{1,2\}$ the following inequality holds:
    \begin{equation}\label{eqn:expected-powe2-op}\E\left[\norm{\bA}_{\op}^q\right]\leq C_q\sigma^q(d+n)^{q/2},\end{equation}
    for some universal constant $C_q$ depending only on $q$.
\end{lem}

\begin{proof}
    We  use Lemma \ref{lem:norm_upperbound_covering} with $\epsilon=\frac{1}{4}$ and write:
    \begin{equation}\label{eq:norm_upperbound1}
        \begin{split}
            \Pp\left[\norm{\bA}_{\op}\geq t\right] &\leq \Pp\bigg[\max\limits_{\bu\in\cN_{1/4}^{(d)},\bv\in\cN_{1/4}^{(n)}}\bu^{\top} \bA \bv\geq \frac{t}{2}\bigg]\\
            &\leq \sum\limits_{i=1}^{|\cN_{1/4}^{(d)}|}\sum\limits_{j=1}^{|\cN_{1/4}^{(n)}|} \Pp\bigg[\bu^{(i)\top} \bA \bv^{(j)}\geq \frac{t}{2}\bigg],
        \end{split}
    \end{equation}
    where we denote the $1/4$--covering points of $\cS^{d-1}$ by $\{\bu^{(i)}\}_{i=1}^{|\cN_{1/4}^{(d)}|}$ and $1/4$--covering points of $\cS^{n-1}$ by $\{\bv^{(j)}\}_{j=1}^{|\cN_{1/4}^{(n)}|}$. 
    
    We  rewrite $\bu^{(i)\top} \bA \bv^{(j)}$ as:
    \begin{equation}
        \begin{split}
            \bu^{(i)\top} \bA \bv^{(j)} &= \sum\limits_{k=1}^{n} v^{(j)}_k \bu^{(i)\top} \bA_k,
        \end{split}
    \end{equation}
    where $v^{(j)}_k$ is the $k$--th element of $\bv^{(j)}$. Therefore, from Definition \ref{def:sub_gaussian_random_vector} and Lemma \ref{lem:sum_independent_sub_Gaussian&Gamma}, $\bu^{(i)\top} \bA \bv^{(j)}$ is a sub--Gaussian random variable with parameter $\sigma$. Therefore we  write:
    \begin{equation}
        \Pp\bigg[\bu^{(i)\top} \bA \bv^{(j)}\geq \frac{t}{2}\bigg] \leq \exp\Big(\frac{-t^2}{8\sigma^2}\Big).
    \end{equation}
    We know from \citep[~Lemma 5.7]{wainwright2019high} that:
    \begin{equation}\label{eq:norm_upperbound2}
        \cN_{1/4}^{(d)}\leq 9^d,\qquad \cN_{1/4}^{(n)}\leq 9^n.
    \end{equation}
    Then from \eqref{eq:norm_upperbound1}, we have:
    \begin{equation}\label{eqn:op-norm-pr-upper}
        \begin{split}
            \Pp\left[\norm{\bA}_{\op}\geq t\right] &\leq \sum\limits_{i=1}^{|\cN_{1/4}^{(d)}|}\sum\limits_{j=1}^{|\cN_{1/4}^{(n)}|} \Pp\bigg[\bu^{(i)\top} \bA \bv^{(j)}\geq \frac{t}{2}\bigg]\\
            &\leq  9^{n+d} \exp\Big(\frac{-t^2}{8\sigma^2}\Big).
        \end{split}
    \end{equation}
   Setting $t=6\sigma\sqrt{d+n}$ implies:
    \begin{equation}
        \begin{split}
            \Pp\left[\norm{\bA}_{\op}\geq 6\sigma\sqrt{d+n}\right] &\leq  9^{n+d} \exp\Big(\frac{-9(d+n)}{2}\Big)\\
            &\leq \exp\big(-2(n+d)\big).
        \end{split}
    \end{equation}
    The inequality \eqref{eqn:expected-powe2-op} is a straightforward consequence of \eqref{eqn:op-norm-pr-upper} and the integral representation of expectation. First we consider the case $q=1$:
    \begin{equation}
    \begin{aligned}
        \E\left[\norm{\bA}_{\op}\right] &= \int_{0}^{+\infty} \Pp\left[\norm{\bA}_{\op}\geq t\right] dt\\
        &\leq \int_{0}^{+\infty} \min\left\{1,9^{n+d} \exp\Big(\frac{-t^2}{8\sigma^2}\Big)\right\} dt\\
        &= \int_{0}^{2\sigma\sqrt{2\ln(9)(n+d)}} dt + \int_{2\sigma\sqrt{2\ln(9)(n+d)}}^{+\infty} 9^{n+d} \exp\Big(\frac{-t^2}{8\sigma^2}\Big) dt\\
        &\leq 2\sqrt{2\ln(9)} \sigma (n+d)^{1/2} + 1\\
        &\leq 9 \sigma (n+d)^{1/2}.
    \end{aligned}
    \end{equation}
    Now we consider the case $q=2$:
    \begin{equation}
    \begin{aligned}
        \E\left[\norm{\bA}_{\op}^2\right] &= \int_{0}^{+\infty} \Pp\left[\norm{\bA}_{\op}^2\geq u\right] du\\
        &= 2\int_{0}^{+\infty} t\Pp\left[\norm{\bA}_{\op}\geq t\right] dt\\
        &\leq 2\int_{0}^{+\infty} t\min\left\{1,9^{n+d} \exp\Big(\frac{-t^2}{8\sigma^2}\Big)\right\} dt\\
        &= 2\int_{0}^{2\sigma\sqrt{2\ln(9)(n+d)}} tdt + 2\int_{2\sigma\sqrt{2\ln(9)(n+d)}}^{+\infty} 9^{n+d} t \exp\Big(\frac{-t^2}{8\sigma^2}\Big) dt\\
        &= 8\ln(9)\sigma^2 (n+d) + 8\sigma^2 9^{n+d} \int_{\ln(9)(n+d)}^{+\infty} e^{-t'}dt'\\
        &= 8\ln(9)\sigma^2 (n+d) + 8\sigma^2\\
        &\leq 36\sigma^2 (n+d).
    \end{aligned}
    \end{equation}
\end{proof}
\section{Detailed Proof of Theorems \ref{thm:achievable_scheme} and \ref{thm:achievable_scheme-Fr}}
\label{app:proof_thm_achievable_scheme}
\subsection{Proof of Theorem \ref{thm:achievable_scheme}}
\begin{proof}
    We prove the existence of a scheme having distortion error less than $\varepsilon$ (under operator norm), with the following choices for the sample number and communication budgets:
    \begin{align}    
        m &\geq 2^{19}  \frac{d}{\tilde{\varepsilon}^2},\label{eqn:ach-con-1}\\
        B_k& \geq \frac{2^{18}\beta d_k d }{\tilde{\varepsilon}^2},\label{eqn:ach-con-3}\\
        n&=\frac{\frac{B_1}{d_1}\wedge\frac{B_2}{d_2}}{\beta}\bigwedge m,\label{eqn:ach-con-2}
    \end{align}
    where for brevity we set $\tilde{\varepsilon}=\frac{\varepsilon}{\sigma^2}\le 1$ and $\beta$ will be determined in the sequel. Also observe that \eqref{eqn:ach-con-1}--\eqref{eqn:ach-con-3} imply:
    \begin{equation}
        n\ge \frac{2^{18}d}{\tilde{\varepsilon}^2}.\label{eqn:ach-con-4}
    \end{equation}
    We define events $\Eb_{1,1}$, $\Eb_{2,1}$, $\Eb_{1,2}$, and $\Eb_{2,2}$ as "Receiving error from Agent 1, when $\norm[1]{\bCt_{\bX_1\bX_1}}_{\op}>11\sigma^2$.", "Receiving error from Agent 2, when $\norm[1]{\bCt_{\bX_2\bX_2}}_{\op}>11\sigma^2$.", "Receiving error from Agent 1, when $\norm{\bmsfX_1}_{\op}\geq 6\sigma\sqrt{d_1+n}$.", and "Receiving error from Agent 2, when $\norm{\bmsfX_2}_{\op}\geq 6\sigma\sqrt{d_2+n}$.", respectively.
    We also define event $\Eb$ as $\Eb=\Eb_{1,1}\vee\Eb_{2,1}\vee \Eb_{1,2}\vee \Eb_{2,2}$. We write:
    \begin{equation}\label{eq:achievable1}
    \begin{split}
        \E\left[\norm[1]{\bCh-\bC}_{\op}\right] &= \E\left[\norm[1]{\bCh-\bC}_{\op}\mid\Eb\right]\Pp\left[\Eb\right]+ \E\left[\norm[1]{\bCh-\bC}_{\op}\mid\Eb^{\msfc}\right]\Pp\left[\Eb^{\msfc}\right].
    \end{split}
    \end{equation}
    We find an upper bound for every term of  \eqref{eq:achievable1}. First, notice that when an error is received in central server, the central server returns $\bCh=\bzero$, therefore:
    \begin{equation}
    \begin{split}
        \E\left[\norm[1]{\bCh-\bC}_{\op}\mid\Eb\right] &= \E\left[\norm[1]{\bC}_{\op}\mid\Eb\right]\\
        &\leq \sigma^2.
    \end{split}
    \end{equation}
    From Lemma \ref{lem:norm_upperbound_selfcovar}, we have:
    \begin{equation}\notag
    \begin{split}
        &\Pp\left[\norm[1]{\bCt_{\bX_1\bX_1}}_{\op}\geq 11\sigma^2\right] \leq \min\left\{1,\exp\big(6d_1-m\big)\right\},\\
        &\Pp\left[\norm[1]{\bCt_{\bX_2\bX_2}}_{\op}\geq 11\sigma^2\right] \leq \min\left\{1,\exp\big(6d_2-m\big)\right\}.
    \end{split}
    \end{equation}
    Also Lemma \ref{lem:norm_upperbound_crosscovar} yields:
    \begin{equation}\notag
    \begin{split}
        &\Pp\left[\norm{\bmsfX_1}_{\op}\geq 6\sigma\sqrt{d_1+n}\right]\leq \exp\left(-2(d_1+n)\right),\\
        &\Pp\left[\norm{\bmsfX_2}_{\op}\geq 6\sigma\sqrt{d_2+n}\right]\leq \exp\left(-2(d_2+n)\right)
    \end{split}
    \end{equation}
    Therefore we can upper--bound $\Pp\left[\Eb\right]$:
    \begin{equation}\label{eq:error-event}
    \begin{split}
        \Pp\left[\Eb\right] &\leq \Pp[\Eb_{1,1}] + \Pp[\Eb_{2,1}] + \Pp[\Eb_{1,2}] + \Pp[\Eb_{2,2}]\\
        &= \Pp\left[\norm[1]{\bCt_{\bX_1\bX_1}}_{\op}\geq 11\sigma^2\right] + \Pp\left[\norm[1]{\bCt_{\bX_2\bX_2}}_{\op}\geq 11\sigma^2\right]\\
        &\quad+\Pp\left[\norm[1]{\bmsfX_1}_{\op}\geq 6\sigma\sqrt{d_1+n}\right]\\
        &\quad+\Pp\left[\norm[1]{\bmsfX_2}_{\op}\geq 6\sigma\sqrt{d_2+n}\right]\\
        &\leq \exp\big(6d_1-m\big) + \exp\big(6d_2-m\big)\\
        &\quad+ \exp\big(-2(d_1+n)\big) + \exp\big(-2(d_2+n)\big)\\
        &\leq 2\exp\big(6d-m\big) + 2\exp(-2(n+1))\\
        &<\frac{\tilde{\varepsilon}}{10000} 
    \end{split},
    \end{equation}
    where the last equation follows from the inequalities $\exp(6d-m)\le \exp(d(6-2^{19}\tilde{\varepsilon}^{-2}))\le \exp(-2^{18})\tilde{\varepsilon}$ and $\exp(-2n)\le \exp(-2^{19}d\tilde{\varepsilon}^{-2})\le  \exp(-2^{19})\tilde{\varepsilon}$. 
    Since $\tilde{\varepsilon}\le 1$, we have:
    \begin{equation}
        \Pp\left[\Eb^{\msfc}\right] \geq 0.9999.
    \end{equation}
    Now we find an upper bound for $\E\left[\norm[1]{\bCh-\bC}_{\op}\mid\Eb^{\msfc}\right]$:
    \begin{equation}\label{eq:Weyl-application}
    \begin{split}
        \E\left[\norm[1]{\bCh-\bC}_{\op}\mid\Eb^{\msfc}\right] &= \E\left[\norm[1]{\bCh^{\ast}_{+}-\bC}_{\op}\mid\Eb^{\msfc}\right]\\
        &\leq\E\left[\norm[1]{\bCh^{\ast}_{+} - \bCh^{\ast}}_{\op}\mid\Eb^{\msfc}\right] + \E\left[\norm[1]{\bCh^{\ast} - \bC}_{\op}\mid\Eb^{\msfc}\right]\\
        &=\E\left[\big|\lambda_{\min}(\bCh^{\ast})\big|\mathbbm{1}_{\{\lambda_{\min}(\bCh^{\ast})< 0\}}\mid\Eb^{\msfc}\right] + \E\left[\norm[1]{\bCh^{\ast} - \bC}_{\op}\mid\Eb^{\msfc}\right]\\
        &\leq\E\left[\big|\lambda_{\min}(\bCh^{\ast})-\lambda_{\min}(\bC)\big|\mathbbm{1}_{\{\lambda_{\min}(\bCh^{\ast})< 0\}}\mid\Eb^{\msfc}\right]\\
        &\quad+\E\left[\norm[1]{\bCh^{\ast} - \bC}_{\op}\mid\Eb^{\msfc}\right]\\
        &\overset{\text{(a)}}{\leq} 2\E\left[\norm[1]{\bCh^{\ast} - \bC}_{\op}\mid\Eb^{\msfc}\right],
    \end{split}
    \end{equation}
    where (a) is a consequence of Weyl's inequality \citep[~Section 4.3]{johnson1985matrix}. Also we have:
    \begin{equation}\label{eq:norm_traingle_inequality}
    \begin{split}
        \E\left[\norm[1]{\bCh^{\ast}-\bC}_{\op}\mid\Eb^{\msfc}\right] &= \E\left[\norm{\begin{bmatrix}
                        \bCh_{\bX_1\bX_1} & \frac{1}{n}\widehat{\bmsfX}_1\widehat{\bmsfX}_2^\top \\
                        \frac{1}{n}\widehat{\bmsfX}_2\widehat{\bmsfX}_1^\top & \bCh_{\bX_2\bX_2}
                    \end{bmatrix}
                    -\begin{bmatrix}
                        \bC_{\bX_1\bX_1} & \bC_{\bX_1\bX_2} \\
                        \bC_{\bX_1\bX_2}^\top & \bC_{\bX_2\bX_2}
                    \end{bmatrix}}_{\op}\mid\Eb^{\msfc}\right]\\
        &=\E\left[\norm{\begin{bmatrix}
                        \bCh_{\bX_1\bX_1} - \bC_{\bX_1\bX_1}& \frac{1}{n}\widehat{\bmsfX}_1\widehat{\bmsfX}_2^\top - \bC_{\bX_1\bX_2}\\
                        \frac{1}{n}\widehat{\bmsfX}_2\widehat{\bmsfX}_1^\top - \bC_{\bX_1\bX_2}^\top & \bCh_{\bX_2\bX_2} - \bC_{\bX_2\bX_2}
                    \end{bmatrix}
                    }_{\op}\mid\Eb^{\msfc}\right]\\
        &\leq \E\left[\norm[1]{\bCh_{\bX_1\bX_1}-\bC_{\bX_1\bX_1}}_{\op}\mid\Eb^{\msfc}\right] + \E\left[\norm[1]{\bCh_{\bX_2\bX_2} - \bC_{\bX_2\bX_2}}_{\op}\mid\Eb^{\msfc}\right]\\& \qquad+  \E\left[\norm[2]{\frac{1}{n}\widehat{\bmsfX}_1\widehat{\bmsfX}_2^\top - \bC_{\bX_1\bX_2}}_{\op}\mid\Eb^{\msfc}\right].
    \end{split}
    \end{equation}
    We use matrix quantization scheme defined in Appendix \ref{app:pack_cover_matrix} to quantize matrices $\bCt_{\bX_1\bX_1}$, $\bCt_{\bX_2\bX_2}$, $\bmsfX_1$, and $\bmsfX_2$. Therefore, we can use the relation between communication load and the resolution of this quantization, which is stated in Appendix \ref{app:pack_cover_matrix}.
    \begin{itemize}
        \item Quantization of $\bCt_{\bX_1\bX_1}\in\R^{d_1\times d_1}$:
        $r=11\sigma^2$, therefore:
        \begin{equation}\label{eq:quantization_upperbound1}
            d_1^2\log_2\left(\frac{33\sigma^2}{\epsilon_1'}\right)= B_1' = \frac{B_1}{2}\Rightarrow \epsilon_1' =  33\sigma^2\cdot 2^{\frac{-B_1}{2d_1^2}}.
        \end{equation}
        \item Quantization of $\bCt_{\bX_2\bX_2}\in\R^{d_2\times d_2}$:
        $r=11\sigma^2$, therefore:
        \begin{equation}\label{eq:quantization_upperbound2}
            d_2^2\log_2\left(\frac{33\sigma^2}{\epsilon_2'}\right)= B_2' = \frac{B_2}{2}\Rightarrow \epsilon_2' =  33\sigma^2 \cdot 2^{\frac{-B_2}{2d_2^2}}.
        \end{equation}
        \item Quantization of $\bmsfX_1\in\R^{d_1\times n}$:
        $r=6\sigma\sqrt{d_1+n}$, therefore:
        \begin{equation}\label{eq:quantization_upperbound3}
            n d_1\log_2\left(\frac{18\sigma\sqrt{d_1+n}}{\epsilon_1''}\right)= B_1'' = \frac{B_1}{2}\Rightarrow \epsilon_1'' =  18\sigma\sqrt{d_1+n} \cdot 2^{\frac{-B_1}{2n d_1}}.
        \end{equation}
        \item Quantization of $\bmsfX_2\in\R^{d_2\times n}$:
        $r=6\sigma\sqrt{d_2+n}$, therefore:
        \begin{equation}\label{eq:quantization_upperbound4}
            n d_2\log_2\left(\frac{18\sigma\sqrt{d_2+n}}{\epsilon_2''}\right)= B_2'' = \frac{B_2}{2}\Rightarrow \epsilon_2'' =  18\sigma\sqrt{d_2+n} \cdot 2^{\frac{-B_2}{2n d_2}}.
        \end{equation}
    \end{itemize}
    From Proposition \ref{prop:norm_upperbound_selfcovar} and the choice of $m$ in \eqref{eqn:ach-con-1}, we have (for $k=1,2$):
    \begin{equation}\label{eq:norm_main_upperbound1}
        \E\left[\norm[1]{\bCt_{\bX_k\bX_k}-\bC_{\bX_k\bX_k}}_{\op}\mid\Eb^{\msfc}\right]\leq\frac{\E\left[\norm[1]{\bCt_{\bX_k\bX_k}-\bC_{\bX_k\bX_k}}_{\op}\right]}{\mathbb{P}[\Eb^c]}\leq 32\sigma^2 \sqrt{\frac{2d_k}{m}} 
    \end{equation}
    We also have:
    \begin{equation}\label{eq:norm_main_upperbound2}
    \begin{split}
        \E\left[\norm{\frac{1}{n}\bmsfX_1\bmsfX_2^{\top}-\bC_{\bX_1\bX_2}}_{\op}\mid \Eb^c\right]&\leq \frac{\E\left[\norm{\frac{1}{n}\bmsfX_1\bmsfX_2^{\top}-\bC_{\bX_1\bX_2}}_{\op}\right]}{\mathbb{P}[\Eb^c]} \\&\leq 32\sigma^2 \max\left\{\sqrt{\frac{d_1+d_2}{n}},\frac{d_1+d_2}{n}\right\}\\
        &= 32\sigma^2 \sqrt{\frac{d}{n}}.
    \end{split}
    \end{equation}
    From \eqref{eq:norm_main_upperbound1} and \eqref{eq:quantization_upperbound1} we write:
    \begin{equation}\label{eq:norm_Cxx_upper_bound}
    \begin{split}
        \E\left[\norm[1]{\bCh_{\bX_1\bX_1}-\bC_{\bX_1\bX_1}}_{\op}\mid\Eb^{\msfc}\right] 
        &\leq \E\left[\norm[1]{\bCh_{\bX_1\bX_1}-\bCt_{\bX_1\bX_1}}_{\op}\mid\Eb^{\msfc}\right]\\&\qquad +\E\left[\norm[1]{\bCt_{\bX_1\bX_1}-\bC_{\bX_1\bX_1}}_{\op}\mid\Eb^{\msfc}\right]\\
        &\leq \epsilon_1' + 32\sigma^2 \sqrt{\frac{2d_1}{m}} \\
        &\stackrel{\text{(a)}}{<} 33\sigma^2 \cdot 2^{\frac{-B_1}{2d_1^2}} + 32\sigma^2 \sqrt{\frac{2 d}{m}} \\
        &\stackrel{\text{(b)}}{\le} \sigma^2\left(33 \cdot 2^{\frac{-2^{17}\beta d}{d_1\tilde{\varepsilon}^2}} +\frac{\tilde{\varepsilon}}{16}\right)\\
        &\stackrel{\text{(c)}}{\le} \sigma^2\left( 33 \cdot 2^{-2^{17}\beta } +\frac{\tilde{\varepsilon}}{16}\right) 
    \end{split}
    \end{equation}
    where (a) and (b) follow from \eqref{eqn:ach-con-1} and \eqref{eqn:ach-con-3} and (c) follows because $d_1<d$ and $\tilde{\varepsilon}\le 1$.
    
    Similarly we have:
    \begin{equation}\label{eq:norm_Cyy_upper_bound}
    \begin{split}
        \E\left[\norm[1]{\bCh_{\bX_2\bX_2}-\bC_{\bX_2\bX_2}}_{\op}\mid\Eb^{\msfc}\right] 
        &\le\sigma^2 \left(33\cdot  2^{-2^{17}\beta } +\frac{\tilde{\varepsilon}}{16}\right).
    \end{split}
    \end{equation}
    
    Next, we consider the estimation error of the \emph{cross--covariance} matrix.    From \eqref{eq:norm_main_upperbound2}, \eqref{eq:quantization_upperbound3}, and \eqref{eq:quantization_upperbound4} we write:
    \begin{equation}\label{eq:norm_Cxy_upper_bound}
    \begin{split}
        \E\left[\norm[2]{\frac{1}{n}\widehat{\bmsfX}_1\widehat{\bmsfX}_2^\top - \bC_{\bX_1\bX_2}}_{\op}\mid\Eb^{\msfc}\right] 
        &\leq \E\left[\norm[2]{\frac{1}{n}\widehat{\bmsfX}_1(\widehat{\bmsfX}_2-\bmsfX_2)^\top}_{\op}\mid\Eb^{\msfc}\right] + \E\left[\norm[2]{\frac{1}{n}(\widehat{\bmsfX}_1-\bmsfX_1)\bmsfX_2^\top}_{\op}\mid\Eb^{\msfc}\right]\\
        &\quad+ \E\left[\norm[2]{\frac{1}{n}\bmsfX_1\bmsfX_2^\top - \bC_{\bX_1\bX_2}}_{\op}\mid\Eb^{\msfc}\right]\\
        &\leq \frac{1}{n} \E\left[\norm[1]{\widehat{\bmsfX}_1}_{\op}\norm[1]{\widehat{\bmsfX}_2-\bmsfX_2}_{\op}\mid\Eb^{\msfc}\right] + \frac{1}{n} \E\left[\norm[1]{\bmsfX_2}_{\op}\norm[1]{\widehat{\bmsfX}_1-\bmsfX_1}_{\op}\mid\Eb^{\msfc}\right]\\
        &\quad+ \E\left[\norm[2]{\frac{1}{n}\bmsfX_1\bmsfX_2^\top - \bC_{\bX_1\bX_2}}_{\op}\mid\Eb^{\msfc}\right]\\
        &\leq \frac{6\sigma\sqrt{d_1+n}}{n}\epsilon_2'' + \frac{6\sigma\sqrt{d_2+n}}{n}\epsilon_1'' + 32\sigma^2 \sqrt{\frac{d}{n}} \\
        &= \frac{108\sigma^2\sqrt{(d_1+n)(d_2+n)}}{n}\Big(2^{\frac{-B_1}{2n d_1}} + 2^{\frac{-B_2}{2n d_2}}\Big) + 32\sigma^2 \sqrt{\frac{d}{n}} \\
        &\leq \frac{108\sigma^2\sqrt{(d_1+n)(d_2+n)}}{n}\Big(2^{\frac{-1}{2n}\min\{\frac{B_1}{d_1},\frac{B_2}{d_2}\}}\Big) + 32\sigma^2 \sqrt{\frac{d}{n}}.
    \end{split}
    \end{equation}
    The choice of $n$ in \eqref{eqn:ach-con-2} implies:    
    \begin{equation}
        2^{\frac{-1}{2n}\min\{\frac{B_1}{d_1},\frac{B_2}{d_2}\}} \leq 2^{-\frac{\beta}{2}}.
    \end{equation}
    Also \eqref{eqn:ach-con-4}  implies:
    \begin{equation}
        \frac{d_k+n}{n} < 1+\frac{d}{n}
        < 1+\tilde{\varepsilon}^2\le 2.
    \end{equation}
    Thus:
    \begin{equation}\label{eqn:d1d2n}
    \begin{split}
        \frac{\sqrt{(d_1+n)(d_2+n)}}{n} &= \sqrt{\frac{d_1+n}{n}.\frac{d_2+n}{n}}<2.
    \end{split}
    \end{equation}
    Next consider:
    \begin{equation}
    \begin{split}
        \sqrt{\frac{d}{n}} &= \sqrt{\frac{d\beta}{\frac{B_1}{d_1}\wedge\frac{B_2}{d_2}}\bigvee{\frac{d}{m}}}\\
        &=\sqrt{\frac{\beta dd_1}{B_1}\vee\frac{\beta dd_2}{B_2}\vee\frac{d}{m}}\\
        &<\frac{ \tilde{\varepsilon}}{512}.
    \end{split}
    \end{equation}
    In summary, we have:
    \begin{equation}
        \E\left[\norm[2]{\frac{1}{n}\widehat{\bmsfX}_1\widehat{\bmsfX}_2^\top - \bC_{\bX_1\bX_2}}_{\op}\mid \Eb^c\right] 
            \leq \sigma^2\left(432\cdot 2^{-\frac{\beta}{2}}+\frac{ \tilde{\varepsilon}}{16}\right).
    \end{equation}
    Choosing $\beta=2\log_2\frac{6912}{\tilde{\varepsilon}}$ yields:
    \begin{equation}
        \E\left[\norm[2]{\frac{1}{n}\widehat{\bmsfX}_1\widehat{\bmsfX}_2^\top - \bC_{\bX_1\bX_2}}_{\op}\mid \Eb^c\right] 
            \leq \frac{\sigma^2 \tilde{\varepsilon}}{8}=\frac{ {\varepsilon}}{8}.\label{eq:norm-CXY-final}
    \end{equation}
    Also substituting the value of $\beta$ in \eqref{eq:norm_Cxx_upper_bound} and \eqref{eq:norm_Cyy_upper_bound} implies:
    \begin{equation}\label{eq:norm-CKK}
        \E\left[\norm[1]{\bCh_{\bX_k\bX_k}-\bC_{\bX_k\bX_k}}_{\op}\mid \Eb^c\right] 
        \le\frac{ {\varepsilon}}{8},\qquad\qquad k=1,2.
    \end{equation}  
    Putting \eqref{eq:norm_traingle_inequality}, \eqref{eq:norm-CXY-final} and \eqref{eq:norm-CKK} together, gives:
    \begin{equation}
        \E\left[\norm[1]{\bCh^{\ast}-\bC}_{\op}\right]\le \frac{3\varepsilon}{8}.    
    \end{equation}
    
    Finally from \eqref{eq:Weyl-application} we can write:
    \begin{equation}
        \begin{split}
            \E\left[\norm[1]{\bCh-\bC}_{\op}\right] &= \E\left[\norm[1]{\bCh-\bC}_{\op}\mid\Eb\right]\Pp\left[\Eb\right]+ \E\left[\norm[1]{\bCh-\bC}_{\op}\mid\Eb^{\msfc}\right]\Pp\left[\Eb^{\msfc}\right]\\
            &\leq \sigma^2\Pp\left[\Eb\right] + 2\E\left[\norm[1]{\bCh^{\ast}-\bC}_{\op}\mid \Eb^c\right]\\
            &< \frac{{\varepsilon}}{10}+\frac{{3\varepsilon}}{4}<\varepsilon.
        \end{split}
    \end{equation}

    The proof of Theorem \ref{thm:achievable_scheme}     is completed.
\end{proof}
\subsection{Proof of Theorem \ref{thm:achievable_scheme-Fr}}
\begin{proof}
The proof of Theorem \ref{thm:achievable_scheme-Fr} is approximately the same as the proof of Theorem  \ref{thm:achievable_scheme} with some minor modifications. 

We prove the existence of a scheme having distortion error less than $\varepsilon$ (under Frobenius norm), with the following choices for the sample number and communication budgets:
\begin{align}    
    m &\geq 2^{19}  \frac{d}{\tilde{\varepsilon}^2},\label{eqn:ach-con-fr-1}\\
    B_k& \geq \frac{2^{18}\beta d_k d_{\min} }{\tilde{\varepsilon}^2}\bigvee 2d_k^2\log_2\left(\frac{528}{\tep}\right),\label{eqn:ach-con-fr-3}\\
    n&=\frac{\frac{B_1}{d_1}\wedge\frac{B_2}{d_2}}{\beta}\bigwedge m,\label{eqn:ach-con-fr-2}
\end{align}
where for brevity we set $\tilde{\varepsilon}=\frac{\varepsilon}{\sigma^2\sqrt{d}}\le 1$,
and $\beta$ will be determined in the sequel. Also observe that \eqref{eqn:ach-con-fr-1}--\eqref{eqn:ach-con-fr-3} imply:
\begin{equation}\label{eqn:ach-con-fr-4}
    n\ge \frac{2^{18}d_{\min}}{\tilde{\varepsilon}^2}.
\end{equation}
We use the same error events $(\Eb_{i,j}:i=1,2,j=1,2)$ as in the proof of Theorem \ref{thm:achievable_scheme}. Then similar calculations to \eqref{eq:error-event} shows the inequality $\mathbb{P}[\Eb]<\frac{\tep}{10000}$ still holds for new assignment of $m,n,B_k$. We consider two regimes for the distortion error $\varepsilon$ and specify the covariance matrix estimator $\bCh$ for each regime attaining the distortion error $\varepsilon$, separately.

{\bf Case I: Reasonable distortion error}. In this regime, the distortion error satisfies $\varepsilon<512\sigma^2\sqrt{d_{\min}}$. Here, we take the estimator $\bCh$ as in \eqref{eq:ach-estimator}. Further we have:
\begin{equation}\label{eq:achievable1-fr}
\begin{split}
    \E\left[\norm[1]{\bCh-\bC}_{\Fr}\right] &= \E\left[\norm[1]{\bCh-\bC}_{\Fr}\mid\Eb\right]\Pp\left[\Eb\right]+ \E\left[\norm[1]{\bCh-\bC}_{\Fr}\mid\Eb^{\msfc}\right]\Pp\left[\Eb^{\msfc}\right]\\
    &\leq \sigma^2\sqrt{d}\frac{\tep}{10000}+ \E\left[\norm[1]{\bCh-\bC}_{\Fr}\mid\Eb^{\msfc}\right]\Pp\left[\Eb^{\msfc}\right]\\
    &= \frac{\varepsilon}{10000}+ \E\left[\norm[1]{\bCh-\bC}_{\Fr}\mid\Eb^{\msfc}\right]\Pp\left[\Eb^{\msfc}\right].
\end{split}
\end{equation}
Next consider:
\begin{equation}\label{eq:convex-projection}
\begin{split}
    {\Pp\left[\Eb^{\msfc}\right]}\E\left[\norm[1]{\bCh-\bC}_{\Fr}\mid\Eb^{\msfc}\right] &\leq \E\left[\norm[1]{\bCh^{\ast}_{+}-\bC}_{\Fr}\mid\Eb^{\msfc}\right]\\
    &\overset{\text{(a)}}{\leq} \E\left[\norm[1]{\bCh^{\ast} - \bC}_{\Fr}\mid\Eb^{\msfc}\right],
\end{split}
\end{equation}
where (a) follows from the following inequalities:
\begin{align}
    \norm[1]{\bCh^{\ast} - \bC}_{\Fr}^2&=\Tr{\left(\bCh^{\ast} - \bC\right)^2}\notag\\
    &=\Tr{\left(\bCh^{\ast} - \bCh^{\ast}_+\right)^2}+\Tr{\left(\bCh^{\ast}_+ - \bC\right)^2}+2\Tr{\left(\bCh^{\ast} - \bCh^{\ast}_+\right)\left(\bCh^{\ast}_+ - \bC\right)}\notag\\
    &\geq \norm[1]{\bCh^{\ast}_+ - \bC}_{\Fr}^2+2\Tr{\left(\bCh^{\ast} - \bCh^{\ast}_+\right)\left(\bCh^{\ast}_+ - \bC\right)}\notag\\
    &= \norm[1]{\bCh^{\ast}_+ - \bC}_{\Fr}^2+2\Tr{\left( \bCh^{\ast}_+-  \bCh^{\ast}\right) \bC}\label{eq:matrix-decomposition-orthogonal}\\
    &\geq \norm[1]{\bCh^{\ast}_+ - \bC}_{\Fr}^2,\label{eq:trace-mul-positive}
    \end{align}
in which  \eqref{eq:matrix-decomposition-orthogonal} is true because the positive part $\bCh^{\ast}_+$ of $\bCh^{\ast}$ is orthogonal to the negative part $\bCh^{\ast}_-=\bCh^{\ast}-\bCh^{\ast}_+$ of it, and \eqref{eq:trace-mul-positive} is due to the fact that the trace of the multiplication of two positive semi--definite matrices is non--negative \cite[Exercise 12.14]{abadir2005matrix}.
    
Now, we use the following counterpart of \eqref{eq:norm_traingle_inequality}:
    
\begin{equation}\label{eq:norm_traingle_inequality-fr}
\begin{split}
    \E\left[\norm[1]{\bCh^{\ast}-\bC}_{\Fr}\mid\Eb^{\msfc}\right] &= \E\left[\norm{\begin{bmatrix}
                    \bCh_{\bX_1\bX_1} & \frac{1}{n}\widehat{\bmsfX}_1\widehat{\bmsfX}_2^\top \\
                    \frac{1}{n}\widehat{\bmsfX}_2\widehat{\bmsfX}_1^\top & \bCh_{\bX_2\bX_2}
                \end{bmatrix}
                -\begin{bmatrix}
                    \bC_{\bX_1\bX_1} & \bC_{\bX_1\bX_2} \\
                    \bC_{\bX_1\bX_2}^\top & \bC_{\bX_2\bX_2}
                \end{bmatrix}}_{\Fr}\mid\Eb^{\msfc}\right]\\
    &=\E\left[\norm{\begin{bmatrix}
                    \bCh_{\bX_1\bX_1} - \bC_{\bX_1\bX_1}& \frac{1}{n}\widehat{\bmsfX}_1\widehat{\bmsfX}_2^\top - \bC_{\bX_1\bX_2}\\
                    \frac{1}{n}\widehat{\bmsfX}_2\widehat{\bmsfX}_1^\top - \bC_{\bX_1\bX_2}^\top & \bCh_{\bX_2\bX_2} - \bC_{\bX_2\bX_2}
                \end{bmatrix}
                }_{\Fr}\mid\Eb^{\msfc}\right]\\
    &\leq \E\left[\norm[1]{\bCh_{\bX_1\bX_1}-\bC_{\bX_1\bX_1}}_{\Fr}\mid\Eb^{\msfc}\right] + \E\left[\norm[1]{\bCh_{\bX_2\bX_2} - \bC_{\bX_2\bX_2}}_{\Fr}\mid\Eb^{\msfc}\right]\\& \qquad+ \sqrt{2} \E\left[\norm[2]{\frac{1}{n}\widehat{\bmsfX}_1\widehat{\bmsfX}_2^\top - \bC_{\bX_1\bX_2}}_{\Fr}\mid\Eb^{\msfc}\right]
    \\
     &\leq \sqrt{d_1}\E\left[\norm[1]{\bCh_{\bX_1\bX_1}-\bC_{\bX_1\bX_1}}_{\op}\mid\Eb^{\msfc}\right] + \sqrt{d_2}\E\left[\norm[1]{\bCh_{\bX_2\bX_2} - \bC_{\bX_2\bX_2}}_{\op}\mid\Eb^{\msfc}\right]\\& \qquad+ \sqrt{2d_{\min}} \E\left[\norm[2]{\frac{1}{n}\widehat{\bmsfX}_1\widehat{\bmsfX}_2^\top - \bC_{\bX_1\bX_2}}_{\op}\mid\Eb^{\msfc}\right].
\end{split}
\end{equation}
Using the same inequalities \eqref{eq:quantization_upperbound1}--\eqref{eq:quantization_upperbound4} and \eqref{eq:norm_Cxx_upper_bound} with the new assignments \eqref{eqn:ach-con-fr-1}--\eqref{eqn:ach-con-fr-3} imply:
\begin{equation}\label{eq:CKK-fr}
\begin{aligned}
    \sqrt{d_k}\E\left[\norm[1]{\bCh_{\bX_k\bX_k}-\bC_{\bX_k\bX_k}}_{\op}\mid\Eb^{\msfc}\right]&\leq \sqrt{d_k} \sigma^2 \left(33 \cdot 2^{\frac{-B_k}{2d_k^2}} + 32 \sqrt{\frac{2 d}{m}}\right)\\
    &\leq \sqrt{d_k} \sigma^2 \left(33 \cdot 2^{\log_2\left(\frac{\tep}{528}\right)} + \frac{\tep}{16}\right)\\
    &\leq\frac{\sqrt{d}\sigma^2\tep}{8}<\frac{\varepsilon}{8}.
\end{aligned}
\end{equation}
Further \eqref{eq:norm_main_upperbound2} still holds in the reasonable distortion error regime $\varepsilon<512\sigma^2\sqrt{d_{\min}}$, because in this regime $d<n$, which is a consequence of \eqref{eqn:ach-con-fr-4}. Moreover the inequalities \eqref{eq:norm_Cxy_upper_bound}--\eqref{eqn:d1d2n} are still valid. The assignment \eqref{eqn:ach-con-fr-3} guaranties:
\begin{equation}
\begin{split}
    \sqrt{\frac{d}{n}} &= \sqrt{\frac{d\beta}{\frac{B_1}{d_1}\wedge\frac{B_2}{d_2}}\bigvee{\frac{d}{m}}}\\
    &=\sqrt{\frac{\beta dd_1}{B_1}\vee\frac{\beta dd_2}{B_2}\vee\frac{d}{m}}\\
    &<\frac{ \tilde{\varepsilon}}{512}\sqrt{\frac{d}{d_{\min}}}.
\end{split}
\end{equation}
In summary, we have:
\begin{equation}
    \E\left[\norm[2]{\frac{1}{n}\widehat{\bmsfX}_1\widehat{\bmsfX}_2^\top - \bC_{\bX_1\bX_2}}_{\op}\mid\Eb^{\msfc}\right] 
        \leq \sigma^2\left(432\cdot 2^{-\frac{\beta}{2}}+\frac{ \tilde{\varepsilon}}{16}\sqrt{\frac{d}{d_{\min}}}\right).
\end{equation}
Choosing $\beta=2\log_2\frac{6912\sigma^2\sqrt{d_{\min}}}{{\varepsilon}}$ yields:
\begin{equation}
   \sqrt{d_{\min}} \E\left[\norm[2]{\frac{1}{n}\widehat{\bmsfX}_1\widehat{\bmsfX}_2^\top - \bC_{\bX_1\bX_2}}_{\Fr}\mid\Eb^{\msfc}\right] 
       \leq \frac{ {\varepsilon}}{8}.\label{eq:norm-CXY-final-fr}
\end{equation}
Putting \eqref{eq:norm_traingle_inequality-fr}, \eqref{eq:CKK-fr} and \eqref{eq:norm-CXY-final-fr}  together, gives:
\begin{equation}
    \E\left[\norm[1]{\bCh^{\ast}-\bC}_{\Fr}\mid\Eb^{\msfc}\right]\leq \frac{\varepsilon}{2}.    
\end{equation}
Finally from \eqref{eq:achievable1-fr} we write:
\begin{equation}
\begin{split}
    \E\left[\norm[1]{\bCh-\bC}_{\Fr}\right] &= \E\left[\norm[1]{\bCh-\bC}_{\Fr}\mid\Eb\right]\Pp\left[\Eb\right]+ \E\left[\norm[1]{\bCh-\bC}_{\Fr}\mid\Eb^{\msfc}\right]\Pp\left[\Eb^{\msfc}\right]\\
    &\leq \sigma^2\sqrt{d}\Pp\left[\Eb\right] + \E\left[\norm[1]{\bCh^{\ast}-\bC}_{\Fr}\mid\Eb^{\msfc}\right]\\
    &<\varepsilon.
\end{split}
\end{equation}
This concludes the proof of Theorem \ref{thm:achievable_scheme-Fr} in the reasonable distortion error regime.
  
{\bf Case II: High distortion error}. In this regime, the distortion error satisfies $\varepsilon\geq 512\sigma^2\sqrt{d_{\min}}$. Here, we slightly modify the estimator $\bCh$ given in \eqref{eq:ach-estimator}. In this regime, we do not quantize the matrices $\bmsfX_1$ and $\bmsfX_2$. Instead,  agent $k$  devotes all  the communication budget $B_k$  for transmitting its self--covariance matrix estimator $\bCh_{\bX_k\bX_k}$. The central server simply returns:
\begin{equation}
  \bCh = \left[\begin{matrix}
    \bCh_{\bX_1\bX_1} & \bzero_{d_1\times d_2} \\
    \bzero_{d_2\times d_1} & \bCh_{\bX_2\bX_2}
\end{matrix}\right].
\end{equation}
In this case, error events $\Eb_{1,2}$ and $\Eb_{2,2}$ will never occur, because we don't aim to quantize matrices $\bmsfX_1$ and $\bmsfX_2$. Note that if we define $\Ebt$ as $\Ebt=\Eb_{1,1}\vee \Eb_{2,1}$, then $\Pp[\Ebt]\leq \Pp[\Eb]<\frac{\tep}{10000}$ and we can write:
\begin{equation}\label{eq:achievable1-fr-case2}
\begin{split}
    \E\left[\norm[1]{\bCh-\bC}_{\Fr}\right] &= \E\left[\norm[1]{\bCh-\bC}_{\Fr}\mid\Ebt\right]\Pp\left[\Ebt\right]+ \E\left[\norm[1]{\bCh-\bC}_{\Fr}\mid\Ebt^{\msfc}\right]\Pp\left[\Ebt^{\msfc}\right]\\
    &\leq \sigma^2\sqrt{d}\frac{\tep}{10000}+ \E\left[\norm[1]{\bCh-\bC}_{\Fr}\mid\Ebt^{\msfc}\right]\Pp\left[\Ebt^{\msfc}\right]\\
    &= \frac{\varepsilon}{10000}+ \E\left[\norm[1]{\bCh-\bC}_{\Fr}\mid\Ebt^{\msfc}\right]\Pp\left[\Ebt^{\msfc}\right].
\end{split}
\end{equation}
Now, we use the following counterpart of \eqref{eq:norm_traingle_inequality}:
\begin{equation}\label{eq:norm_traingle_inequality-fr-case2}
\begin{split}
    \E\left[\norm[1]{\bCh-\bC}_{\Fr}\mid\Ebt^{\msfc}\right] &= \E\left[\norm{\begin{bmatrix}
                    \bCh_{\bX_1\bX_1} & \bzero_{d_1\times d_2} \\ \bzero_{d_2\times d_1} & \bCh_{\bX_2\bX_2}
                \end{bmatrix}
                -\begin{bmatrix}
                    \bC_{\bX_1\bX_1} & \bC_{\bX_1\bX_2} \\
                    \bC_{\bX_1\bX_2}^\top & \bC_{\bX_2\bX_2}
                \end{bmatrix}}_{\Fr}\mid\Eb^{\msfc}\right]\\
    &=\E\left[\norm{\begin{bmatrix}
                    \bCh_{\bX_1\bX_1} - \bC_{\bX_1\bX_1}&  - \bC_{\bX_1\bX_2}\\
                    - \bC_{\bX_1\bX_2}^\top & \bCh_{\bX_2\bX_2} - \bC_{\bX_2\bX_2}
                \end{bmatrix}
                }_{\Fr}\mid\Eb^{\msfc}\right]\\
    &\leq \E\left[\norm[1]{\bCh_{\bX_1\bX_1}-\bC_{\bX_1\bX_1}}_{\Fr}\mid\Eb^{\msfc}\right] + \E\left[\norm[1]{\bCh_{\bX_2\bX_2} - \bC_{\bX_2\bX_2}}_{\Fr}\mid\Eb^{\msfc}\right]\\& \qquad+ \sqrt{2} \norm[1]{\bC_{\bX_1\bX_2}}_{\Fr}
    \\
    &\leq \sqrt{d_1}\E\left[\norm[1]{\bCh_{\bX_1\bX_1}-\bC_{\bX_1\bX_1}}_{\op}\mid\Eb^{\msfc}\right] + \sqrt{d_2}\E\left[\norm[1]{\bCh_{\bX_2\bX_2} - \bC_{\bX_2\bX_2}}_{\op}\mid\Eb^{\msfc}\right]\\& \qquad+ \sqrt{2d_{\min}} \norm[1]{\bC_{\bX_1\bX_2}}_{\op}.
\end{split}
\end{equation}
Using \eqref{eq:quantization_upperbound1}-- \eqref{eq:quantization_upperbound2} with $B_1' = B_1$,$B_2'=B_2$, values of $m,n,B_k$ in \eqref{eqn:ach-con-fr-1}--\eqref{eqn:ach-con-fr-2}, and $\beta=2\log_2(\frac{6912}{{\tep}})$, we have for $k=1,2$:
\begin{equation}\label{eq:norm_Cxx_upper_bound-Fr-case2}
\begin{split}
    \E\left[\norm[1]{\bCh_{\bX_k\bX_k}-\bC_{\bX_k\bX_k}}_{\op}\mid\Eb^{\msfc}\right] 
    &\leq \E\left[\norm[1]{\bCh_{\bX_k\bX_k}-\bCt_{\bX_k\bX_k}}_{\op}\mid\Eb^{\msfc}\right]\\&\qquad +\E\left[\norm[1]{\bCt_{\bX_k\bX_k}-\bC_{\bX_k\bX_k}}_{\op}\mid\Eb^{\msfc}\right]\\
    &\leq \epsilon_k' + 32\sigma^2 \sqrt{\frac{2d_k}{m}} \\
    &< 33\sigma^2 \cdot 2^{\frac{-B_k}{d_k^2}} + 32\sigma^2 \sqrt{\frac{2 d}{m}} \\
    &\leq \sigma^2\left(33 \cdot 2^{\log_2\left(\frac{\tep}{528}\right)} +\frac{\tilde{\varepsilon}}{16}\right)\\
    &\leq \frac{\sigma^2\tep}{8}.
\end{split}
\end{equation}
We can use this upper bound for $\norm[1]{\bC_{\bX_1\bX_2}}_{\op}$:
\begin{equation}\label{eq:norm_Cxy_upper_bound-Fr-case2}
\begin{split}
    \norm[1]{\bC_{\bX_1\bX_2}}_{\op} &= \sup_{\substack{\bu\in\R^{d_1},\bv\in\R^{d_2}\\ \norm{\bu}=\norm{\bv}=1}} \left\{\bu^\top \bC_{\bX_1\bX_2} \bv\right\}\\
    &= \sup_{\substack{\bu\in\R^{d_1},\bv\in\R^{d_2}\\ \norm{\bu}=\norm{\bv}=1}} \left\{\E\left[\bu^\top \bX_1\bX_2^\top \bv\right]\right\}\\
    &\leq \sup_{\substack{\bu\in\R^{d_1},\bv\in\R^{d_2}\\ \norm{\bu}=\norm{\bv}=1}} \left\{\sqrt{\E\left[(\bu^\top \bX_1)^2\right]\E\left[(\bv^\top\bX_2)^2\right]}\right\}\\
    &\leq \sigma^2.
\end{split}
\end{equation}
Now from \eqref{eq:achievable1-fr-case2}, \eqref{eq:norm_traingle_inequality-fr-case2}, \eqref{eq:norm_Cxx_upper_bound-Fr-case2}, and \eqref{eq:norm_Cxy_upper_bound-Fr-case2}, we conclude:
\begin{equation}
\begin{split}
    \E\left[\norm[1]{\bCh-\bC}_{\Fr}\right] &\leq \frac{\varepsilon}{10000} + \E\left[\norm[1]{\bCh-\bC}_{\Fr}\mid\Ebt^{\msfc}\right]\Pp\left[\Ebt^{\msfc}\right]\\
    &\leq \frac{\varepsilon}{10000} + \sqrt{d_1}\E\left[\norm[1]{\bCh_{\bX_1\bX_1}-\bC_{\bX_1\bX_1}}_{\op}\mid\Eb^{\msfc}\right] + \sqrt{d_2}\E\left[\norm[1]{\bCh_{\bX_2\bX_2} - \bC_{\bX_2\bX_2}}_{\op}\mid\Eb^{\msfc}\right]\\& \qquad+ \sqrt{2d_{\min}} \norm[1]{\bC_{\bX_1\bX_2}}_{\op}\\
    &\leq \frac{\varepsilon}{10000} + \frac{\sigma^2\sqrt{d_1}\tep}{8} + \frac{\sigma^2\sqrt{d_2}\tep}{8} + \sqrt{2d_{\min}} \sigma^2\\
    &\stackrel{\text{(a)}}{\leq} \frac{\varepsilon}{10000} + \frac{\varepsilon}{4} + \frac{\varepsilon}{256\sqrt{2}}\\
    &\leq \varepsilon,
\end{split}
\end{equation}
where (a) follows from the condition $\varepsilon\geq 512\sigma^2\sqrt{d_{\min}}$.
\end{proof}

\section{Achievable Scheme for Multi--Agent Scenario: Proof of Theorem \ref{thm:achievable_scheme-MulA} }
\label{app:proof_thm_achievable_scheme_MulA}
\begin{proof}
    We prove Theorem \ref{thm:achievable_scheme-MulA} by first focusing on a simplified case, specifically the fully distributed scenario where there is a distinct agent for each dimension $(K=d)$ and each agent controls only one dimension $(d_k=1)$.
    
    This proof holds for the more general case as well. Any agent that possesses $d_k > 1$ dimensions can be treated as $d_k$ "virtual" agents, each responsible for a single dimension. By applying our coding scheme to these virtual agents, we can reduce the general problem to the fully distributed case.

    Let's make the following choices for the number of samples and the parameter $\beta$:
    \begin{equation}
    \begin{aligned}
        m &\geq   n:=\frac{d}{\tilde{\varepsilon}^2},\\
        \beta &= \frac{{\varepsilon}}{2\sigma^2}.
    \end{aligned}
    \end{equation}
    The exact value of $\tilde{\varepsilon}$ will be determined in the sequel. We can proof Theorem \ref{thm:achievable_scheme-MulA} in 3 steps:
    
    \paragraph{Randomized quantization at user $k$}
    Let's represent the first $n$ samples from agent $k$ as the vector $\bmsfX_k = \left[ X_k^{(1)}, X_k^{(2)}, \dots, X_k^{(n)} \right] \in \mathbb{R}^n$. It is well--known that, with high probability, the magnitude of every sample is bounded. Specifically, we have $\left|X_k^{(i)}\right| \le L$ for all agents $k$ and all samples $i$, with probability at least $1-\beta$, where $L := \sigma\sqrt{2 \log ( dn/\beta)}$. This means the entire collection of $nd$ samples is guaranteed to lie within the $nd$--dimensional hypercube $[-L, L]^{nd}$.\\
    If $\|\bmsfX_k\|_\infty > L$, Agent $k$ transmits an error signal. Otherwise, Agent $k$ quantizes its observation vector coordinate--wise as follows. For simplicity, we assume $\frac{L}{\sigma\tilde{\varepsilon}}$ is a positive integer, and define $N = \frac{L}{\sigma\tilde{\varepsilon}}$. The quantization is performed by partitioning the interval $[-L, L]$ into $2N$ contiguous intervals of length $\sigma\tilde{\varepsilon}$, denoted by $\left\{ \left[ j\sigma\tilde{\varepsilon}, (j+1)\sigma\tilde{\varepsilon} \right) \right\}_{j = -N}^{N-2}$ and $\left[ (N-1)\sigma\tilde{\varepsilon}, N\sigma\tilde{\varepsilon} \right]$. Now for each $k\in[d]$ and $i\in[n]$, if $X_k^{(i)} \in \left[ j\sigma\tilde{\varepsilon} , (j+1)\sigma\tilde{\varepsilon}  \right)$ for some integer $j \in \{-N, -N+1, \dots, N-1\}$, it is randomly and independently quantized to $\hat{X}_k^{(i)} \in \left\{ j\sigma\tilde{\varepsilon} , (j+1)\sigma\tilde{\varepsilon}  \right\}$ such that $\E\left[ \hat{X}_k^{(i)} \mid X_k^{(i)} \right] = X_k^{(i)}$.\\
    This quantization requires $B_k^{(i)} := \log_2(2N + 1)$ bits per coordinate. Thus, $\bmsfX_k$ is quantized to $\widehat{\bmsfX}_k := \left[ \hat{X}_k^{(1)}, \dots, \hat{X}_k^{(n)} \right]$ using $B_k := \sum_{i=1}^n B_k^{(i)} = n \log_2(2N + 1)$ bits.

    \paragraph{Covariance Matrix Estimation at the Server}
    If the central server receives an error signal from any agent, it immediately sets $\bCh = \bzero$. Otherwise, after receiving the quantized vectors $\hat{\bmsfX}_k$ from all agents, the server constructs a single aggregated matrix by vertically stacking the received vectors:
    \begin{equation}
        \widehat{\bmsfX} := \begin{bmatrix} \widehat{\bmsfX}_1 \\ \vdots \\ \widehat{\bmsfX}_d \end{bmatrix}.
    \end{equation}
    The server then computes the final covariance matrix estimate using the sample covariance estimator:
    \begin{equation}
        \bCh = \frac{1}{n} \widehat{\bmsfX} \widehat{\bmsfX}^\top.
    \end{equation}
    \paragraph{Analysis:} Let $\Eb$ be the event of receiving the error signal from at least one agent. Thus $\Eb$ occurs if $|X_{k}^{(i)}|> L$, for some $(k,i)$. We have:
    \begin{equation}\label{eq:achievable1-mul}
    \begin{aligned}
        \E\left[\norm[1]{\bCh-\bC}_{\op}\right] &= \E\left[\norm[1]{\bCh-\bC}_{\op}\mid\Eb\right]\Pp\left[\Eb\right]+ \E\left[\norm[1]{\bCh-\bC}_{\op}\mid\Eb^{\msfc}\right]\Pp\left[\Eb^{\msfc}\right].
    \end{aligned}
    \end{equation}
    We find an upper bound for every term of  \eqref{eq:achievable1-mul}. First, notice that when an error is received in central server, the central server returns $\bCh=\bzero$, therefore:
    \begin{equation}
        \E\left[\norm[1]{\bCh-\bC}_{\op}\mid\Eb\right] = \E\left[\norm[1]{\bC}_{\op}\mid\Eb\right]
        \leq \sigma^2.
    \end{equation}
    Now we can find an upper bound for $\Pp[\Eb]$ and $\Pp[\Eb^{\msfc}]$:
    \begin{equation}
        \Pp[\Eb]\leq \beta, \qquad \Pp[\Eb^{\msfc}]\leq 1.
    \end{equation}
    And finally we upper bound $\E\left[\norm[1]{\bCh-\bC}_{\op}\mid\Eb^{\msfc}\right]$:
    \begin{equation}\label{eqn:norm-apx-mul-0}
    \begin{aligned}
        \E\left[\norm[1]{\bCh-\bC}_{\op}\mid\Eb^{\msfc}\right] &= \E\left[\norm[2]{\frac{1}{n} \widehat{\bmsfX} \widehat{\bmsfX}^\top-\bC}_{\op}\mid\Eb^{\msfc}\right]\\
        &= \E\left[\norm[2]{\frac{1}{n}\left(
        (\widehat{\bmsfX}-\bmsfX) (\widehat{\bmsfX}-\bmsfX)^\top  +  (\widehat{\bmsfX}-\bmsfX) \bmsfX^\top +  \bmsfX (\widehat{\bmsfX}-{\bmsfX})^\top\right) + \frac{1}{n} \bmsfX \bmsfX^\top -\bC}_{\op}\mid\Eb^{\msfc}\right]\\
        &\leq \frac{1}{n}\E\left[\norm[2]{ \left(\widehat{\bmsfX} - \bmsfX\right)\left(\widehat{\bmsfX} - \bmsfX\right)^\top}_{\op}\mid\Eb^{\msfc}\right] + \frac{2}{n}\E\left[\norm[2]{ \left(\widehat{\bmsfX} - \bmsfX\right)\bmsfX^\top}_{\op}\mid\Eb^{\msfc}\right]\\
        &\qquad+ \E\left[\norm[2]{\frac{1}{n} \bmsfX \bmsfX^\top -\bC}_{\op}\mid\Eb^{\msfc}\right]\\
        &\leq \frac{1}{n}\E\left[\norm[1]{\widehat{\bmsfX} - \bmsfX}_{\op}^2\mid\Eb^{\msfc}\right] + \frac{2}{n}\E\left[\norm[1]{ \widehat{\bmsfX} - \bmsfX}_{\op}\norm[1]{\bmsfX}_{\op}\mid\Eb^{\msfc}\right] + \E\left[\norm[2]{\frac{1}{n} \bmsfX \bmsfX^\top -\bC}_{\op}\mid\Eb^{\msfc}\right].
    \end{aligned}
    \end{equation}
    We observe that when $\bmsfX$ satisfies the event $\Eb^{\msfc}$, the entries of the random matrix $\widehat{\bmsfX} - \bmsfX$ are conditionally independent. Furthermore, all the entries are bounded, $|\widehat{X}_k^{(i)}-{X}_k^{(i)}|<\sigma\tilde{\varepsilon}$, and are zero mean. Thus each entry is $\sigma\tilde{\varepsilon}$--sub--Gaussian. Hence given $(\bmsfX,\Eb^{\mathsf{c}})$, $\widehat{\bmsfX} - \bmsfX$ satisfies the conditions of Lemma \ref{lem:norm_upperbound_crosscovar}, and we can invoke this lemma to obtain:
    \begin{equation}
    \begin{aligned}
        \E\left[\norm[1]{\widehat{\bmsfX} - \bmsfX}_{\op}^2\mid(\bmsfX,\Eb^{\msfc})\right]  &\le 36 \sigma^2\tilde{\varepsilon}^2(d+n),\\
        \E\left[\norm[1]{\widehat{\bmsfX} - \bmsfX}_{\op}\mid(\bmsfX,\Eb^{\msfc})\right]  &\le 9 \sigma\tilde{\varepsilon}\sqrt{d+n}.
    \end{aligned}
    \end{equation}
    These yield:
    \begin{equation}\label{eqn:norm-apx-mul-1}
    \begin{aligned}
        \E\left[\norm[1]{\widehat{\bmsfX} - \bmsfX}_{\op}^2\mid\Eb^{\msfc}\right]  &\le 36 \sigma^2\tilde{\varepsilon}^2(d+n),\\
        \E\left[\norm[1]{\widehat{\bmsfX} - \bmsfX}_{\op}\norm[1]{ \bmsfX}_{\op}\mid\Eb^{\msfc}\right]  &\le 9 \sigma\tilde{\varepsilon}\sqrt{d+n}\E\left[\norm[1]{ \bmsfX}_{\op}\mid\Eb^{\msfc}\right].
    \end{aligned}
    \end{equation}
    Furthermore $\bmsfX$ satisfies the conditions of Lemma \ref{lem:norm_upperbound_crosscovar}, and we have:
    \begin{equation}\label{eqn:norm-apx-mul-2}
        \E\left[\norm[1]{ \bmsfX}_{\op}\mid\Eb^{\msfc}\right]\le \frac{\E\left[\norm[1]{ \bmsfX}_{\op}\right]}{\Pp\left[\Eb^{\msfc}\right]}  \le \frac{9\sigma\sqrt{d+n}}{1-\beta}.
    \end{equation}
    We also have:
    \begin{equation}\label{eqn:norm-apx-mul-3}
        \E\left[\norm[2]{\frac{1}{n} \bmsfX \bmsfX^\top -\bC}_{\op}\mid\Eb^{\msfc}\right]\le 32\sigma^2\sqrt{\dfrac{d}{n}}.
    \end{equation}
    Putting \eqref{eqn:norm-apx-mul-1}, \eqref{eqn:norm-apx-mul-2}, and \eqref{eqn:norm-apx-mul-3} in \eqref{eqn:norm-apx-mul-0} yields:
    \begin{equation}\label{eqn:full-cov-app-mul}
    \begin{aligned}
        \E\left[\norm[1]{\bCh-\bC}_{\op}\mid\Eb^{\msfc}\right] &\leq  36\sigma^2\tilde{\varepsilon}^2\frac{d+n}{n} + \frac{162}{1-\beta}\sigma^2\tilde{\varepsilon}\frac{d+n}{n} + 32\sigma^2\sqrt{\frac{d}{n}}\\
        &\leq 360\sigma^2\frac{d+n}{n}\tilde{\varepsilon} + 32\sigma^2\tilde{\varepsilon}\\
        &\leq 760 \sigma^2\tilde{\varepsilon}.
    \end{aligned}
    \end{equation}
    By setting $\tilde{\varepsilon}=\frac{\varepsilon}{1520\sigma^2}$ and substituting \eqref{eqn:full-cov-app-mul} in the expression \eqref{eq:achievable1-mul}, we conclude that the achievable scheme has the desired  distortion error less than $\varepsilon$: 
     \begin{equation}
        \E\left[\norm[1]{\bCh-\bC}_{\op}\right]\leq \beta\sigma^2 + \frac{\varepsilon}{2}\leq \varepsilon. 
     \end{equation}
    Finally the communication budget of each agent  in this scheme satisfies:
    \begin{equation}
         B_k=n\log_2(2N+1)\leq \tau' \frac{\sigma^4d}{{\varepsilon}^2}\log_2\left(\tau''\frac{\sigma^4}{\varepsilon^2}\log(d\sigma^2/\varepsilon)\right),
    \end{equation}
    for some constant $\tau',\tau''$. This concludes the proof.
 
\end{proof}

\end{document}